\documentclass[11pt]{article}

\usepackage{microtype}
\usepackage{graphicx}
\usepackage{subfigure}
\usepackage{booktabs} 
\usepackage{natbib}
\usepackage{epsfig,epsf,fancybox}
\usepackage{amsmath}
\usepackage{mathrsfs}
\usepackage{amssymb}
\usepackage{color}
\usepackage{multirow}
\usepackage{paralist}
\usepackage{verbatim}
\usepackage{algorithm}
\usepackage{algorithmic}
\usepackage{galois}
\usepackage{boxedminipage}
\usepackage{accents}
\usepackage{stmaryrd}
\usepackage[bottom]{footmisc}

\usepackage{natbib}

\usepackage{url}
\usepackage[colorlinks,linkcolor=magenta,citecolor=blue, pagebackref=true,backref=true]{hyperref}
\renewcommand*{\backrefalt}[4]{%
    \ifcase #1 \footnotesize{(Not cited.)}%
    \or        \footnotesize{(Cited on page~#2.)}%
    \else      \footnotesize{(Cited on pages~#2.)}%
    \fi}

\usepackage{hyperref}

\usepackage{natbib}
\usepackage[colorlinks,linkcolor=magenta,citecolor=blue, pagebackref=true,backref=true]{hyperref}
\usepackage{url}
\usepackage{macro}
\usepackage{undertilde}
\usepackage{enumitem}
\usepackage{xspace}

\newcommand{\Winit}{\btheta^{(0)}}
\newcommand{\oK}{\overline{K}}
\newcommand{\ok}{\overline{k}}
\newcommand{\ophi}{\phi_{\overline{V}_{h+1}^t}}
\newcommand{\ophiholder}{\overline{\phiholder}}
\newcommand{\oV}{\overline{V}}
\newcommand{\oLambda}{\overline{\Lambda}}
\newcommand{\otheta}{\overline{\btheta}}

\newcommand{\algname}{\texttt{KernelCCE-VTR}\xspace}
\newcommand{\algnameprime}{\texttt{KernelCCE-VTR}\xspace}
\newcommand{\algnameB}{\texttt{KernelCCE-VTR+}\xspace}
\newcommand{\OURALGO}{\algname}

\newcommand{\genpo}{^{\pi, \nu}} 
\newcommand{\tbr}{\text{br}}
\newcommand{\hprod}[2]{\left\langle #1, #2 \right\rangle_{\cH}}
\newcommand{\W}{\btheta}
\newcommand{\Wbar}{\overline{\btheta}_h^t}
\newcommand{\Wubar}{\underline{\btheta}_h^t}
\newcommand{\phiholder}{\bm{\Psi}}
\newcommand{\QU}[1]{\overline{Q}_h^t(#1)}
\newcommand{\QL}[1]{\underline{Q}_h^t(#1)}
\newcommand{\VUh}{\overline{V}_h^t}
\newcommand{\VLh}{\underline{V}_h^t}
\newcommand{\VUp}{\overline{V}_{h+1}^t}
\newcommand{\VLp}{\underline{V}_{h+1}^t}
\newcommand{\QUxh}{\QU{x_h^t, a_h^t, b_h^t}}
\newcommand{\QLxh}{\QL{x_h^t, a_h^t, b_h^t}}
\newcommand{\QUxab}{\QU{x, a, b}}
\newcommand{\QLxab}{\QL{x, a, b}}
\newcommand{\yholder}{\mathbf{y}}
\newcommand{\xholder}{\mathbf{x}}
\newcommand{\Wholder}{\mathbf{W}}
\newcommand{\normh}[1]{\left\|#1 \right\|_{\cH}}
\newcommand{\norml}[1]{\left\|#1\right\|_{\left(\Lambda_h^t\right)^{-1}}}
\newcommand{\logdet}{\text{logdet}}

\def\vec{\mathop{\text{vec}}}

\newcommand{\Vest}{\mathbb{V}^{\text{est}}}
\newcommand{\deff}{d_{\text{eff}}}
\newcommand{\uppvar}{\overline{R}_h}
\newcommand{\lowvar}{\underline{R}_h}

\usepackage{enumitem}
\usepackage[dvipsnames]{xcolor}
\def\blue#1{\textcolor{black}{#1}}

\def\red#1{}

\newcommand{\bonus}{w}
\usepackage{wrapfig}
\begin{document}






\begin{center}

{\bf{\LARGE{Learning Two-Player Mixture Markov Games: Kernel Function Approximation and Correlated Equilibrium}}}

\vspace*{.2in}

{\large{
\begin{tabular}{cccc}
Chris Junchi Li$^{\diamond *}$
&
Dongruo Zhou$^{\ddagger *}$ 
&
Quanquan Gu$^{\ddagger}$
&
Michael I.~Jordan$^{\diamond,\dagger}$
\end{tabular}
}}

\vspace*{.2in}

\begin{tabular}{c}
Department of Electrical Engineering and Computer Sciences,
University of California, Berkeley$^\diamond$
\\
Department of Statistics,
University of California, Berkeley$^\dagger$
\\
Department of Computer Sciences,
University of California, Los Angeles$^\ddagger$
\end{tabular}

\vspace*{.2in}

\today
\end{center}

\vspace*{.2in}

\begin{abstract}
We consider learning Nash equilibria in two-player zero-sum Markov Games with nonlinear function approximation, where the action-value function is approximated by a function in a Reproducing Kernel Hilbert Space (RKHS). The key challenge is how to do exploration in the high-dimensional function space. We propose a novel online learning algorithm to find a Nash equilibrium by minimizing the duality gap. At the core of our algorithms are upper and lower confidence bounds that are derived based on the principle of optimism in the face of uncertainty. We prove that our algorithm is able to attain an $O(\sqrt{T})$ regret with polynomial computational complexity, under very mild assumptions on the reward function and the underlying dynamic of the Markov Games.
We also propose several extensions of our algorithm, including an algorithm with Bernstein-type bonus that can achieve a tighter regret bound, and another algorithm for model misspecification that can be applied to neural function approximation.

\end{abstract}

\newcommand\blfootnote[1]{%
  \begingroup
  \renewcommand\thefootnote{}\footnote{#1}%
  \addtocounter{footnote}{-1}%
  \endgroup
}
\blfootnote{$^*$Equal contribution.}

\section{Introduction}\label{sec:intro}
Multi-agent reinforcement learning (MARL) has been the focus of research across a range of research communities~\citep{shapley1953stochastic,littman1994markov}.  The case of two-player Markov Games (MG) has been of particular interest. In this case, two players select their actions based on the current state simultaneously and independently. One player (the max-player) aims to maximize the return based on the reward provided by the environment, while the other (the min-player) aims to minimize it. A series of recent results have established polynomial sample complexity/regret guarantees that depend on the cardinality of state/action spaces for two-player zero-sum MGs~\citep{wei2017online,bai2020provable,bai2020near,liu2020sharp, jia2019feature,sidford2020solving,cui2020minimax, lagoudakis2012value,perolat2015approximate,perolat2016softened,perolat2016use,perolat2017learning, jin2021v}.

Meanwhile, most of the  recent successful  applications of MARL deal with \emph{large state/action spaces} that may be continuous or a fine-grained discretizations of a continous space. Examples include GO \citep{silver2016mastering}, autonomous driving \citep{shalev2016safe}, TexasHold’em poker \citep{brown2019superhuman}, and AlphaStar for the game Starcraft \citep{vinyals2019grandmaster}. In order to tackle problems with large state/action spaces, researchers have designed MARL algorithms based on \emph{function approximation} which approximate the original high-dimensional value function/policy by a function approximator. For instance, \citet{xie2020cce} and \citet{chen2021almost} studied RL for two-player zero-sum MGs with \emph{linear function approximation}, where it is assumed that there are a set of \emph{linear features} that span the transition kernel and reward function spaces. In contrast to RL with linear function approximation, RL with  \emph{nonlinear function approximation} (e.g., kernel and neural network approximation) aims to take advantage of the superior representational power of nonlinear functions compared to linear parameterizations. For example, \citet{jin2021power} studied neural-network-based RL in the setting of \emph{MGs with low multi-agent Bellman eluder dimension}, obtaining algorithms that have polynomial dependence on the complexity of the  underlying function class. Although this yields a strong theoretical guarantee, the algorithm that they propose is not computationally efficient due to the nonconvexity of the confidence sets that are constructed. The following question remains open:%
\emph{
Can we design a computationally and statistically efficient RL algorithm for learning two-player Markov Games with nonlinear function approximation?
}

In this paper, we give an affirmative answer to this question for a class of episodic Markov Games, dubbed \emph{mixture Markov Games}, when using nonlinear approximations from a Reproducing Kernel Hilbert Space (RKHS). We propose a novel kernel-based MARL algorithmic framework for general episodic two-player zero-sum MGs which provides provable regret guarantees. We summarize the contributions of our work as follows:

\begin{enumerate}[label=(\roman*)]
\item
We propose a $\algname$ algorithm for two-player zero-sum MGs. In particular,  at each episode, $\algname$ uses kernel function approximation to approximate the optimal value function and constructs corresponding confidence sets, following the \emph{``Optimism-in-Face-of-Uncertainty''} principle \citep{abbasi2011improved} to select an action based on the current state. In contrast to the algorithms in \citet{jin2021power}, which construct implicit confidence sets that are in general computationally intractable, our  $\algname$ algorithm crafts a computationally efficient exploration bonus based on the Gram matrix of the kernel function. 

\item
Under the assumption that the transition dynamics belongs to some RKHS, we show that our $\algname$ algorithm  is able to find a Nash equilibrium of the game with an $\tilde{O}( d_{\cF} H^2 \sqrt{T})$ regret bound on the duality gap, where $H$ is the horizon, $T$ is the number of the episodes, and $d_{\cF}$ represents the complexity of the function class $\cF$. We also propose an extension of $\algname$ that utilizes \emph{weighted kernel ridge regression} and a \emph{Bernstein-type bonus} to achieve $\tilde O(d_{\cF} H^{3/2} \sqrt{T})$ regret. When $\cF$ reduces to the $d$-dimensional linear function class, our regret reduces to $\tilde{O}(dH^{3/2}\sqrt{T})$, which almost matches the lower bound in~\citet{chen2021almost}. 

\item We also study the general case where the transition dynamics belongs to some RKHS up to a misspecification error. We show that $\algname$ can achieve a similar regret as in the well-specified case. In particular, we study the neural network function approximation case which can be regarded as a special instance of the misspecified RKHS case and derive the corresponding regret bound.
\end{enumerate}


\paragraph{Notation.}
We use lower-case letters to denote scalars, and lower- and upper-case bold letters to denote vectors and matrices. We use $\| \cdot \|$ to indicate Euclidean norm, and for a semi-positive definite matrix $\bSigma$ and any vector $\xb$, we define $\| \xb \|_{\bSigma} := \| \bSigma^{1/2} \xb \| = \sqrt{\xb^{\top} \bSigma \xb}$. For real $t$ and interval $[a,b]$, we use $\Pi_{[a,b]}[t]$ to indicate the projection of $t$ onto $[a,b]$, i.e.~$
    \Pi_{[a,b]}[t] = \max\left(a,\min(b,t)\right)
$.
For positive integer $N$ we define $[N] = \{1,\dots,N\}$.
We also adopt the standard big-$O$ and big-$\Omega$ notation: say $a_n = O(b_n)$ if and only if there exists $C > 0, N > 0$, for any $n > N$, $a_n \le C b_n$; $a_n = \Omega(b_n)$ if $a_n \ge C b_n$. The notations $\tilde{O}$ and $\tilde{\Omega}$ are adopted when $C$ hides a polylogarithmic factor.







\section{Related Work}\label{sec_related}
\paragraph{Online RL with function approximation.}
MARL  with function approximation can be seen as an extension of RL with function approximation on MDPs.  There are several lines of work studying RL with function approximation. The first line of work studies the so-called linear MDP which assumes the reward function and transition dynamics are linear functions of a feature mapping defined on the state and action spaces \citep{yang2019reinforcement, jin2019provably, zanette2020learning}. These works propose model-free algorithms with sublinear regret with respect to the number of episodes $K$. The second line of work studies the linear mixture MDP which assumes the transition kernel is a linear combination of several base models \citep{modi2019sample, jia2020model, zhou2021provably, zhou2020nearly}. These studies propose model-based RL algorithms that estimate the transition kernel with finite sample complexity or sublinear regret guarantees. The third line of work studies general function approximation for either the value function or the transition kernel \citep{osband2014model, jiang2017contextual, sun2019model, wang2020reinforcement, yang2020function, du2021bilinear, jin2021bellman}. Algorithms proposed in this vein enjoy finite regret or sample complexity bounds that depend on  general complexity measures such as Eluder dimension \citep{russo2013eluder, osband2014model}, Bellman rank \citep{jiang2017contextual}, witness rank \citep{sun2019model}, information gain \citep{yang2020function}, bilinear class \citep{du2021bilinear} and Bellman eluder dimension \citep{jin2021bellman}.

\paragraph{Learning two-player MGs with function approximation.}
There is a large body of literature on MARL for two-player MGs with function approximation. These works can be generally categorized into MARL with \emph{linear function approximation} and MARL with \emph{general function approximation}. For example, for linear function approximation, \citet{xie2020cce} studied zero-sum simultaneous-move MGs where both the reward and transition kernel can be parameterized as linear functions of feature mappings. They proposed an OMVI-NI algorithm with an $\tilde O(\sqrt{d^3 H^3 T})$ regret, where $d$ is the number of the feature dimension, $H$ is the episode length and $T$ is the total number of rounds. \citet{chen2021almost} studied the linear mixture MGs and proposed a nearly minimax optimal Nash-UCRL-VTR algorithm with an $\tilde O(dH\sqrt{T})$ regret and an $\Omega(dH\sqrt{T})$ matching lower bound. In contrast to this  work, our $\algname$ does not assume the underlying transition dynamic or reward function have a linear structure. For MARL with general function approximation, \citet{jin2021power} studied the two-player zero-sum MGs with low multi-agent Bellman Eluder dimension and proposed a ``Golf with Exploiter'' algorithm using a general function class. They showed their algorithm enjoys an $\tilde O(H\sqrt{dK\log N})$ regret, where $d$ is the multi-agent Bellman eluder dimension, and $K$ is the number of episodes. 
\citet{huang2021towards} studied two-player MGs with a finite minimax Eluder dimension and proposed an ONEMG method with an $\tilde O(H\sqrt{dK\log N})$ regret, where $d$ is the minimax Eluder dimension. To obtain the desired function approximator, both Golf with Exploiter and ONEMG need to solve a constrained optimization problem which is computationally intractable even in the linear function approximation setting. 
In contrast to \citet{jin2021power} and \citet{huang2021towards}, our proposed algorithms are computationally efficient and nearly optimal when using the Bernstein-type bonus. \citet{qiu2021reward} also studied kernel function approximation for two-player MGs. However, there are two key differences between our work and theirs. First, \citet{qiu2021reward} studied MGs where the expectation of the value function is in some RKHS; we, on the other hand, assume that the transition dynamics of the MG lies in an RKHS. Second, while the regret result in \citet{qiu2021reward} depends on the covering number of the function space, our regret is \emph{independent} of the covering number.


\section{Preliminaries}\label{sec_prelim}
In this section, we present the necessary definitions that will be adopted throughout the paper. 
Section~\ref{sec_prelim_twoplayer} describes simultaneous-move games in the setting of zero-sum two-player Markov Games (MG) and recaps the concepts of equilibrium and duality gap that are employed in the game theory literature.
Section~\ref{sec_prelim_rkhs} provides necessary definitions and notation for approximations based on a reproducing kernel Hilbert space (RKHS). 

\subsection{Two-player Markov Games}\label{sec_prelim_twoplayer}
A simple instance of Markov Games, referred to as turn-based games, can be seen as a special case of simultaneous-move games.%
\footnote{We present a discussion of the implications of our results for turn-based games in the supplementary materials.}
In a zero-sum two-player simultaneous-move Markov Game, the dynamical structure is captured by an MG, denoted $(\cS, \cA_1, \cA_2, r, \PP, H)$, where $\cS$ is the space of  states of the environment, $\cA_1$ is the action space of the first player and $\cA_2$ is the action space of the second player. $H$ is the time horizon representing the maximum step of each round of play. The reward function $r: \left\{r_h(x, a, b): h \in [H]\right\}$ is a sequence of mappings from $\cS \times \cA_1 \times \cA_2$ to $[-1, 1]$. The transition matrix $\PP: \left\{\PP_h(\cdot | x, a, b): h \in [H] \right\}$ gives for each triplet $(x, a, b)$ and at each time $h$ the stochastic response of the environment to the next $x' \in \cS$.  Here by ``simultaneous move'' we refer to the setting where at each round of game the two players $P_1$ and $P_2$ take actions $a \in \cA_1, b \in \cA_2$ simultaneously at a given state $x \in \cS$, in contrast with the turn-based game where $r_h$ and $\PP_h$ are defined for a state-action pair $(x, a)$ where the action can be taken by either player. In the context of this paper, for simplicity of notation we let $\cA_1 = \cA_2 = \cA$, noting that the results can be easily generalized to the case when $\cA_1 \neq \cA_2$. Similar definitions of a zero-sum two-player simultaneous-move episodic Markov Games can be found in~\citet{wei2017online,perolat2018actor,xie2020cce}. 

In the above setting, two players $P_1$ and $P_2$ take actions according to their individual strategies. We use $\pi := \{\pi_h\}_{h \in [H]}$ to denote the stochastic policy of $P_1$ and use $\nu := \{\nu_h\}_{h \in [H]}$ to denote the stochastic policy of $P_2$. We note that at time $h$, $\pi_h: \cS \mapsto \Delta_{\cA}$ maps the current state $x_h$ to a probability distribution of the actions, and similarly for $\nu_h$.
Given two agents' policies, $\pi, \nu$, across $h$ steps, the state value function is defined as the expected total reward through $H$ steps where at step $h \in [H]$ player $P_1$ follows policy $\pi_h(\cdot | x_h)$ and player $P_2$ follows policy $\nu_h(\cdot | x_h)$,
\begin{align}
    V_h^{\pi, \nu}(x) 
:= 
\EE_{\pi, \nu} \left[\sum_{t = h}^H r_t(s_t, a_t, b_t)  \mid x_h = x\right]
,\quad \blue{
V_{H+1}^{\pi, \nu}(x):=0
},\notag
\end{align}
and where $V^{\pi, \nu}(x) := V_1^{\pi, \nu}(x)$.
Note that the expectation is taken over all stochasticity in $\pi_h, \nu_h$ and $\PP_h$.
The action-value function is defined as 
\begin{align}
Q_h^{\pi, \nu}(x, a, b) 
:= 
\EE_{\pi, \nu} \left[\sum_{t = h}^H r_t(x_t, a_t, b_t) \,\bigg|\, x_h = x,  a_h = a, b_h = b\right]
,\quad \blue{
Q_{H+1}^{\pi, \nu}(x,a,b):=0
},\notag
\end{align}
and
$
Q^{\pi, \nu}(x, a, b) 
:= 
Q_1^{\pi, \nu}
.
$
From the definition of two value functions, we observe that for any $x \in \cS$, the state value function given policy pair $(\pi, \nu)$ is the expectation of the corresponding action-value function
\begin{align}
    V_h\genpo(x) 
:= 
\EE_{(a, b) \sim (\pi, \nu)} Q_h\genpo(x, a, b)
,\notag
\end{align}
where the expectation is taken over the action distribution induced by the policy pair. {\blue During this paper, we use superscripts to denote the number of episodes and subscripts to denote the number of horizon steps. }


\paragraph{Nash equilibrium and duality gap.} 
In a zero-sum two-player Markov Game, $P_1$ wants to maximize the expected reward $V^{\pi, \nu}(x)$ via the choice of the policy $\pi$. On the other hand, $P_2$ wants to minimize $V^{\pi, \nu}(x)$ by the choice of $\nu$. For fixed $\nu$, we define the best-response policy with respect to $V$ and $\nu$ as $\tbr (\nu)$ and define $V_h^{*, \nu} := V_h^{\tbr(\nu), \nu}$ and $Q_h^{*, \nu} := Q_h^{\tbr(\nu), \nu}$, We define $V_h^{\pi, *} := V_h^{\pi, \tbr(\pi)}$ and $Q_h^{\pi, *} := Q_h^{\pi, \tbr(\pi)}$ similarly.
A Nash equilibrium is a pair of policies $(\pi^*, \nu^*)$ that are the best-response policy for each other, which we write as $V^{\pi^*, *}(x) = V^{\pi^*, \nu^*}(x) = V^{*, \nu^*}(x)$.
For notational simplicity we write $V^* := V^{\pi^*, \nu^*}, Q^* := Q^{\pi^*, \nu^*}$. By definition of the best-response policy, we obtain weak duality:
\begin{align}
    V_h^{\pi, *}(x) \leq V_h^*(x) \leq V_h^{*, \nu}(x).\notag
\end{align}
For any policy pair $(\pi, \nu)$, we define the duality gap as $V_1^{*, \nu^t}(x_1^t) -V_1^{\pi^t, *}(x_1^t)$. We call a pair an \emph{$\epsilon$-approximate Nash equilibrium (NE)} if $V_1^{*, \nu^t}(x_1^t) -V_1^{\pi^t, *}(x_1^t) \leq \epsilon$. We also define the regret in the MG setting as follows:
\begin{align}
    \textrm{Regret}(T)
:=
\sum_{t = 1}^T V_1^{*, \nu^t}(x_1^t) 
-
V_1^{\pi^t, *}(x_1^t)
.\notag
\end{align}

\paragraph{Coarse Correlated Equilibrium.}
We introduce the \emph{Coarse Correlated Equilibrium (CCE)} solution concept. Given payoff matrices $Q_1, Q_2 : \cS \times \cA \times \cA \mapsto \RR$ and the state $x$, we define the CCE of the game as a joint distribution $\sigma$ on $\cA \times \cA$ satisfying:
\begin{align}
\mathbb{E}_{(a, b) \sim \sigma}\left[Q_{1}(x, a, b)\right] &\geq \mathbb{E}_{b \sim \mathcal{P}_{2} \sigma}\left[Q_{1}\left(x, a^{\prime}, b\right)\right]
,\quad
\forall a^{\prime} \in \mathcal{A}\label{cce:1}
,\\
\mathbb{E}_{(a, b) \sim \sigma}\left[Q_{2}(x, a, b)\right] & \leq \mathbb{E}_{a \sim \mathcal{P}_{1} \sigma}\left[Q_{2}\left(x, a, b^{\prime}\right)\right]
,\quad
\forall b^{\prime} \in \mathcal{A}
,\label{cce:2}
\end{align}
where $\cP_1 \sigma$ denotes the marginal of $\sigma$ on the first coordinate (min-player) and $\cP_2 \sigma$ denotes the marginal of $\sigma$ on the second coordinate (max-player). {\blue We use $\texttt{FIND\_CCE}(Q_1, Q_2, x)$ to denote $\sigma$. When $\sigma$ can be written as a product of two policies over action space $\cA$, it is an Nash equilibrium \citep{xie2020cce}. To compute a CCE given $Q_1, Q_2, x$, please refer to Appendix \ref{app:cce}.}




\subsection{Nonlinear function approximation by reproducing kernel Hilbert spaces}\label{sec_prelim_rkhs}

For simplicity of notation, we use $z = (x, a, b)$ to denote a state-action triplet in $\cZ := \cS \times \cA \times \cA$.
An RKHS $\cH$ with kernel $K(\cdot, \cdot): \cZ \times \cZ \mapsto \RR$ is a general form of linear function class.
Every RKHS $\cH$ consists of functions on $\cZ$, with a feature mapping, $\phi: \cZ \mapsto \cH$, such that $\forall f \in \cH$ and $\forall z \in \cZ$, $f(z) = \hprod{f}{\phi(z)}$. The kernel $K$ is thus defined for every $x, y \in \cZ \times \cZ$ as $K(x, y) = \hprod{\phi(x)}{\phi(y)}$. We call $\phi$ the feature mapping induced by the RKHS $\cH$ with kernel $K$. In the following sections, we use $f^\top g$ as a simplification of $\hprod{f}{g}$ when $f, g \in \cH$. We make no distinction in notation between the vector product and the product $\hprod{\cdot}{\cdot}$; the distinction can be read out from the nature of the two objects in the product. For every RKHS $\cH$, there exists a natural eigenvalue decomposition in $\cL^2(\cZ)$. RKHS approximation is a generalization of the linear function approximation of finite dimension $d$ which can be infinite dimensional. In the following, we define the so-called \emph{kernel mixture MG}, which can be regarded as an extension from the linear mixture MDP \citep{jia2020model, ayoub2020model, zhou2020nearly} and linear mixture MG \citep{chen2021almost} to their kernel counterpart.

\paragraph{Kernel mixture MG.}
In a kernel mixture MG model, we model the transition probability $\PP_h(s' | z): \cZ \mapsto \Delta(\cS)$ as an element in an RKHS $\cH$ with feature mapping $\phi(s' | z): \cZ \times \cS \rightarrow \cH$, such that for an unknown parameter $\btheta_h^* \in \cH$, we have $\PP_h(s' | z) = \left\langle \phi(s' | z), \btheta_h^*  \right\rangle_{\cH}$ for all $ s' \in \cS$ and $z \in \cZ$. 
A similar MG structure called kernel MG has been studied by \citet{qiu2021reward}, which assumes that the transition probability satisfies $\PP_h(s'|z) = \la \phi(z), \mu_h(s')\ra$ for some $\phi(\cdot), \mu_h(\cdot) \in \cH$. The single-agent MDP counterparts of kernel MGs and kernel mixture MGs are linear MDPs and linear mixture MDPs. \citet{zhou2021provably} have shown that linear MDPs and linear mixture MDPs are different classes of MDPs and one cannot be covered by the other. Following a similar argument, we can also show that kernel mixture MGs and kernel MGs are different classes of MGs and cannot imply each other. 

At time $h$, for any estimate of the value function $V_h(\cdot): \cS \mapsto \RR$, the expectation of value function at time $h+1$, $\PP_h V_{h+1}$, is an element in the RKHS $\PP_h V_{h+1}(z) = \left\langle \phi_{V_{h+1}}(z), \btheta_h^* \right\rangle_{\cH}$,
where 
$
\phi_{V_{h+1}}(z) := \sum_{s' \in \cS} \phi(s' | z) V_{h+1}(s')
$
integrates the product of the feature mapping with the estimated value of $s'$ over $\cS$. It is worth noting that the quantity $\phi_V(\cdot)$ plays an important role in previous linear mixture model-based algorithms \citep{jia2020model, ayoub2020model, zhou2020nearly, chen2021almost}. We assume that for any bounded value function $V(\cdot): \cS \mapsto [-1, 1]$ and any $z \in \cZ$, we have $\|\phi_{V}(z)\|_{\cH}\leq 1$.
Given that the reward function $r_h(z)$ is known, we obtain through the Bellman equation that 
\begin{align}
Q_h^{*, \nu}(\cdot) = r_h(\cdot) + (\PP_h V_{h+1}^{*, \nu})(\cdot)
    &=
r_h(\cdot) + \left\langle \phi_{V_{h+1}^{*, \nu}}(\cdot), \btheta_h^* \right\rangle_{\cH}
,
\\
Q_h^{\pi, *}(\cdot) = r_h(\cdot) + (\PP_h V_{h+1}^{\pi, *})(\cdot)
    &=
r_h(\cdot) + \left\langle \phi_{V_{h+1}^{\pi, *}}(\cdot), \btheta_h^* \right\rangle_{\cH}
.
\end{align}

\paragraph{Weighted kernel function.}
In this work, we consider a general RKHS $\cH$ and do not assume that we can access the feature mapping $\phi$ directly. Instead, we assume that we can access the \emph{weighted kernel function} $k_{V_1, V_2}(\cdot, \cdot)$, which is defined as follows:
\begin{definition}\label{def:weighkernel}
For any function pairs $V_1, V_2: \cS \rightarrow [0,1]$ which map states to real numbers, the weighted kernel function $k_{V_1, V_2}(\cdot, \cdot)$ is defined as follows:
\begin{align}
\forall z_1, z_2 \in \cZ,\ k_{V_1, V_2}(z_1, z_2) := \sum_{s_1, s_2 \in \cS}V_1(s_1)V_2(s_2) \left\langle\phi(s_1|z_1), \phi(s_2|z_2)\right\rangle_{\cH}
.\notag
\end{align}
\end{definition}
It is easy to see from Definition \ref{def:weighkernel} that 
\begin{align}
    &k_{V_1, V_2}(z_1, z_2) = \bigg\la\sum_{s_1 \in \cS}V_1(s_1)\phi(s_1|z_1),\sum_{s_2 \in \cS}V_2(s_2)\phi(s_2|z_2) \bigg\ra_{\cH} = \la \phi_{V_1}(z_1), \phi_{V_2}(z_2)\ra_{\cH},\notag
\end{align}
which suggests that the weighted kernel function $k_{V_1, V_2}(\cdot, \cdot)$ captures the inner product relation between $\phi_{V_1}(z_1)$ and $\phi_{V_2}(z_2)$. 
We assume that we can access an integration oracle that can calculate $k_{V_1, V_2}(z_1, z_2)$ for any function $V_1, V_2$ and state-action tuples $z_1, z_2$ efficiently.

\section{Algorithm}\label{sec_algo}
In this section, we introduce our value-targeted iteration algorithm for the zero-sum two-player Markov Game setting with RKHS function approximation. 
We follow the \emph{value-targeted regression} framework and the confidence set design as in UCRL~\citep{jia2020model,ayoub2020model}, and combine the CCE technique~\citep{xie2020cce} to deal with the zero-sum sub-game induced by upper confidence bound (UCB) and lower confidence bound (LCB) value functions. These techniques enable us to adapt the results from the linear setting to the nonlinear RKHS regime~\citep{chowdhury2017kernelized,yang2020function, zhou2020neural} and obtain a structure-dependent regret bound that is both computationally simple and statistically efficient.

      \begin{algorithm}[!tb]
	\caption{\algnameprime}\label{alg:base}
\begin{algorithmic}[1]
\STATE \textbf{Input:} bonus parameter $ \beta>0 $.
\FOR {episode $t=1,2,\ldots,T$}
\FOR {step $h=H,H-1,\ldots,1$} 
\STATE {\blue Calculate $\QU{\cdot, \cdot, \cdot}, \QL{\cdot, \cdot, \cdot}$ as in \eqref{eq:Q_neural_update}}
\STATE Let $\sigma_h^t(\cdot) = \texttt{FIND\_CCE}(\overline{Q}_h^t, \underline{Q}_h^t, \cdot)$
\STATE Let $\overline{V}_h^t(\cdot) = \EE_{(a, b) \sim \sigma_h^t(\cdot)} \overline{Q}_h^t(\cdot, a, b)$ and $\underline{V}_h^t(\cdot) = \EE_{(a, b) \sim \sigma_h^t(\cdot)} \underline{Q}_h^t(\cdot, a, b)$
\ENDFOR
\STATE Receive initial state $x_{1}^{t}$
\FOR {step $h=1,2,\ldots,H$} 
\STATE  Sample $(a_{h}^{t},b_{h}^{t})\sim\sigma_{h}^{t}(x_{h}^{t})$.
\STATE $P_1$ takes action $a_h^t$, $P_2$ takes action $b_h^t$
\STATE  Observe next state $x_{h+1}^{t}$.
\ENDFOR
\ENDFOR
\end{algorithmic}
\end{algorithm}

To find an equilibrium $(\pi^*, \nu^*)$ of the value function $V_1^{\pi, \nu}(x_1)$,
we design an algorithm using value-targeted regression (VTR) and upper/lower confidence bound-based exploration. As the min-player aims to minimize the value function while the max-player wishes to maximize the value function, we use upper confidence bound to encourage the exploration of the max-player and use a lower confidence bound to encourage the exploration of the min-player. Thus we need to define two value functions for the min/max-players respectively, i.e., $\overline{Q}_h^t, \underline{Q}_h^t, \overline{V}_h^t, \underline{V}_h^t$, where we adopt the overline notation for the over-estimation by the max-player and the underline notation for the under-estimation by the min-player. In the following,  we only describe how to estimate the value functions for the max-player; the value functions for the min-player can be estimated analogously. At each round of the game, we solve the following ridge regression problem for minimizing the Bellman error:
\begin{align}
\Wbar &= \min_{\W \in \cH } 
\sum_{\tau = 1}^{t - 1}
\left[ 
 \overline{V}_{h+1}^\tau(x_{h+1}^\tau) - \left\langle \phi_{\overline{V}_{h+1}^\tau}(z_h^\tau), \btheta \right\rangle_{\cH}
\right]^2 
+
\lambda  \norm{\W }^2_{\cH}
.\label{eq:rkhs_ridge}
\end{align}
Note that in~\eqref{eq:rkhs_ridge},  $\overline{V}_{h+1}^\tau$ only depends on the previous trajectories $
\left\{x_i^j, a_i^j, b_i^j: j \in [\tau - 1], i \in [H]\right\}
$. We denote the corresponding $\sigma$-algebra as $\cF_{\tau - 1}$. Thus we have $\overline{V}_{h+1}^\tau \in \cF_{\tau - 1}$. As each  $\overline{V}_{h+1}^\tau(x_{h+1}^\tau)$ can be seen as a stochastic sample of $(\PP_h \overline{V}_{h+1}^\tau)(z_h^\tau)$, the regularized regression problem of the max-player in~\eqref{eq:rkhs_ridge} can be seen as solving a linear bandit problem with context $\phi_{\overline{V}_{h+1}^\tau}(z_h^\tau)$, reward function $(\PP_h \overline{V}_{h+1}^\tau)(z_h^\tau)$ and noise term $\overline{V}_{h+1}^\tau(x_{h+1}^\tau) - (\PP_h \overline{V}_{h+1}^\tau)(z_h^\tau)$. From the solution to the ridge regression problem~\eqref{eq:rkhs_ridge}, we can define the upper/lower confidence bound of the action-value functions $Q_h^{*, \nu}, Q_h^{\pi, *}$ respectively. For  simplicity of notation, we define the vectors $\overline{\phiholder}_h^t := \left( 
\phi_{\overline{V}_{h+1}^1}(z_h^1), \ldots 
\phi_{\overline{V}_{h+1}^{t - 1}}(z_h^{t - 1})
\right)^\top  \in \cH^{t - 1}$.

For a positive parameter $\beta_t>0$ that will be chosen in later analysis, the confidence region centered at $\overline{\btheta}_h^t$ in the RKHS $\cH$ is defined as 
\begin{align}
    \overline{\cC}_h^t = \bigg\{\btheta: \sqrt{
    \lambda \norm{\btheta - \overline{\btheta}_h^t}_{\cH}^2 
        +
    \left\|\left\langle \overline{\phiholder}_h^t , \btheta - \overline{\btheta}_h^t \right\rangle_{\cH}\right\|^2}
     \leq \beta_t\bigg\}.
     \label{eq:region}
\end{align}
We omit the definition of $\underline{\cC}_h^t$ which is an analogue of Eq.~\eqref{eq:region} obtained by changing all overline symbols to underline ones. Based on the confidence regions, we construct an optimistic/pessimistic estimate of $Q_h^{*, \nu}$ as 
\begin{align}
    \overline{Q}_h^t := \Pi_{[-H, H]}\bigg[r_h + \max_{\btheta \in \overline{\cC}_h^t} \left\langle \phi_{\overline{V}_{h+1}^t}, \btheta \right\rangle_{\cH} \bigg]
,\quad
\underline{Q}_h^t := \Pi_{[-H, H]}\left[r_h + \min_{\btheta \in \underline{\cC}_h^t} \left\langle \phi_{\underline{V}_{h+1}^t}, \btheta \right\rangle_{\cH} \right]
,\label{eq:Q_neural_update}
\end{align}
where $\Pi_{[-H, H]}$ is the projection operator onto $[-H, H]$, which is by definition the range of value functions. {\blue For the convenience of conducting an induction argument we define $\overline{V}^{t}_{H+1} = \underline{V}^{t}_{H+1} = 0$, and also $V_{H+1}^{\pi, \nu}(x) = 0$ and $V_{H+1}^{*, \nu^t} = V_{H+1}^{\pi^t, *} = 0$, since there are no more future steps starting from $h = H+1$}. Given the estimation of $\overline{Q}_h^t, \underline{Q}_h^t$, the next step is to estimate the corresponding state value functions $\VUh, \VLh$. We utilize the \texttt{FIND\_CCE} algorithm in~\citet{xie2020cce} to find a coarse-correlated equilibrium of the payoff pair $(\overline{Q}_h^t(z), \underline{Q}_h^t(z))$.

\paragraph{Computational efficiency.}
By substituting the closed-form solutions to the
maximization/minimization problems in \eqref{eq:Q_neural_update}, we can derive the analytic-form for $\overline{Q}_h^t$ and $\underline{Q}_h^t$. Taking $\overline{Q}_h^t$ as an example, we have
\begin{align}
\QU{z} 
&=
\Pi_{[-H, H]}\bigg[
r_h(z) 
+
\overline{k}_h^t(z)^\top (\overline{K}_h^t + \lambda I)^{-1} \overline{y}_h^t
+
\beta_t \cdot \overline{\bonus}_h^t(z)
\bigg]
,
\end{align}
where the Gram matrix $\overline{K}_h^t$ and vector-valued function $\overline{k}_h^t$ are defined as
\begin{align}
&\overline{K}_h^t
=
\left(\overline{\phiholder}_h^t\right)\left(\overline{\phiholder}_h^t\right)^\top \in \RR^{(t - 1) \times (t - 1)}
,\quad
\overline{k}_h^t = \left(\overline{\phiholder}_h^t\right) \phi_{\overline{V}_{h+1}^t}(z) = \left(k_{\overline{V}_{h+1}^i, \overline{V}_{h+1}^t}(z_h^i, z)\right)_i \in \RR^{t-1}.\notag
\end{align}
Also, we have $\overline{y}_h^t := \left[ 
\overline{V}_{h+1}^1(x_h^1), \ldots \overline{V}_{h+1}^{t - 1}(x_h^{t - 1})
\right]^\top$ and $\overline{\bonus}_h^t(z) 
= 
\lambda^{-1/2}\bigg[
k_{\overline{V}_{h+1}^t, \overline{V}_{h+1}^t}(z, z)
- \overline{k}_h^t(z)^\top \big(\overline{K}_h^t + \lambda \cdot \Ib \big)^{-1} \overline{k}_h^t(z) 
\bigg]^{1/2}$.
Therefore, by the assumption that the weighted kernel function $k_{V_1, V_2}$ can be evaluated efficiently, $\overline{Q}_h^t$ and $\underline{Q}_h^t$ can also be computed efficiently. Furthermore, given $\overline{Q}_h^t$ and $\underline{Q}_h^t$, \texttt{FIND\_CCE} can also be implemented efficiently \citep{xie2020cce}. Thus, Algorithm \ref{alg:base} is computationally efficient.


\section{Main Results}\label{sec_RKHS}
In this section, we present the regret bound of our algorithm for the kernel mixture Markov Game.
Recall that for the linear function class, the regret upper bound is characterized by the dimension of the linear function, the horizon of the game, and the number of episodes \citep{chen2021almost}.
Our analysis in the RKHS function approximation setting  aligns with the linear function approximation setting when $K(z, z') = \phi(z)^\top \phi(z')$. 

When considering the nonlinear function class as an approximator of the value function, we need to develop a new concept analogous to the dimension $d$ that characterizes the intrinsic complexity of the function class $\cF$.  We do so by making use of the maximal information gain, $\Gamma_K(T, \lambda)$ \citep{srinivas2009gaussian}, where $T$ is the episode number and $H$ is the time horizon. In particular, we define the \emph{effective dimension} of the RKHS $\cH$ with respect to the mixture MG as follows:
\begin{definition}\label{def:eff}
We define the effective dimension $\Gamma_K(T, \lambda)$ as follows:
\begin{align}
    \Gamma_K(T, \lambda): = \sup_{(V_i)_i, (z_i)_i }\frac{1}{2}\log \det (\Ib + K(\{V_i\}_i, \{z_i\}_i)/\lambda),
\notag\end{align}
for any $1 \leq i \leq T,\ V_i: \cS \rightarrow [-H,H],\ z_i \in \cZ$, where $V_i$'s are functions mapping from $\cS$ to $[-H,H]$ and $z_i$'s are state-action tuples. Here, $K(\{V_i\}_i, \{z_i\}_i) \in \RR^{T \times T}$ and its $(p,q)$-th entry for any $1 \leq p,q \leq T$ is $[K(\{V_i\}_i, \{z_i\}_i)]_{p,q} = k_{V_p, V_q}(z_p, z_q)$.
\end{definition}
By the boundedness of $\phi_V$ as in Section~\ref{sec_prelim_rkhs}, 
it is easy to verify that both the tabular MG and the linear mixture MG enjoy a finite effective dimension.
Specifically, for finite RKHS $\cH$ with rank $d$, $\Gamma_K(T, \lambda) = O(d \cdot \log T)$ approximates the rank of $\cH$.
Via a  concentration argument, we first present our main lemma for bounding the estimation error when choosing $\beta_t = \beta$ for all $t\geq 1$:


\begin{lemma}\label{lem:main}
Assuming that for any $h \in [H]$, $\norm{\btheta_h^*}_{\cH} \leq B$. Let $\lambda = 1 + 1/T$ and $\beta$ satisfies $\left(\beta/H\right)^2 \geq 2\Gamma_K(T, \lambda)
+
2
+
4\cdot \log \left( 1/\delta \right)
+
2 \lambda \left(B/H\right)^2$.
Then, for any $\delta > 0$, with probability at least $1 - \delta$, the following holds for any $(t, h) \in [T] \times [H]$ and any $(x, a, b) \in \cS \times \cA \times \cA$:
\begin{align}
\left| \left\langle \phi_{\overline{V}_{h+1}^t}(x, a, b) , \overline{\btheta}_h^t - \btheta_h^*\right\rangle_{\cH} \right| 
&\leq 
\beta\cdot \overline{\bonus}_h^t(x, a, b)
,\ 
\left| \left\langle \phi_{\underline{V}_{h+1}^t}(x, a, b), \overline{\btheta}_h^t - \btheta_h^*\right\rangle_{\cH} \right| 
\leq 
\beta\cdot \underline{\bonus}_h^t(x, a, b)
.\notag
\end{align}
\end{lemma}


We are now ready to present our main theorem.

\begin{theorem}[RKHS function approximation]\label{thm:main}
Under the same conditions as Lemma \ref{lem:main},
with probability at least $1 - \delta$, $\algname$ has the following regret:
\begin{align}
    \operatorname{Regret}(T)
    = 
O \left(\beta H \sqrt{T \cdot \Gamma_{K}(T, \lambda)} + 1 \right)
.\notag
\end{align}
\end{theorem}
%
\begin{remark}
Theorem \ref{thm:main} suggests that by treating the norm $B$ as a constant, $\algname$ achieves an $\tilde O(\Gamma_{K}(T, \lambda) H^{2}\sqrt{T})$ regret bound. When the RKHS degenerates to the Euclidean space, the regret bound reduces to $\tilde O(dH^{2}\sqrt{T})$, which matches the $\tilde O(dH^{3/2}\sqrt{T})$ regret for linear mixture MGs presented by \citet{chen2021almost} up to a $\sqrt{H}$ factor. 
\end{remark}

Similar to \citet{xie2020cce}, by using a standard online-to-batch conversion technique, we can convert the regret bound in Theorem \ref{thm:main} to a PAC bound. For simplicity, let the initial states of each episode be the same, i.e., $x_1^t = x_1$. After $T$ episodes, we select $t_0 \in [T]$ satisfying
\begin{align}
    t_0 = \argmin_{t \in [T] }\{\overline{V}_1^t(x_1) - \underline{V}_1^t(x_1)\},\label{help:444}
\end{align}
which yields the following sample complexity guarantee for finding an $\epsilon$-approximate NE policy pair $(\pi^{t_0}, \nu^{t_0})$.
\begin{corollary}[Sample complexity]\label{coro:main}
Under the same conditions as Theorem \ref{thm:main}, by setting $T = \\ O\big(\beta^2H^2\Gamma_{K}(T, \lambda)/\epsilon^2\big) = \tilde O\big(H^4\Gamma_{K}^2(T, \lambda)/\epsilon^2\big)$
and selecting $t_0$ as in \eqref{help:444}, the policy pair $(\pi^{t_0}, \nu^{t_0})$ is an $\epsilon$-approximate NE.
\end{corollary}

\section{Bernstein-type Bonus, Misspecification, and Neural Function Approximation}\label{sec_RKHSmis}
In this section, we propose several extensions of \algnameprime. 
Section~\ref{sec:bern11} introduces $\algname$ with a Bernstein-type bonus.
Section~\ref{subsec_RKHSmis} discusses kernel function approximation with misspecification.
Section~\ref{subsec:NNmis} specialize the kernel function approximation with misspecification to the neural function approximation setting.



\subsection{$\algname$ with a Bernstein-type bonus}\label{sec:bern11}
Recall that in \algnameprime, we need to choose $\beta$ in order to calculate the optimistic and pessimistic state-action value functions $\QU{\cdot}, \QL{\cdot}$ defined in \eqref{eq:Q_neural_update}.  The theoretical value of $\beta$ is defined in Lemma~\ref{lem:main}, which controls the uncertainty of the action-value estimate. Such choice of $\beta$ is due to a Hoeffding-type concentration used in the proof of Lemma~\ref{lem:main}. It has been shown in \citet{zhou2020nearly} that by using a Bernstein-type bonus and a sharp analysis based on the total variance lemma, one can obtain an improved algorithm with a tighter regret bound. Following this idea, we propose a $\algnameB$ algorithm, which replaces the Hoeffding-type bonus with a Bernstein-type bonus. 
To demonstrate the construction of the Bernstein-type bonus, we take the max player for example. In particular, we solve the following weighted kernel ridge regression problem:
\begin{align}
\overline{\btheta}_{h, 1}^t &= \min_{\W \in \cH } 
\sum_{\tau = 1}^{t - 1}
\Big[ 
\overline{V}_{h+1}^\tau(x_{h+1}^\tau)  - \left\langle \phi_{\overline{V}_{h+1}^\tau}(z_h^\tau), \btheta \right\rangle_{\cH}
\Big]^2 / \left(\uppvar^\tau\right)^2
+
\lambda_1  \norm{\W }^2_{\cH}
,\label{eq:rkhs_ridge_weighted} 
\end{align}
where the input is the normalized feature mapping $\phi_{\overline{V}_{h+1}^\tau}(z_h^\tau)/\uppvar^\tau$, the output is the normalized value function $\overline{V}_{h+1}^\tau(x_{h+1}^\tau)/\uppvar^\tau$, and the normalization factor $\uppvar^\tau$ is an upper bound on the conditional variance of $\overline{V}_{h+1}^\tau(x_{h+1}^\tau)$. 
It is straightforward to verify that \eqref{eq:rkhs_ridge_weighted} admits a closed-form solution. 
Given that solution, we can compute the upper confidence bound of the action-value functions $Q_h^{*, \nu}$. In detail, we define $\overline{\phiholder}_{h, 1}^t 
:= 
\left( 
\phi_{\overline{V}_{h+1}^1}(z_h^1)/\uppvar^1
,\ldots,
\phi_{\overline{V}_{h+1}^{t - 1}}(z_h^{t - 1})/\uppvar^{t - 1}
\right)^\top  \in \cH^{t - 1}$.
The Gram matrix $\overline{K}_{h, 1}^t$, vector-valued function $\overline{k}_{h, 1}^t$ and 
the confidence region centered at $\overline{\btheta}_{h, 1}^t$ in the RKHS $\cH$ can be calculated the same as in Algorithm \ref{alg:base}, except that $\overline{\phiholder}_{h}^t$ is replaced by $\overline{\phiholder}_{h, 1}^t$. Then the optimistic estimate of the action-value function $Q_h^{*, \nu}$ has the following form:
\begin{align}
\QU{z} 
&=
\Pi_{[-H, H]}\left[
r_h(z) +
\overline{k}_{h, 1}^t(z)^\top (\overline{K}_{h, 1}^t + \lambda I)^{-1} \overline{y}_{h, 1}^t  + \beta_t \cdot \overline{\bonus}_{h, 1}^t(z)
\right],\label{eq:Q_neural_update_weighted}
\end{align}
where $\overline{y}_{h, 1}^t := \left[ 
\overline{V}_{h+1}^1(x_h^1)/\uppvar^1, \ldots \overline{V}_{h+1}^{t - 1}(x_h^{t - 1})/\uppvar^{t - 1}
\right]^\top $
and 
\begin{align}
\overline{\bonus}_{h, 1}^t(z) 
&= 
\lambda_1^{-1/2}\cdot \bigg[
k_{\overline{V}_{h+1}^t, \overline{V}_{h+1}^t}(z, z) - \overline{k}_{h, 1}^t(z)^\top \left(\overline{K}_{h, 1}^t + \lambda_1 \cdot \Ib \right)^{-1} \overline{k}_{h, 1}^t(z) 
\bigg]^{1/2}
.\notag
\end{align}
We defer the presentatino of the conditional variance estimator $\overline{R}_h^t$ to Appendix \ref{sec:bern}. 
Similarly, we can construct the pessimistic estimate of the action-value function $Q_h^{\pi, *}$ for the min player.  
We have the following informal result for $\algnameB$. The full algorithm and its formal guarantee can be found in Appendix \ref{sec:bern}. 

\begin{theorem}[Informal]\label{thm:inform_bernstein}
Let $\deff = \Gamma_K(T, \lambda)
$, with proper choice of $\overline{R}_h^t, \underline{R}_h^t$ and $\beta_t$, with probability at least $1 - \delta$,  $\algnameB$ has the following regret
\begin{align}
\operatorname{Regret}(T)
=
\tilde{O}\bigg(
&
\deff^2 H^3 
+
\sqrt{\deff H^4 + \deff^2 H^3} \sqrt{T} 
+
\left(\deff^7 H^7 + \deff^4 H^9 \right)^{1/4} T^{1/4}
\bigg)
.\notag
\end{align}

\end{theorem}

\begin{remark}
When $T$ is sufficiently large and $\Gamma_K(T, \lambda)$ is larger than $H$, the regret bound in Theorem \ref{thm:inform_bernstein} is dominated by $\tilde O(\Gamma_K(T, \lambda) H^{3/2}\sqrt{T})$, which improves the $\tilde O(\Gamma_K(T, \lambda) H^2\sqrt{T})$ regret derived in Theorem \ref{thm:main} by a factor of $\sqrt{H}$. Compared with the $\tilde \Omega(d H^{3/2}\sqrt{T})$ lower bound proposed in \citet{chen2021almost}, our $\algnameB$ algorithm is almost optimal when it reduces to the linear mixture MG. 
\end{remark}

\subsection{Kernel function approximation with misspecification}\label{subsec_RKHSmis}
In this subsection, we consider the case where the function class may not be confined to an RKHS, but instead the distance to it can be bounded. This can be formulated as kernel function approximation with misspecification.
We assume that there exists a misspecification error between the RKHS $\cH$ and the true transition probability $\PP_h(s' | z)$.

\begin{assumption}\label{assu:misspecification}
There exists an $\iota_{\mis} > 0$, an RKHS $\cH$ with feature mapping $\phi: \cZ \mapsto \cS \times \cH$, and an unknown parameter $\btheta_h^* \in \cH$ satisfying $\left\|\btheta_h^*\right\|_{\cH} \leq B$ such that for any $h \in [H]$, the distance of the transition probability $\PP_h$ to $\cH$ can be bounded by $\iota_{\mis}$, which is $\left\|\PP_h(\cdot | z) - \left\langle \phi(\cdot | z), \btheta_h^* \right\rangle_{\cH} \right\|_{\text{TV}} \leq \iota_{\mis}$.
\end{assumption}

In order to deal with model misspecification, the key idea is to enlarge $\beta_t$ in the definition of the optimistic action-value function in \eqref{eq:Q_neural_update}. More specifically, we will add an extra $\cO(H \iota_{\mis} \sqrt{t})$ term brought by misspecification error to $\beta$ specified in Lemma~\ref{lem:main}. We can show that $\algname$ with such enlarged $\beta$ will have a sublinear regret in the presence of misspecification.

\begin{theorem}[RKHS function approximation with misspecification]\label{thm:main_RKHSmis}
Assuming that for any $h \in [H]$, $\norm{\btheta_h^*}_{\cH} \leq B$. Set $\lambda = 1 + 1/T$ in the $\OURALGO$ Algorithm.
For any $\delta > 0$ and any $\beta_t$ satisfying $\left(\beta_t/H\right)^2 \geq 2\Gamma_K(T, \lambda)
+
3
+
6\cdot \log \left(1/\delta \right)
+
3 \lambda \left(B/H\right)^2
+
3 \iota_{\mis}^2 t$, 
there exists a global constant $c > 0$ such that with probability at least $1 - \delta$, we have 
\begin{align}
    \operatorname{Regret}(T)
    \leq 
c \left(\beta_T H \sqrt{T \cdot \Gamma_{K}(T, \lambda)} + 1 + H^2 T \iota_{\mis} \right)
.\notag
\end{align}
\end{theorem}
In words, Theorem \ref{thm:main_RKHSmis} suggests that in the misspeficified case, $\algname$ can achieve the same regret as that in the well-specified case up to an $O(\sqrt{\Gamma_{K}(T, \lambda)}H^2T\iota_{\mis})$ error. Such a linear dependence on $\iota_{\mis}$ matches the result of single agent RL for the finite-dimensional case \citep{jin2019provably, zanette2020learning}.  



\begin{algorithm}[!tb]
	\caption{NeuralCCE-VTR}
\begin{algorithmic}[1]
\STATE \textbf{Input:} bonus parameter $ \beta_t>0 $.
\FOR {episode $t=1,2,\ldots,T$}
\STATE Receive initial state $x_{1}^{t}$
\FOR {step $h=H,H-1,\ldots,1$} 
\STATE Solve the optimization problem~\eqref{eq:nn_ridge}
\STATE Calculate $\QU{\cdot}, \QL{\cdot}$ as in Eq.~\eqref{eq:Q_nn_update}
\STATE For each $x$, let $\sigma_h^t(x) = \texttt{FIND\_CCE}(\overline{Q}_h^t, \underline{Q}_h^t, x)$
\STATE Let $\overline{V}_h^t(x_h^t) = \EE_{(a, b) \sim \sigma_h^t(x_h^t)} \overline{Q}_h^t(x_h^t, a, b)$ and $\underline{V}_h^t(x_h^t) = \EE_{(a, b) \sim \sigma_h^t(x_h^t)} \underline{Q}_h^t(x_h^t, a, b)$
\ENDFOR
\FOR {step $h=1,2,\ldots,T$} 
\STATE  Sample $(a_{h}^{t},b_{h}^{t})\sim\sigma_{h}^{t}(x_{h}^{t})$.
\STATE $P_1$ takes action $a_h^t$, $P_2$ takes action $b_h^t$
\STATE  Observe next state $x_{h+1}^{t}$.
\ENDFOR
\ENDFOR
\end{algorithmic}
\label{algo:neural}
\end{algorithm}

\subsection{Neural function approximation}\label{subsec:NNmis}
In this subsection, we provide details for the application of our algorithm to the neural function approximation setting, showing that neural network (NN) function approximation can be treated as a special case of kernel function approximation with misspecification. 

We denote $z:= (x, a, b)$ as a vector in $\RR^d$ that satisfies $\norm{z} = 1$
and represent the parameters of a $L$-layer fully connected neural network $f$ by $\btheta := \left[\vec(\Wholder_1)^\top, \vec(\Wholder_2)^\top, \ldots, \vec(\Wholder_L)^\top  \right]^\top$, where $\Wb_1 \in \RR^{m \times d}$, $\Wb_l \in \RR^{m \times m}$ for $2 \leq l \leq L-1$ and $\Wb_L \in \RR^{1 \times m}$. The neural network $f(z; \btheta)$ with parameter set $\btheta$ can be defined as:
\begin{align*}f(z; \btheta)=
\sqrt{m} \Wholder_{L} G\left(\cdots G\left( \Wholder_2 G\left(\Wholder_1  z \right) \right)\right),
\end{align*}
where $G(\cdot): \RR \mapsto \RR $ is an activation function. For $1 \leq l \leq L-1$, $\Wb_l$ is initialized as $\Wb_l = (\Wb, \zero; \zero, \Wb)$, where each entry of $\Wb$ is generated independently from normal distribution $N(0, 4/m)$; $\Wb_L$ is initialized as $\Wb_L = (\bw^\top, -\bw^\top)$, where each entry of $\bw$ is generated independently from $N(0, 2/m)$. Given the initialized parameter $\btheta^{(0)}$, we choose the feature map as the gradient of $f$ at $\btheta^{(0)}$:
\begin{align*}\phi(z) = \nabla_{\btheta} f(z;\btheta^{(0)})/\sqrt{m}.\end{align*}  
Then we define the weighted kernel function $k_{V_1, V_2}(\cdot, \cdot)$ in Definition \ref{def:weighkernel} with $\phi(z)$. Similarly, we define the effective dimension $\Gamma_K(T, \lambda)$ with respect to the kernel function $k_{V_1, V_2}(\cdot, \cdot)$, in the same fashion of Definition \ref{def:eff}. Our assumption is that for $\forall h \in [H]$ our transition probability $\PP_h$ can be modeled by the neural network with parameter $\btheta_h^*$ satisfying $\left\|\btheta_h^* - \btheta^{(0)}\right\|_2 \leq B$:
\begin{align*}
\PP_h(x' | z) = f(x', z; \btheta_h^*)
.
\end{align*}

We now explicate our algorithm, shown formally in Algorithm~\ref{algo:neural}. As in Eq.~\eqref{eq:rkhs_ridge}, we solve penalized ridge regression problems for the min-player and the max-player respectively:
\begin{align}
\Wbar &= \min_{\W \in \RR^{P}} 
\sum_{\tau = 1}^{t - 1}
\left[ 
 \overline{V}_{h+1}^\tau(x_{h+1}^\tau) - f_{\overline{V}_{h+1}^\tau}(z_h^\tau; \W)
\right]^2 
+
\lambda \cdot \norm{\W - \Winit}^2
,\notag
\\
\Wubar &= \min_{\W \in \RR^{P}} 
\sum_{\tau = 1}^{t - 1}
\left[ 
 \underline{V}_{h+1}^\tau(x_{h+1}^\tau) - f_{\underline{V}_{h+1}^\tau}(z_h^\tau; \W)
\right]^2 
+
\lambda \cdot \norm{\W - \Winit}^2
,\label{eq:nn_ridge}
\end{align} 
where $p = md + m^2(L-2) +m$ is the dimension of the parameter space, and $f_{\overline{V}_{h+1}^\tau}, f_{\underline{V}_{h+1}^\tau}$ are defined similarly as $\phi_{\overline{V}_{h+1}^\tau}$ as follows:
\begin{align*}
f_{\overline{V}_{h+1}^\tau}(z; \btheta) &= \sum_{s' \in \cS}\overline{V}_{h+1}^\tau(s') f(s', z; \btheta)
,\qquad 
f_{\underline{V}_{h+1}^\tau}(z; \btheta)= \sum_{s' \in \cS} \underline{V}_{h+1}^\tau(s') f(s', z; \btheta)
.
\end{align*}
For given $\Wbar, \Wubar$, we define 
\begin{align}
    &\overline{\phiholder}_h^t := \left( 
\phi_{\overline{V}_{h+1}^1}(z_h^1; \overline{\W}_h^{2}), \ldots 
\phi_{\overline{V}_{h+1}^{t -1}}(z_h^{t - 1}; \overline{\W}_h^t)
\right)^\top 
, \notag \\
&\underline{\phiholder}_h^t := \left( 
\phi_{\underline{V}_{h+1}^{1}}(z_h^1; \underline{\W}_h^{2}), \ldots 
\phi_{\underline{V}_{h+1}^{t -1}}(z_h^{t - 1}; \underline{\W}_h^t)
\right)^\top 
.
\end{align}
Furthermore, 
\begin{align*}
\overline{\Lambda}_h^t 
:= 
\lambda \cdot \Ib
+
(\overline{\phiholder}_h^t)^\top 
\overline{\phiholder}_h^t
, \qquad
\underline{\Lambda}_h^t 
:= 
\lambda \cdot \Ib
+
(\underline{\phiholder}_h^t)^\top 
\underline{\phiholder}_h^t
,
\end{align*}
and 
\begin{align}
    &\overline{\bonus}_h^t(z)
:= 
\left[ 
\ophi(z; \overline{\W}_h^t)^\top (\overline{\Lambda}_h^t)^{-1} \ophi(z; \overline{\W}_h^t)
\right]^{1/2}
, \notag \\
&\underline{\bonus}_h^t(z)
:= 
\left[ 
\phi_{\underline{V}_{h+1}^t}(z; \underline{\W}_h^t)^\top (\underline{\Lambda}_h^t)^{-1} \phi_{\underline{V}_{h+1}^t}(z; \underline{\W}_h^t)
\right]^{1/2}
.
\end{align}
Using the $\overline{\Lambda}_h^t, \underline{\Lambda}_h^t, \overline{\bonus}_h^t, \underline{\bonus}_h^t$, we estimate the optimal value functions as
\begin{align}
&\QU{z}
=
\Pi_{[-H,H]}\{r_h(z) + f_{\overline{V}_{h+1}^t}(z;\overline{\W}_h^t) + \beta \cdot \overline{\bonus}_h^t(z)\}
,\notag \\
&
\QL{z}
=
\Pi_{[-H,H]}\{r_h(z) +f_{\underline{V}_{h+1}^t}(z; \underline{\W}_h^t) - \beta \cdot \underline{\bonus}_h^t(z)\}
.\label{eq:Q_nn_update}
\end{align} 
Combining with the procedures for finding a CCE, we obtain the full version of our algorithm as in Algorithm~\ref{algo:neural}.

We have the following result on the neural network at initialization.


\begin{lemma}\label{lemm:initiallinear}
There exist constants $C_i >0$ such that for any $\delta \in (0,1)$, if $B$ satisfies that
\begin{align}
    &B \geq C_1m^{-1}L^{-3/2}\max\{\log^{-3/2}m, \log^{3/2}(|\cZ|HL^2/\delta)\},\notag \\
    &B \leq C_2 L^{-6}(\log m)^{-3/2},\notag
\end{align}
then with probability at least $1-\delta$, we have for all $z \in \cZ$, $h \in [H]$ and $V_h: \cS\rightarrow [-1, 1]$, 
\begin{align}
    &
    |\PP_h V_h(z) - \la \bphi_{V_h}(z), \btheta_h^* - \btheta^{(0)}\ra|  
    \leq C_3 |\cS|B^{4/3} m^{-1/6}L^3\sqrt{\log m}
,\notag
\end{align}
and
\begin{align}
    &
    \|\phi_{V_h}(z)\|_2 \leq C:=C_4|\cS|\sqrt{L}
.\notag
\end{align}
\end{lemma}
\noindent
Lemma \ref{lemm:initiallinear} suggests that in the NN approximation setting, Assumption \ref{assu:misspecification} for the misspecified kernel approximation setting is satisfied with $\iota_{\mis} = C_3 |\cS|B^{4/3} m^{-1/6}L^3\sqrt{\log m}$ and with probability at least $1-\delta$. 
The misspecified error is sufficiently small when $m$ is large.
We note that the definition of $\phi(z)$ in the NN setting does not match the boundedness assumption in Section~\ref{sec_prelim_rkhs}. We balance the scale of $\phi(z)$ by the constant $C$ in Lemma~\ref{lemm:initiallinear} which goes into the choice of $\lambda = C^2(1 + 1/T)$.
With these choices in hand, we are ready to present our main result for NN approximation.

\begin{theorem}[NN approximation]\label{theo:misspecification}
Let $C$ be the constant in Lemma \ref{lemm:initiallinear}. Assuming that for any $h \in [H]$, $\norm{\btheta_h^* - \btheta^{(0)}}_2 \leq B$. Set $\lambda = C^2 \left(1 + 1/T\right)$ in the \OURALGO Algorithm.
For any $\delta > 0$ and any $\beta_t$ satisfying
\begin{align*}
\left(\frac{\beta_t}{H}\right)^2 
&\geq 
2\Gamma_K(T, \lambda)
+
3
+
6\cdot \log \left( \frac{1}{\delta} \right)
+
3 \lambda \left(\frac{B}{H}\right)^2
+
3 \cdot C^2 \cdot B^{8/3} \cdot m^{-1/12} \cdot t\cdot  \log m
,
\end{align*}
there exists a global constant $c > 0$ such that with probability at least $1 - 2\delta$, we have 
\begin{align*}
\operatorname{Regret}(T)
    &\leq 
c \left(\beta_T H \sqrt{T \cdot \Gamma_{K}(T, \lambda)} + 1+ B^{4/3} H^2 T m^{-1/6} \sqrt{\log m}\right)
.
\end{align*}
\end{theorem}
\noindent
Theorem \ref{theo:misspecification} suggests that when we use an overparameterized deep neural network ($m \gg 1$) to approximate the transition dynamics, $\algname$ achieves an $\tilde O(\Gamma_{K}(T, \lambda) H^{2}\sqrt{T})$ regret, which is of the same order as that in Theorem \ref{thm:main}.

\section{Conclusions}\label{sec_conclu}
In this work, we studied learning for two-player mixture MGs using kernel function approximation. We introduced a new formulation of kernel mixture MGs and proposed an algorithm $\algname$ that exploits the kernel function of the MG. We show that our $\algname$ is able to achieve a sublinear $\tilde O(d_K H^2\sqrt{T})$ regret. 
We further improve our algorithm with a \emph{Bernstein-type bonus} and \emph{weighted kernel ridge regression}, which enjoys a better $\tilde O(d_K H^{3/2}\sqrt{T})$ regret and nearly matches the regret lower bound in~\citet{chen2021almost} when reducing to linear mixture MGs. Finally, we extend our analysis of the basic RKHS setting to a more general nonlinear function approximation setting with misspecification errors and demonstrate that neural networks can be treated as a special instance of this misspecification framework. We believe our framework and analysis greatly broadens the applicability of these function classes for game-theoretic problems. 

\bibliography{ref}

\begin{thebibliography}{}

\bibitem[Abbasi-Yadkori et~al., 2011]{abbasi2011improved}
Abbasi-Yadkori, Y., P{\'a}l, D., \& Szepesv{\'a}ri, C. (2011).
\newblock Improved algorithms for linear stochastic bandits.
\newblock In {\em Advances in Neural Information Processing Systems}  (pp.\
  2312--2320).

\bibitem[Ayoub et~al., 2020]{ayoub2020model}
Ayoub, A., Jia, Z., Szepesvari, C., Wang, M., \& Yang, L.~F. (2020).
\newblock Model-based reinforcement learning with value-targeted regression.
\newblock {\em arXiv preprint arXiv:2006.01107}.

\bibitem[Bai \& Jin, 2020]{bai2020provable}
Bai, Y. \& Jin, C. (2020).
\newblock Provable self-play algorithms for competitive reinforcement learning.
\newblock In {\em International Conference on Machine Learning}  (pp.\
  551--560).: PMLR.

\bibitem[Bai et~al., 2020]{bai2020near}
Bai, Y., Jin, C., \& Yu, T. (2020).
\newblock Near-optimal reinforcement learning with self-play.
\newblock {\em Advances in Neural Information Processing Systems}, 33.

\bibitem[Brown \& Sandholm, 2019]{brown2019superhuman}
Brown, N. \& Sandholm, T. (2019).
\newblock Superhuman {AI} for multiplayer poker.
\newblock {\em Science}, 365(6456), 885--890.

\bibitem[Cao \& Gu, 2019]{cao2019generalization}
Cao, Y. \& Gu, Q. (2019).
\newblock Generalization bounds of stochastic gradient descent for wide and
  deep neural networks.
\newblock {\em Advances in Neural Information Processing Systems}, 32,
  10836--10846.

\bibitem[Chen et~al., 2021]{chen2021almost}
Chen, Z., Zhou, D., \& Gu, Q. (2021).
\newblock Almost optimal algorithms for two-player {M}arkov games with linear
  function approximation.
\newblock {\em arXiv preprint arXiv:2102.07404}.

\bibitem[Chowdhury \& Gopalan, 2017]{chowdhury2017kernelized}
Chowdhury, S.~R. \& Gopalan, A. (2017).
\newblock On kernelized multi-armed bandits.
\newblock In {\em International Conference on Machine Learning}  (pp.\
  844--853).: PMLR.

\bibitem[Cohen et~al., 2021]{cohen2021solving}
Cohen, M.~B., Lee, Y.~T., \& Song, Z. (2021).
\newblock Solving linear programs in the current matrix multiplication time.
\newblock {\em Journal of the ACM (JACM)}, 68(1), 1--39.

\bibitem[Cui \& Yang, 2020]{cui2020minimax}
Cui, Q. \& Yang, L.~F. (2020).
\newblock Minimax sample complexity for turn-based stochastic game.
\newblock {\em arXiv preprint arXiv:2011.14267}.

\bibitem[Du et~al., 2021]{du2021bilinear}
Du, S., Kakade, S., Lee, J., Lovett, S., Mahajan, G., Sun, W., \& Wang, R.
  (2021).
\newblock Bilinear classes: A structural framework for provable generalization
  in rl.
\newblock In {\em International Conference on Machine Learning}  (pp.\
  2826--2836).: PMLR.

\bibitem[Huang et~al., 2021]{huang2021towards}
Huang, B., Lee, J.~D., Wang, Z., \& Yang, Z. (2021).
\newblock Towards general function approximation in zero-sum {M}arkov games.
\newblock {\em arXiv preprint arXiv:2107.14702}.

\bibitem[Jia et~al., 2020]{jia2020model}
Jia, Z., Yang, L., Szepesvari, C., \& Wang, M. (2020).
\newblock Model-based reinforcement learning with value-targeted regression.
\newblock In {\em Learning for Dynamics and Control}  (pp.\ 666--686).: PMLR.

\bibitem[Jia et~al., 2019]{jia2019feature}
Jia, Z., Yang, L.~F., \& Wang, M. (2019).
\newblock Feature-based {Q}-learning for two-player stochastic games.
\newblock {\em arXiv preprint arXiv:1906.00423}.

\bibitem[Jiang et~al., 2017]{jiang2017contextual}
Jiang, N., Krishnamurthy, A., Agarwal, A., Langford, J., \& Schapire, R.~E.
  (2017).
\newblock Contextual decision processes with low {B}ellman rank are
  {PAC}-learnable.
\newblock In {\em International Conference on Machine Learning}  (pp.\
  1704--1713).: PMLR.

\bibitem[Jin et~al., 2021a]{jin2021bellman}
Jin, C., Liu, Q., \& Miryoosefi, S. (2021a).
\newblock Bellman eluder dimension: New rich classes of {RL} problems, and
  sample-efficient algorithms.
\newblock {\em Advances in Neural Information Processing Systems}, 34.

\bibitem[Jin et~al., 2021b]{jin2021v}
Jin, C., Liu, Q., Wang, Y., \& Yu, T. (2021b).
\newblock V-learning--a simple, efficient, decentralized algorithm for
  multiagent {RL}.
\newblock {\em arXiv preprint arXiv:2110.14555}.

\bibitem[Jin et~al., 2021c]{jin2021power}
Jin, C., Liu, Q., \& Yu, T. (2021c).
\newblock The power of exploiter: Provable multi-agent {RL} in large state
  spaces.
\newblock {\em arXiv preprint arXiv:2106.03352}.

\bibitem[Jin et~al., 2020]{jin2019provably}
Jin, C., Yang, Z., Wang, Z., \& Jordan, M.~I. (2020).
\newblock Provably efficient reinforcement learning with linear function
  approximation.
\newblock In {\em Conference on Learning Theory}  (pp.\ 2137--2143).

\bibitem[Karmarkar, 1984]{karmarkar1984new}
Karmarkar, N. (1984).
\newblock A new polynomial-time algorithm for linear programming.
\newblock In {\em Proceedings of the Sixteenth Annual ACM Symposium on Theory
  of Computing}  (pp.\ 302--311).

\bibitem[Lagoudakis \& Parr, 2002]{lagoudakis2012value}
Lagoudakis, M.~G. \& Parr, R. (2002).
\newblock Value function approximation in zero-sum {M}arkov games.
\newblock In {\em Proceedings of the Eighteenth Conference on Uncertainty in
  Artificial Intelligence}  (pp.\ 283--292).: Morgan Kaufmann Publishers Inc.

\bibitem[Littman, 1994]{littman1994markov}
Littman, M.~L. (1994).
\newblock Markov games as a framework for multi-agent reinforcement learning.
\newblock In {\em Proceedings of the International Conference on Machine
  Learning}  (pp.\ 157--163). Elsevier.

\bibitem[Liu et~al., 2020]{liu2020sharp}
Liu, Q., Yu, T., Bai, Y., \& Jin, C. (2020).
\newblock A sharp analysis of model-based reinforcement learning with
  self-play.
\newblock {\em arXiv preprint arXiv:2010.01604}.

\bibitem[Modi et~al., 2020]{modi2019sample}
Modi, A., Jiang, N., Tewari, A., \& Singh, S. (2020).
\newblock Sample complexity of reinforcement learning using linearly combined
  model ensembles.
\newblock In {\em International Conference on Artificial Intelligence and
  Statistics}  (pp.\ 2010--2020).: PMLR.

\bibitem[Osband \& Van~Roy, 2014]{osband2014model}
Osband, I. \& Van~Roy, B. (2014).
\newblock Model-based reinforcement learning and the eluder dimension.
\newblock {\em arXiv preprint arXiv:1406.1853}.

\bibitem[P{\'e}rolat et~al., 2016a]{perolat2016softened}
P{\'e}rolat, J., Piot, B., Geist, M., Scherrer, B., \& Pietquin, O. (2016a).
\newblock Softened approximate policy iteration for {Markov} games.
\newblock In {\em International Conference on Machine Learning}  (pp.\
  1860--1868).: PMLR.

\bibitem[Perolat et~al., 2018]{perolat2018actor}
Perolat, J., Piot, B., \& Pietquin, O. (2018).
\newblock Actor-critic fictitious play in simultaneous move multistage games.
\newblock In {\em International Conference on Artificial Intelligence and
  Statistics}  (pp.\ 919--928).

\bibitem[P{\'e}rolat et~al., 2016b]{perolat2016use}
P{\'e}rolat, J., Piot, B., Scherrer, B., \& Pietquin, O. (2016b).
\newblock On the use of non-stationary strategies for solving two-player
  zero-sum {M}arkov games.
\newblock In {\em Artificial Intelligence and Statistics}  (pp.\ 893--901).

\bibitem[Perolat et~al., 2015]{perolat2015approximate}
Perolat, J., Scherrer, B., Piot, B., \& Pietquin, O. (2015).
\newblock Approximate dynamic programming for two-player zero-sum {Markov}
  games.
\newblock In {\em International Conference on Machine Learning}  (pp.\
  1321--1329).

\bibitem[P{\'e}rolat et~al., 2017]{perolat2017learning}
P{\'e}rolat, J., Strub, F., Piot, B., \& Pietquin, O. (2017).
\newblock Learning {N}ash equilibrium for general-sum {M}arkov games from batch
  data.
\newblock In {\em Artificial Intelligence and Statistics}  (pp.\ 232--241).:
  PMLR.

\bibitem[Qiu et~al., 2021]{qiu2021reward}
Qiu, S., Ye, J., Wang, Z., \& Yang, Z. (2021).
\newblock On reward-free {RL} with kernel and neural function approximations:
  Single-agent {MDP} and {M}arkov game.
\newblock In {\em International Conference on Machine Learning}  (pp.\
  8737--8747).: PMLR.

\bibitem[Russo \& Van~Roy, 2013]{russo2013eluder}
Russo, D. \& Van~Roy, B. (2013).
\newblock Eluder dimension and the sample complexity of optimistic exploration.
\newblock In {\em NIPS}  (pp.\ 2256--2264).: Citeseer.

\bibitem[Shalev-Shwartz et~al., 2016]{shalev2016safe}
Shalev-Shwartz, S., Shammah, S., \& Shashua, A. (2016).
\newblock Safe, multi-agent, reinforcement learning for autonomous driving.
\newblock {\em arXiv preprint arXiv:1610.03295}.

\bibitem[Shapley, 1953]{shapley1953stochastic}
Shapley, L.~S. (1953).
\newblock Stochastic games.
\newblock {\em Proceedings of the National Academy of Sciences}, 39(10),
  1095--1100.

\bibitem[Sidford et~al., 2020]{sidford2020solving}
Sidford, A., Wang, M., Yang, L., \& Ye, Y. (2020).
\newblock Solving discounted stochastic two-player games with near-optimal time
  and sample complexity.
\newblock In {\em International Conference on Artificial Intelligence and
  Statistics}  (pp.\ 2992--3002).: PMLR.

\bibitem[Silver et~al., 2016]{silver2016mastering}
Silver, D., Huang, A., Maddison, C.~J., Guez, A., Sifre, L., Van Den~Driessche,
  G., Schrittwieser, J., Antonoglou, I., Panneershelvam, V., Lanctot, M.,
  et~al. (2016).
\newblock Mastering the game of {G}o with deep neural networks and tree search.
\newblock {\em Nature}, 529(7587), 484.

\bibitem[Srinivas et~al., 2009]{srinivas2009gaussian}
Srinivas, N., Krause, A., Kakade, S.~M., \& Seeger, M. (2009).
\newblock Gaussian process optimization in the bandit setting: No regret and
  experimental design.
\newblock {\em arXiv preprint arXiv:0912.3995}.

\bibitem[Sun et~al., 2019]{sun2019model}
Sun, W., Jiang, N., Krishnamurthy, A., Agarwal, A., \& Langford, J. (2019).
\newblock Model-based {RL} in contextual decision processes: {PAC} bounds and
  exponential improvements over model-free approaches.
\newblock In {\em Conference on Learning Theory}  (pp.\ 2898--2933).: PMLR.

\bibitem[Vinyals et~al., 2019]{vinyals2019grandmaster}
Vinyals, O., Babuschkin, I., Czarnecki, W.~M., Mathieu, M., Dudzik, A., Chung,
  J., Choi, D.~H., Powell, R., Ewalds, T., Georgiev, P., et~al. (2019).
\newblock Grandmaster level in {StarCraft} {II} using multi-agent reinforcement
  learning.
\newblock {\em Nature}, 575(7782), 350--354.

\bibitem[Wang et~al., 2020]{wang2020reinforcement}
Wang, R., Salakhutdinov, R.~R., \& Yang, L. (2020).
\newblock Reinforcement learning with general value function approximation:
  Provably efficient approach via bounded eluder dimension.
\newblock {\em Advances in Neural Information Processing Systems}, 33.

\bibitem[Wei et~al., 2017]{wei2017online}
Wei, C.-Y., Hong, Y.-T., \& Lu, C.-J. (2017).
\newblock Online reinforcement learning in stochastic games.
\newblock In {\em Advances in Neural Information Processing Systems}  (pp.\
  4987--4997).

\bibitem[Xie et~al., 2020]{xie2020cce}
Xie, Q., Chen, Y., Wang, Z., \& Yang, Z. (2020).
\newblock Learning zero-sum simultaneous-move {M}arkov games using function
  approximation and correlated equilibrium.
\newblock In {\em Conference on Learning Theory, to appear}.

\bibitem[Yang \& Wang, 2020]{yang2019reinforcement}
Yang, L. \& Wang, M. (2020).
\newblock Reinforcement learning in feature space: Matrix bandit, kernels, and
  regret bound.
\newblock In {\em International Conference on Machine Learning}  (pp.\
  10746--10756).: PMLR.

\bibitem[Yang et~al., 2020]{yang2020function}
Yang, Z., Jin, C., Wang, Z., Wang, M., \& Jordan, M.~I. (2020).
\newblock On function approximation in reinforcement learning: Optimism in the
  face of large state spaces.
\newblock {\em arXiv preprint arXiv:2011.04622}.

\bibitem[Zanette et~al., 2020]{zanette2020learning}
Zanette, A., Lazaric, A., Kochenderfer, M., \& Brunskill, E. (2020).
\newblock Learning near optimal policies with low inherent {B}ellman error.
\newblock In {\em International Conference on Machine Learning}  (pp.\
  10978--10989).: PMLR.

\bibitem[Zhou et~al., 2021a]{zhou2020nearly}
Zhou, D., Gu, Q., \& Szepesvari, C. (2021a).
\newblock Nearly minimax optimal reinforcement learning for linear mixture
  {M}arkov decision processes.
\newblock In {\em Conference on Learning Theory}: PMLR.

\bibitem[Zhou et~al., 2021b]{zhou2021provably}
Zhou, D., He, J., \& Gu, Q. (2021b).
\newblock Provably efficient reinforcement learning for discounted {MDP}s with
  feature mapping.
\newblock In {\em International Conference on Machine Learning}  (pp.\
  12793--12802).: PMLR.

\bibitem[Zhou et~al., 2020]{zhou2020neural}
Zhou, D., Li, L., \& Gu, Q. (2020).
\newblock Neural contextual bandits with {UCB}-based exploration.
\newblock In {\em International Conference on Machine Learning}  (pp.\
  11492--11502).: PMLR.

\end{thebibliography}
\bibliographystyle{apalike2}

\newpage\appendix
\onecolumn
\section*{Appendix}
The appendix is organized as follows.
In Appendix \ref{sec:facts} we introduce basic properties of RKHS.
In Appendix \ref{sec:bern} we discuss the implementation details of $\algname+$.
In Appendix \ref{sec:mainproof_1} we prove results for $\algname$.
In Appendix \ref{sec:mainproof_2} we prove results for $\algname+$.
In Appendix \ref{sec:mainproof_3} we prove results for $\algname$ with misspecification.
In Appendix \ref{sec:mainproof_4} we prove results for $\algname$ with neural function approximation.
In Appendix \ref{sec:auxproof} we prove the remaining auxiliary lemmas.
Finally, in Appendix \ref{app:cce} we discuss the implementation details of \texttt{FIND\_CCE} as an instance of linear programming.



\section{Properties of the Reproducing Kernel Hilbert Spaces}\label{sec:facts}
Recall that in Section~\ref{sec_algo}, we define the update rule of $\overline{Q}_h^t, \underline{Q}_h^t$ in Eq.~\eqref{eq:Q_neural_update}, where each term is defined in the sense of computational accessibility. For convenience of theoretical analysis, in this section we provide the equivalent forms of the $Q$-update on the RKHS.
We have the following simple facts:
\begin{lemma}\label{lem:facts}
Define covariance matrices $\overline{\Lambda}_h^t, \underline{\Lambda}_h^t: \cH \mapsto \cH$ as 
\beq\label{eq:def_lambda}
\overline{\Lambda}_h^t := \lambda \cdot \Ib_{\cH} + \left(\overline{\phiholder}_h^t \right)^\top \left(\overline{\phiholder}_h^t \right),
\quad 
\underline{\Lambda}_h^t := \lambda \cdot \Ib_{\cH} + \left(\underline{\phiholder}_h^t \right)^\top \left(\underline{\phiholder}_h^t \right),
\eeq
where $\Ib_{\cH}$ is the identity mapping on $\cH$. Then the following holds:
\begin{enumerate}[label=(\alph*)]
\item 
$\overline{\btheta}_h^t := \left(\overline{\phiholder}_h^t\right)^\top \left[\overline{K}_h^t + \lambda \cdot \Ib \right]^{-1} \overline{\yholder}_h^t = \left(\overline{\Lambda}_h^t \right)^{-1} \left(\overline{\phiholder}_h^t\right)^\top \overline{\yholder}_h^t \in \cH $ and the same holds for $\underline{\btheta}_h^t$;
\item 
$\overline{\bonus}_h^t = \left[\phi_{\overline{V}_{h+1}^t}(z)^\top \overline{\Lambda}_h^t \phi_{\overline{V}_{h+1}^t}(z)\right]^{1/2}$ and the same holds for $\underline{\bonus}_h^t$;
\item 
$\phi_{\overline{V}_{h+1}^t}(z) = \left(\overline{\phiholder}_h^t\right)^\top (\overline{K}_h^t + \lambda \cdot \Ib )^{-1} \overline{k}_h^t(z) + \lambda \cdot (\overline{\Lambda}_h^t)^{-1} \phi_{\overline{V}_{h+1}^t}(z) $.
\end{enumerate}
\end{lemma}
\begin{proof}
We prove the statements as follows. 
\begin{enumerate}[label = (\alph*)]
\item By definition of $\overline{K}_h^t$ in Section~\ref{sec_algo}, we note that 
\begin{align}
\left(\overline{\phiholder}_h^t\right)^\top \left[\overline{K}_h^t + \lambda \cdot \Ib \right] 
	&=
 \left(\overline{\phiholder}_h^t\right)^\top \left[\left( \overline{\phiholder}_h^t \right)\left( \overline{\phiholder}_h^t \right)^\top + \lambda \cdot  \Ib \right] \notag \\
	&=
\left[ \left(\overline{\phiholder}_h^t\right)^\top \left( \overline{\phiholder}_h^t \right) + \lambda \cdot \Ib_{\cH} \right] \left(\overline{\phiholder}_h^t\right)^\top.\notag
\end{align}
Taking the inverse operation on both sides of the second equality, we conclude
\begin{align}
 \left[\left( \overline{\phiholder}_h^t \right)\left( \overline{\phiholder}_h^t \right)^\top + \lambda \cdot  \Ib \right]^{-1}  \left(\overline{\phiholder}_h^t\right)^{-\top} 
 	=
 \left(\overline{\phiholder}_h^t\right)^{-\top} \left[ \left(\overline{\phiholder}_h^t\right)^\top \left( \overline{\phiholder}_h^t \right) + \lambda \cdot \Ib_{\cH} \right]^{-1},\notag
\end{align}
and hence we arrive at the following equality on the space $\cH \times \RR^t$:
\begin{align}
\left(\overline{\phiholder}_h^t\right)^\top \left[\overline{K}_h^t + \lambda \cdot \Ib \right]^{-1} 
 	&=
\left(\overline{\phiholder}_h^t\right)^{\top} \left[\left( \overline{\phiholder}_h^t \right)\left( \overline{\phiholder}_h^t \right)^\top + \lambda \cdot  \Ib \right]^{-1}  \notag \\
 	&=
\left[ \left(\overline{\phiholder}_h^t\right)^\top \left( \overline{\phiholder}_h^t \right) + \lambda \cdot \Ib_{\cH} \right]^{-1}\left(\overline{\phiholder}_h^t\right)^{\top} \notag
 	=
\left(\overline{\Lambda}_h^t \right)^{-1} \left(\overline{\phiholder}_h^t\right)^\top
.\notag
\end{align}
Multiplying both sides by $\overline{\yholder}_h^t$ we have that the closed-form solution of Eq.~\eqref{eq:rkhs_ridge} satisfies
\begin{align}
    \overline{\btheta}_h^t 
:= 
\left(\overline{\phiholder}_h^t\right)^\top \left[\overline{K}_h^t + \lambda \cdot \Ib \right]^{-1} \overline{\yholder}_h^t 
= 
\left(\overline{\Lambda}_h^t \right)^{-1} \left(\overline{\phiholder}_h^t\right)^\top \overline{\yholder}_h^t \in \cH
,\notag
\end{align}
which proves item $(a)$. The same argument holds for $\underline{\btheta}_h^t$.

\item By definition of $\overline{\bonus}_h^t$, $\overline{k}_h^t$ and $\overline{K}_h^t$, we have
\begin{align*}
\overline{\bonus}_h^t(z) 
	&= 
\lambda^{-1/2}\cdot \bigg[
k_{\overline{V}_{h+1}^t, \overline{V}_{h+1}^t}(z, z)  - \overline{k}_h^t(z)^\top \left(\overline{K}_h^t + \lambda \cdot \Ib \right)^{-1} \overline{k}_h^t(z) 
\bigg]^{1/2}
	\\&=
\lambda^{-1/2}\cdot \bigg[
k_{\overline{V}_{h+1}^t, \overline{V}_{h+1}^t}(z, z)  - \phi_{\overline{V}_{h+1}^t}^\top(z)\left( \overline{\phiholder}_h^t\right)^\top \left(\overline{K}_h^t + \lambda \cdot \Ib \right)^{-1}  \left(\overline{\phiholder}_h^t\right) \phi_{\overline{V}_{h+1}^t}(z) 
\bigg]^{1/2}
	\\&=
\lambda^{-1/2}\cdot \bigg[
k_{\overline{V}_{h+1}^t, \overline{V}_{h+1}^t}(z, z)  - \phi_{\overline{V}_{h+1}^t}^\top(z) \left( \overline{\Lambda}_h^t\right)^{-1}  \left( \overline{\phiholder}_h^t\right)^\top \left(\overline{\phiholder}_h^t\right) \phi_{\overline{V}_{h+1}^t}(z) 
\bigg]^{1/2}
	\\&= 
\lambda^{-1/2}\cdot \bigg[
 \phi_{\overline{V}_{h+1}^t}^\top \left( \overline{\Lambda}_h^t\right)^{-1}  \left( \overline{\Lambda}_h^t\right) \phi_{\overline{V}_{h+1}^t}(z)  - \phi_{\overline{V}_{h+1}^t}^\top \left( \overline{\Lambda}_h^t\right)^{-1}  \left( \overline{\phiholder}_h^t\right)^\top \left(\overline{\phiholder}_h^t\right) \phi_{\overline{V}_{h+1}^t}(z) 
\bigg]^{1/2}
	\\&= 
\left[\phi_{\overline{V}_{h+1}^t}(z) (\overline{\Lambda}_h^t)^{-1} \phi_{\overline{V}_{h+1}^t}(z) \right]^{1/2}
.
\end{align*}
This concludes the proof of item $(b)$. The same argument holds for $\underline{\bonus}_h^t(z)$.

\item 
Noting that from the definition of $\overline{\Lambda}_h^t$ in Eq.~\eqref{eq:def_lambda}, 
\begin{align*}
\phi_{\overline{V}_{h+1}^t}(z)
	&= 
 \left( \overline{\Lambda}_h^t\right)^{-1} \left( \overline{\Lambda}_h^t\right) \phi_{\overline{V}_{h+1}^t}(z)
	=
 \left( \overline{\Lambda}_h^t\right)^{-1} \left(\lambda \cdot \Ib_{\cH}  + \left( \overline{\phiholder}_h^t  \right)^\top \left( \overline{\phiholder}_h^t  \right)  \right) \phi_{\overline{V}_{h+1}^t}(z)
 	\\&= 
 \left( \overline{\Lambda}_h^t\right)^{-1} \left( \overline{\phiholder}_h^t  \right)^\top \left( \overline{\phiholder}_h^t  \right) \phi_{\overline{V}_{h+1}^t}(z)
 	+
\lambda \cdot \left( \overline{\Lambda}_h^t\right)^{-1} \phi_{\overline{V}_{h+1}^t}(z)
.
\end{align*}
Applying the results in the proof of item $(a)$ on $ \left( \overline{\Lambda}_h^t\right)^{-1} \left( \overline{\phiholder}_h^t  \right)^\top $, we have that 
\begin{align}
\phi_{\overline{V}_{h+1}^t}(z) 
	&=  \left( \overline{\phiholder}_h^t  \right)^\top \left[ \overline{K}_h^t + \lambda \cdot \Ib\right]^{-1}  \left( \overline{\phiholder}_h^t  \right) \phi_{\overline{V}_{h+1}^t}(z)
 	+
\lambda \cdot \left( \overline{\Lambda}_h^t\right)^{-1} \phi_{\overline{V}_{h+1}^t}(z)
	\\&=
 \left( \overline{\phiholder}_h^t  \right)^\top \left[ \overline{K}_h^t + \lambda \cdot \Ib\right]^{-1}  \overline{k}_h^t(z)
 	+
\lambda \cdot \left( \overline{\Lambda}_h^t\right)^{-1} \phi_{\overline{V}_{h+1}^t}(z)
,\notag
\end{align}
which concludes the proof of item $(c)$.
\end{enumerate}
\end{proof}

\begin{algorithm}[!tb]
	\caption{\algnameB}\label{alg:bern}
\begin{algorithmic}[1]
\STATE \textbf{Input:} bonus parameter $ \lambda_1, \lambda_2>0 $.
\FOR {episode $t=1,2,\ldots,T$}
\STATE Receive initial state $x_{1}^{t}$
\FOR {step $h=H,H-1,\ldots,1$} 
\STATE Estimate $\uppvar^t, \lowvar^t$ as in Eq.~\eqref{eq:variance_estimate}
\STATE Calculate $\QU{\cdot}, \QL{\cdot}$ as in Eq.~\eqref{eq:Q_neural_update_weighted}
\STATE For each $x$, let $\sigma_h^t(x) = \texttt{FIND\_CCE}(\overline{Q}_h^t, \underline{Q}_h^t, x)$
\STATE Let $\overline{V}_h^t(x) = \EE_{(a, b) \sim \sigma_h^t(x)} \overline{Q}_h^t(x, a, b)$ and $\underline{V}_h^t(x) = \EE_{(a, b) \sim \sigma_h^t(x)} \underline{Q}_h^t(x, a, b)$
\ENDFOR
\FOR {step $h=1,2,\ldots,T$} 
\STATE  Sample $(a_{h}^{t},b_{h}^{t})\sim\sigma_{h}^{t}(x_{h}^{t})$.
\STATE $P_1$ takes action $a_h^t$, $P_2$ takes action $b_h^t$
\STATE  Observe next state $x_{h+1}^{t}$.
\ENDFOR
\ENDFOR
\end{algorithmic}
\label{algo:neural_weighted}
\end{algorithm}

\section{Details of $\algnameB$}\label{sec:bern}
In this section, we present more details for the algorithm \algnameB. 
We consider the following ridge regression problem where each term is weighted by its estimated variance:
\begin{align*}
\overline{\btheta}_{h, 1}^t &= \min_{\W \in \cH } 
\sum_{\tau = 1}^{t - 1}
\left[ 
 \overline{V}_{h+1}^\tau(x_{h+1}^\tau) - \left\langle \phi_{\overline{V}_{h+1}^\tau}(z_h^\tau), \btheta \right\rangle_{\cH}
\right]^2/ \left(\uppvar^\tau\right)^2 +\lambda_1  \norm{\W }^2_{\cH}
,\quad \text{and}
\\
\underline{\btheta}_{h, 1}^t &= \min_{\W \in \cH} 
\sum_{\tau = 1}^{t - 1}
\left[ 
\underline{V}_{h+1}^\tau(x_{h+1}^\tau) - \left\langle \phi_{\underline{V}_{h+1}^\tau}(z_h^\tau), \btheta \right\rangle_{\cH}
\right]^2 /\left(\lowvar^\tau\right)^2 +
\lambda_1  \norm{ \W}^2_{\cH}
. 
\end{align*}
Here we use $\uppvar^\tau, \lowvar^\tau$ to denote upper bounds on the conditional variance of $\overline{V}_{h+1}^\tau(x_{h+1}^\tau)$ and $\underline{V}_{h+1}^\tau(x_{h+1}^\tau)$ respectively, which we will specify in later subsections.
Next we define the necessary quantities in estimating the regret bound. Similarly as in previous sections, we define
\begin{align*}
\overline{\phiholder}_{h, 1}^t 
&:= 
\left( 
\phi_{\overline{V}_{h+1}^1}(z_h^1)/\uppvar^1, \ldots 
\phi_{\overline{V}_{h+1}^{t - 1}}(z_h^{t - 1})/\uppvar^{t - 1}
\right)^\top  \in \cH^{t - 1} 
,\quad \text{and}
\\
\underline{\phiholder}_{h, 1}^t 
&:=  
\left( 
\phi_{\underline{V}_{h+1}^1}(z_h^1)/\lowvar^1, \ldots 
\phi_{\underline{V}_{h+1}^{t - 1}}(z_h^{t - 1})/\lowvar^{t - 1}
\right)^\top \in \cH^{t - 1} 
.\notag 
\end{align*}
The Gram matrix $\overline{K}_{h, 1}^t$, vector-valued function $\overline{k}_{h, 1}^t$ and 
the confidence region centered at $\overline{\btheta}_{h, 1}^t$ in the RKHS $\cH$ are defined by replacing $\overline{\phiholder}_{h}^t, \underline{\phiholder}_{h}^t$ by $\overline{\phiholder}_{h, 1}^t, \underline{\phiholder}_{h, 1}^t$ respectively.
The optimistic (pessimistic version can be defined accordingly) estimates of the action-value function have the following closed-form solution:
\begin{align}\label{eq:Q_rkhs_update_weighted}
 \QU{z} 
    &= \Pi_{[-H, H]}[r_h(z) +
\overline{k}_{h, 1}^t(z)^\top (\overline{K}_{h, 1}^t + \lambda I)^{-1} \overline{y}_{h, 1}^t  + \beta_t \cdot \overline{\bonus}_{h, 1}^t(z)],
\end{align}
where 
\begin{align*}
&
\overline{y}_{h, 1}^t := \left[ 
\overline{V}_{h+1}^1(x_h^1)/\uppvar^1, \ldots \overline{V}_{h+1}^{t - 1}(x_h^{t - 1})/\uppvar^{t - 1}
\right]^\top
,
\end{align*}
and
\begin{align*}
\overline{\bonus}_{h, 1}^t(z) 
&= 
\lambda_1^{-1/2}\cdot \bigg[
k_{\overline{V}_{h+1}^t, \overline{V}_{h+1}^t}(z, z) - \overline{k}_{h, 1}^t(z)^\top \left(\overline{K}_{h, 1}^t + \lambda_1 \cdot \Ib \right)^{-1} \overline{k}_{h, 1}^t(z) 
\bigg]^{1/2}
.
\end{align*}
The full version of the algorithm is presented formally in Algorithm~\ref{algo:neural_weighted}.

\subsection{Variance estimator}
In order to determine the values of $\uppvar^\tau, \lowvar^\tau$, we note that we can solve a ridge regression problem for estimating the expected square of the value function:

\begin{align*}
\overline{\btheta}_{h, 2}^t &= \min_{\W \in \cH } 
\sum_{\tau = 1}^{t - 1}
\left[ 
 \left(\overline{V}_{h+1}^\tau(x_{h+1}^\tau)\right)^2 - \left\langle \phi_{(\overline{V}_{h+1}^\tau)^2}(z_h^\tau), \btheta \right\rangle_{\cH}
\right]^2 
+
\lambda_2  \norm{\W }^2_{\cH}
,
\\
\underline{\btheta}_{h, 2}^t &= \min_{\W \in \cH} 
\sum_{\tau = 1}^{t - 1}
\left[ 
\left(\underline{V}_{h+1}^\tau(x_{h+1}^\tau)\right)^2 - \left\langle \phi_{(\underline{V}_{h+1}^\tau)^2}(z_h^\tau), \btheta \right\rangle_{\cH}
\right]^2 
+
\lambda_2  \norm{ \W}^2_{\cH}
.
\end{align*}
By defining 
\begin{align*}
\overline{\phiholder}_{h, 2}^t
&:=
\left( 
\phi_{(\overline{V}_{h+1}^1)^2}(z_h^1), \ldots 
\phi_{(\overline{V}_{h+1}^{t - 1})^2}(z_h^{t - 1})
\right)^\top  \in \cH^{t - 1}
,
\\
\underline{\phiholder}_{h, 2}^t
&:=
\left( 
\phi_{(\underline{V}_{h+1}^1)^2}(z_h^1), \ldots 
\phi_{(\underline{V}_{h+1}^{t - 1})^2}(z_h^{t - 1})
\right)^\top \in \cH^{t - 1} 
,
\end{align*}
we can define the Gram matrix $\overline{K}_{h, 2}^t$, vector-valued function $\overline{k}_{h, 2}^t$, and
\begin{align}
\overline{\bonus}_{h, 2}^t(z) 
&= 
\lambda_2^{-1/2}\cdot \bigg[
k_{(\overline{V}_{h+1}^t)^2, (\overline{V}_{h+1}^t)^2}(z, z)
- \overline{k}_{h, 2}^t(z)^\top \left(\overline{K}_{h, 2}^t + \lambda_2 \cdot \Ib \right)^{-1} \overline{k}_{h, 2}^t(z) 
\bigg]^{1/2}
.\notag
\end{align}
The variance estimator is thus defined as:
\begin{align}
\Vest \overline{V}_{h+1}^t (z_h^t)
  & :=
\left\langle \phi_{(\overline{V}_{h+1}^t)^2}(z_h^t), \overline{\btheta}_{h, 2}^t \right\rangle_{\cH}
    -
\left(\left\langle \phi_{(\overline{V}_{h+1}^t)^2}(z_h^t), \overline{\btheta}_{h, 1}^t \right\rangle_{\cH}\right)^2\notag
    \\&\approx 
\PP_h \left(\underline{V}_{h+1}^t(x_{h+1}^t)\right)^2
    -
\left(\PP_h \underline{V}_{h+1}^t(x_{h+1}^t)\right)^2,\notag
\end{align}
and 
\begin{align}\label{eq:variance_estimate}
&\left(\uppvar^t \right)^2
    := 
\max\{\Vest \overline{V}_{h+1}^t (z_h^t) 
    +
\overline{E}_h^t, \left(\alpha_t\right)^2\},\notag \\
&
\overline{E}_h^t := 
\min\left\{ 
H^2, \beta_t^{(2)} \overline{\bonus}_{h, 2}^t
\right\}
    +
\min\left\{
H^2, 2H \beta_t^{(1)} \overline{\bonus}_{h, 1}^t
\right\}
.
\end{align}
We have finished the definition of the variance estimator for the upper-value estimator.
The lower-value estimator can be defined in a similar fashion, and we omit the details.

\subsection{Main results}
In this section, we provide theoretical results for the regret bound under the weighted setting described above. First we propose a key lemma which suggests that our constructed $\overline{\btheta}_{h, 1}^t$ and $\overline{\btheta}_{h, 2}^t$ are good estimates of $\btheta_h^*$ with high probability. 


\begin{lemma}\label{lem:main_weighted}
Assume that for any $h \in [H]$, we have $\norm{\btheta_h^*}_{\cH} \leq B$. Letting $\alpha_t, \beta_t^{(1)}$, $\beta_t^{(2)}$ satisfy $\alpha_t = \alpha$, 
\begin{align}
&
\beta_t^{(1)} 
= 
(16 H/\alpha) \sqrt{\Gamma_K(T, \lambda_1 \alpha^2)} \sqrt{\log(4t^2 H/\delta)} 
+ 
(8 H/\alpha) \log(4t^2 H/\delta) + \sqrt{\lambda_1} \cdot B
,\label{eq:beta1_choice} \\
&
\beta_t^{(2)} 
= 
16 H^2 \sqrt{\Gamma_K(T, \lambda_2/H^2)} \sqrt{\log(4t^2 H/\delta)}+ 8 H^2 \log (4t^2 H/\delta) + \sqrt{\lambda_2} \cdot B
,\label{eq:beta2_choice}
\end{align}
then for any $\delta > 0$, there exists an event $\mathcal{E}$ satisfying $\PP(\mathcal{E}) \geq 1 - 2\delta$ such that on $\mathcal{E}$ the following holds for any $(t, h) \in [T] \times [H]$ and any $(x, a, b) \in \cS \times \cA \times \cA$:
\begin{align*}
\Big|\langle\bphi_{\overline{V}_{h+1}^t}(z_h^t), \btheta^{*}_{h} - \overline{\btheta}_{h, 1}^t\rangle_{\cH}\Big|
    \leq 
\beta_t^{(1)}
\cdot 
\overline{\bonus}_{h, 1}^k(z_h^k)
,
\end{align*}
and
\begin{align*}
\Big|\langle\bphi_{(\overline{V}_{h+1}^t)^2}(z_h^t), \btheta^{*}_{h} - \overline{\btheta}_{h, 2}^t\rangle_{\cH}\Big|
    \leq 
\beta_t^{(2)} 
\cdot 
\overline{\bonus}_{h, 2}^k(z_h^k)
.
\end{align*}
\end{lemma}
We now propose our main theorem, which is the formal version of Theorem~\ref{thm:inform_bernstein} and suggests that the regret bound of Algorithm~\ref{alg:bern} is upper bounded by $\tilde{O}\left(\Gamma_K(T, \lambda) H^2 \sqrt{T}\right)$.

\begin{theorem}\label{thm:main_bernstein}
Assuming that for any $h \in [H]$, $\norm{\btheta_h^*}_{\cH} \leq B$. Let $\lambda = 1/B^2$, 
$\deff = \Gamma_K(T, \lambda)
$, $\lambda_1 = \deff/(B^2 H^2)$, $\lambda_2 = H^2/B^2$, and taking $\beta_t, \beta_t^{(1)}, \beta_t^{(2)}$ as in Eq.~\eqref{eq:beta0_choice},~\eqref{eq:beta1_choice} and~\eqref{eq:beta2_choice}, then with probability at least $ 1 - \delta$, the following holds that:
\begin{align*}
\operatorname{Regret}(T)
&:=
\sum_{t = 1}^T V_1^{*, \nu^t}(x_h^t) - V_1^{\pi^t, *}(x_h^t) \notag \\
&\leq 
\tilde{O}\left(
\deff^2 H^3 + \sqrt{\deff H^4 + \deff^2 H^3} \sqrt{T} + \left(\deff^7 H^8 + \deff^4 H^9 \right)^{1/4} T^{1/4}
\right)
.\notag
\end{align*}
\end{theorem}
Proofs of Lemma~\ref{lem:main_weighted} and Theorem~\ref{thm:main_bernstein} are deferred to Section~\ref{sec:mainproof_2}.


\section{Proof of Results for $\algname$}\label{sec:mainproof_1}

In this subsection, we provide the proof of our main Theorem \ref{thm:main} on RKHS.
\subsection{Proof of Theorem \ref{thm:main}}
We recall that the duality gap is defined as $\sum_{t = 1}^T V_1^{*, \nu^t}(x_1^t) - V_1^{\pi^t, *}(x_1^t)$.
As can be seen in our Algorithm~\ref{alg:base}, we maintain an optimistic estimate of $V_h^{*, v^t}(\cdot)$ as $\VUp(\cdot)$ and a pessimistic estimate of $V_h^{\pi^t, *}(\cdot)$ as $\VLp(\cdot)$.
Hence the term $\VUp(x_h^t) - \VLp(x_h^t)$ is approximately the upper bound of the duality gap.
We write the decomposition formally as follows:
\beq\label{eq:decomp}
\begin{aligned}
V_h^{*, \nu^t}(x_h^t) - V_h^{\pi^t, *}(x_h^t)
=
\underbrace{\overline{V}_h^t(x_h^t) - \underline{V}_h^t(x_h^t)}_{\mbox{I}}
-
\underbrace{\left(  V_h^{\pi^t, *}(x_h^t) - \underline{V}_h^t(x_h^t) \right)}_{\mbox{II}}
-
\underbrace{\left( \overline{V}_h^t(x_h^t) - V_h^{*, \nu^t}(x_h^t) \right)}_{\mbox{III}}
.
\end{aligned}
\eeq
We use $\overline{\delta}_h^t$ to denote the important quantity $\left\langle \phi_{\overline{V}_{h+1}^t}(z_h^t), \btheta_h^* - \overline{\btheta}_h^t \right\rangle_{\cH}$ (and $\left\langle \phi_{\underline{V}_{h+1}^t}(z_h^t), \btheta_h^* - \underline{\btheta}_h^t \right\rangle_{\cH}$) in estimating the duality gap. In the rest of the proof we aim to show that all of the above three terms can be bounded by a quantity related to $\overline{\delta}_h^t$ ($\underline{\delta}_h^t$) and a stochastic random variable that forms a martingale difference sequence for all $h \in [H], t \in [T]$.

For bounding term $\mbox{I}$, we first define two sequences of zero-mean variables:
\beq\label{eq:def_gamma_xi}\begin{aligned}
\gamma_h^t
&:=
\QUxh - \QLxh 
- \EE_{(a, b)} \left[\QUxab - \QLxab\right]
,
\\
\xi_h^t
&:=
\left(\PP_h (\VUp - \VLp) \right)(x_h^t, a_h^t, b_h^t) 
-
\left(\VUp(x_{h+1}^t) - \VLp(x_{h+1}^t)\right)
,
\end{aligned}\eeq 
where $\gamma_h^t$ depicts the stochastic error with respect to the policy and $\xi_h^t$ depicts the stochastic error with respect to the transition.
We refer to the proof of Lemma~\ref{lem:bound_UCB_LCB} for detailed explanations on these two error term.
Given the above definition, we have the following Lemma~\ref{lem:bound_UCB_LCB}.

\begin{lemma}\label{lem:bound_UCB_LCB}
Under the settings of Lemma~\ref{lem:main}, we have the following recursive bound for $\forall h \in [H]$:
\beq\label{eq:bound_UCB_LCB}
\begin{aligned}
&\overline{V}_h^t(x_h^t) 
-
\underline{V}_h^t(x_h^t)\notag \\
&\leq 
\overline{V}_{h+1}^t(x_{h+1}^t)
-
\underline{V}_{h+1}^t(x_{h+1}^t)
+
2\beta_t \min\{1, \overline{\bonus}_h^t(x_h^t)\}
+
2 \beta_t \min\{1, \underline{\bonus}_h^t(x_h^t)\}
 + \xi_h^t + \gamma_h^t 
.
\end{aligned}\eeq 
\end{lemma}

\begin{proof}[Proof of Lemma~\ref{lem:bound_UCB_LCB}]
By the update rule of Algorithm~\ref{algo:neural}, we have the following relation:
\beq\label{eq:QtoV}\begin{aligned}
\overline{V}_h^t(x_h^t) 
-
\underline{V}_h^t(x_h^t)
=
\EE_{(a, b) \sim \sigma_h^t(x_h^t)} \left[ 
\overline{Q}_h^t(x_h^t, a, b) 
-
\underline{Q}_h^t(x_h^t, a, b)
\right]
.
\end{aligned}\eeq
We note that the RHS of Eq.~\eqref{eq:QtoV} is an expectation over the CCE distribution $\sigma_h^t(x_h^t)$, which can be decomposed into one sample from the distribution plus a noise term as follows:
\beq\label{eq:gamma}\begin{aligned}
\overline{V}_h^t(x_h^t) 
-
\underline{V}_h^t(x_h^t)
&=
\QUxh - \QLxh + \gamma_h^t
,
\end{aligned}\eeq
where 
\begin{align*}
\gamma_h^t 
:= 
\QUxh - \QLxh 
- \EE_{(a, b)}\left[ \QUxab - \QLxab \right]
.
\end{align*}
Furthermore, for bounding the difference between the upper confidence $Q$ estimation and the lower confidence $Q$ estimation, we have
\begin{align*}
\lefteqn{
\QU{z_h^t} - \QL{z_h^t}
}
\\&\leq
\left\langle \phi_{\overline{V}_{h+1}^t}(z_h^t), \overline{\btheta}_h^t \right\rangle_{\cH} - \left\langle\phi_{\underline{V}_{h+1}^t}(z_h^t), \underline{\btheta}_h^t \right\rangle_{\cH} 
+ \beta_t \overline{\bonus}_h^t(z_h^t) + \beta_t \underline{\bonus}_h^t(z_h^t)
	\\&=
\left\langle \phi_{\overline{V}_{h+1}^t}(z_h^t), \overline{\btheta}_h^t \right\rangle_{\cH}
-
\left\langle \phi_{\overline{V}_{h+1}^t}(z_h^t), \btheta_h^* \right\rangle_{\cH}
+
\left\langle \phi_{\overline{V}_{h+1}^t}(z_h^t), \btheta_h^* \right\rangle_{\cH}
-
\left\langle \phi_{\underline{V}_{h+1}^t}(z_h^t), \btheta_h^* \right\rangle_{\cH}
	\\&~\quad +
\left\langle \phi_{\underline{V}_{h+1}^t}(z_h^t), \btheta_h^* \right\rangle_{\cH}
- \left\langle\phi_{\underline{V}_{h+1}^t}(z_h^t), \underline{\btheta}_h^t \right\rangle_{\cH} 
+ \beta_t \overline{\bonus}_h^t(z_h^t) + \beta_t \underline{\bonus}_h^t(z_h^t)
	\\&=
\left\langle \phi_{\overline{V}_{h+1}^t}(z_h^t), \overline{\btheta}_h^t - \btheta_h^* \right\rangle_{\cH}
+
\left( \PP_h(\overline{V}_{h+1}^t - \underline{V}_{h+1}^t) \right) (z_h^t)
	 +
\left\langle \phi_{\underline{V}_{h+1}^t}(z_h^t), \btheta_h^* - \underline{\btheta}_h^t \right\rangle_{\cH}\notag \\
&~\quad + \beta_t \overline{\bonus}_h^t(z_h^t) + \beta_t \underline{\bonus}_h^t(z_h^t)
.
\end{align*}
By utilizing Lemma~\ref{lem:main}, we further arrive at:
\begin{align*}
\QU{z_h^t} - \QL{z_h^t} 
	\leq 
\left( \PP_h(\overline{V}_{h+1}^t - \underline{V}_{h+1}^t) \right) (z_h^t) 
	+
2 \beta_t \overline{\bonus}_h^t(z_h^t) + 2 \beta_t \underline{\bonus}_h^t(z_h^t)
,
\end{align*}
where again by extracting the sequence
\begin{align*}
\xi_h^t
&:=
\left(\PP_h (\VUp - \VLp) \right)(x_h^t, a_h^t, b_h^t) 
-
\left(\VUp(x_{h+1}^t) - \VLp(x_{h+1}^t)\right)
,
\end{align*}
we have 
\beq\label{eq:xi}
\begin{aligned}
\QU{z_h^t} - \QL{z_h^t}
&\leq
\VUp(x_{h+1}^t) - \VLp(x_{h+1}^t)
    + 2 \beta \overline{\bonus}_h^t(z_h^t) + 2\beta \underline{\bonus}_h^t(z_h^t)
    + \xi_h^t
.
\end{aligned}
\eeq
Combining Eq.~\eqref{eq:gamma} and~\eqref{eq:xi} concludes the following recursive bound:
\begin{align*}
\overline{V}_h^t(x_h^t) - \underline{V}_h^t(x_h^t) 
	\leq 
\overline{V}_{h+1}^t(x_{h+1}^t) - \underline{V}_{h+1}^t(x_{h+1}^t) + 2\beta_t \overline{\bonus}_h^t(z_h^t) + 2 \beta_t \underline{\bonus}_h^t(z_h^t) + \xi_h^t + \gamma_h^t.
\end{align*}
Moreover, due to the fact that $\overline{V}_h^t(x_h^t) - \underline{V}_h^t(x_h^t) \leq 2H$, we can rewrite the above inequality into:
\begin{align*}
&\overline{V}_h^t(x_h^t) - \underline{V}_h^t(x_h^t) \notag \\
	&\leq 
\min\left\{ 2H, \overline{V}_{h+1}^t(x_{h+1}^t) - \underline{V}_{h+1}^t(x_{h+1}^t) + 2\beta_t \overline{\bonus}_h^t(z_h^t) + 2 \beta_t \underline{\bonus}_h^t(z_h^t) + \xi_h^t + \gamma_h^t\right\}
	\\&\leq 
\min\left\{2H, 2\beta_t \overline{\bonus}_h^t(z_h^t) + 2 \beta_t \underline{\bonus}_h^t(z_h^t)\right\} +  \overline{V}_{h+1}^t(x_{h+1}^t) - \underline{V}_{h+1}^t(x_{h+1}^t) + \xi_h^t + \gamma_h^t
	\\&\leq 
2\beta_t \min\{ 1, \overline{\bonus}_h^t(z_h^t) \} + 2\beta_t \min\{ 1, \underline{\bonus}_h^t(z_h^t) \}+  \overline{V}_{h+1}^t(x_{h+1}^t) - \underline{V}_{h+1}^t(x_{h+1}^t) + \xi_h^t + \gamma_h^t,
\end{align*}
where the last inequality is due to the choice of $\beta$ satisfying $\beta/H \geq 1$.
This completes the proof of Lemma~\ref{lem:bound_UCB_LCB}.
\end{proof}

For bounding $\mbox{II}$ and $\mbox{III}$, we use induction to prove that $\mbox{III} \geq 0$ for every $h$, that is, 
\begin{align}
    \overline{V}_h^t(x_h^t) - V_h^{*, \nu^t}(x_h^t) \geq 0.\label{help:recur}
\end{align}
Then the same statement will also hold for $\mbox{II}$ due to the symmetry property. The statement holds for $h = H+1$, where $\mbox{III} = 0$ {\blue (since $\overline{V}_{H+1}^t = V_{H+1}^{*, \nu^t}=0$ by definition)}. Suppose the statement holds for $h+1$. Let $(a, b) \in \cA_1 \times \cA_2$ and $z: = (x_h^t, a, b)$. If $\overline{Q}_h^t(z) \geq H$, then by definition, $\mbox{III} \geq 0$. Suppose $\overline{Q}_h^t(z) < H$, then by definition of $\overline{Q}_h^t(z)$,
we have
\begin{align}
\overline{Q}_h^t(z) - Q_h^{*, \nu^t}(z) 
&= 
\left\langle \phi_{\overline{V}_{h+1}^t}(z), \overline{\btheta}_h^t - \btheta_h^* \right\rangle_{\cH} + \left(\PP_h(\overline{V}_{h + 1}^t - V_{h + 1}^{*, \nu^t})\right)(z) + \beta_t \overline{\bonus}_h^t(z)
\notag \\& \geq  
- \beta_t \overline{\bonus}_h^t(z) +  \beta_t \overline{\bonus}_h^t(z) = 0
,\label{help:55}
\end{align}
where the first inequality holds due to the statement holds for $h+1$, which leads to $\overline{V}_{h + 1}^t - V_{h + 1}^{*, \nu^t} \geq 0$, and Lemma \ref{lem:main} that gives a bound for $\left\langle \phi_{\overline{V}_{h+1}^t}(z), \overline{\btheta}_h^t - \btheta_h^* \right\rangle_{\cH}$. Next, we have
\begin{align*}
\overline{V}_h^t(x_h^t) - V_h^{*, \nu^t}(x_h^t) 
    &= 
\EE_{(a, b) \sim \sigma_h^t(x_h^t)} \overline{Q}_h^t(x_h^t, a, b) 
- 
\EE_{a \sim \text{br}(\nu_h^t), b\sim \nu_h^t}
Q_h^{*, \nu^t}(x_h^t, a, b)
    \\&\geq 
\EE_{a \sim \text{br}(\nu_h^t), b\sim \nu_h^t} \overline{Q}_h^t(x_h^t, a, b) 
- 
\EE_{a \sim \text{br}(\nu_h^t), b\sim \nu_h^t}
Q_h^{*, \nu^t}(x_h^t, a, b)
    \\&=
\EE_{a \sim \text{br}(\nu_h^t), b\sim \nu_h^t} \left[\overline{Q}_h^t(x_h^t, a, b) 
- 
Q_h^{*, \nu^t}(x_h^t, a, b)
\right]
\geq 0
,
\end{align*}
where $\nu_h^t := \cP_2 \sigma_h^t$ and $\pi_h^t := \cP_1 \sigma_h^t$ is the projection of $\sigma_h^t$ on the first and second coordinate respectively and $\text{br}$ is the best-response policy of a given distribution. The last inequality holds due to \eqref{help:55}. Therefore, the statement holds for $h$, which shows that the induction holds.

Combining with Eq.~\eqref{eq:decomp}, we arrive at a bound in terms of $\overline{\bonus}_h^t(z_h^t), \underline{\bonus}_h^t(z_h^t)$ and the martingale difference sequences:
\beq\label{eq:bound_combine}\begin{aligned}
 V_1^{*, \nu^t}(x_h^t) - V_1^{\pi^t, *}(x_h^t) 
	\leq 
\sum_{h = 1}^H \left(2 \beta_t \min\{1, \overline{\bonus}_h^t(z_h^t)\} + 2 \beta_t \min\{1, \underline{\bonus}_h^t(z_h^t)\}
	+
\xi_h^t + \gamma_h^t \right)
,
\end{aligned}\eeq
where $\nu^t$ is the policy that operates according to $\nu_h^t$ at time $h$ and $\pi^t$ is the sequence of $\pi_h^t$ accordingly.
The rest of the proof follows by bounding $\sum_{t = 1}^T \sum_{h = 1}^H \min\{1, \overline{\bonus}_h^t(z_h^t)\}, \sum_{t = 1}^T \sum_{h = 1}^H \min\{1, \underline{\bonus}_h^t(z_h^t)\}$ and the martingale difference sequences.

The bound of $\sum_{t = 1}^T \sum_{h = 1}^H \min\{1, \overline{\bonus}_h^t(z_h^t)\}$ and $\sum_{t = 1}^T \sum_{h = 1}^H \min\{1, \underline{\bonus}_h^t(z_h^t)\}$ comes directly from the following lemma~\ref{lemm:bound_b} which can be simply derived from Lemma 11 in~\citet{abbasi2011improved} and is an analogue of Lemma E.3 of~\citet{yang2020function}:
\begin{lemma}[Lemma E.3 of~\citet{yang2020function}]\label{lemm:bound_b}
For any sequence $\left\{\xholder_t \right\}_{t \geq 1}$ taking values on the RKHS $\cH$ satisfying $\forall t$, $\normh{\xholder_t} \leq L$.
Let $I_{\cH}$ be the identity operator on $\cH$ and $\Lambda_0:= \lambda \cdot I_{\cH}$ the multiplication operator by $\lambda$. Furthermore, if we let $\Lambda_t := \Lambda_0 + \sum_{i = 1}^t \xholder_i \xholder_i^\top$ be a positive definite operator from $\cH$ to $\cH$ and $K_t\in\RR^{t \times t}$ the Gram matrix of $\cH$ obtained from $\{\xholder_t\}_{t \geq 1}$. Then the following holds for $\forall t > 0$:
\begin{align*}
\sum_{i = 1}^t \min \{1, \xholder_i^\top \Lambda_{t - 1}^{-1} \xholder_i\}
	\leq 
2 \logdet(\Ib + K_t /\lambda)
.
\end{align*}
\end{lemma}
We recall that by Lemma~\ref{lem:facts}, 
$
\overline{\bonus}_h^t = \left[\ophi(z)^\top (\oLambda_h^t)^{-1} \ophi(z) \right]^{1/2},
$
and the same holds for $\underline{\bonus}_h^t$. Let $\xholder_t = \ophi(z_h^t)$ and $\Lambda_0 = \lambda \cdot \Ib$ in Lemma~\ref{lemm:bound_b} and by applying the Cauchy-Schwarz inequality, we have that 
\beq\label{eq:bound_b}\begin{aligned}
\sum_{h = 1}^H \sum_{t = 1}^T \beta_t \min\{1, \overline{\bonus}_h^t(z_h^t)\}, \sum_{h = 1}^H \sum_{t = 1}^T  \beta_t \min\{1, \underline{\bonus}_h^t(z_h^t)\} \leq 2 \beta H \cdot \sqrt{T}  \sqrt{\Gamma_K(T, \lambda)}
,
\end{aligned}\eeq
where $\beta_t = \beta$ takes values as in Lemma~\ref{lem:main} for each $t$ and $\Gamma_K(T, \lambda)$ is defined as the supreme over all $V$'s and $z$'s as in Definition~\ref{def:eff}.
For the martingale difference sequence $\xi_h^t + \gamma_h^t$, as $\left| \xi_h^t + \gamma_h^t\right| \leq 4 H$, we bound it by Azuma-Hoeffding which gives us with probability at least $1 - \delta$:
\begin{align*}
\sum_{t = 1}^T \sum_{h = 1}^H \xi_h^t + \gamma_h^t
    \leq 
\mathcal{O}\left(H \sqrt{T H} \cdot \log (1/\delta)\right)
.
\end{align*}
Combining the above inequality with the bound in~\eqref{eq:bound_b} and~\eqref{eq:bound_combine} concludes our proof of Theorem~\ref{thm:main}.

\subsection{Proof of Corollary \ref{coro:main}}
Due to the selection of $t_0$, we have
\begin{align}
    V_1^{*, \nu^{t_0}}(x_1) - V_1^{\pi^{t_0}, *}(x_1)& \leq \overline{V}_1^{t_0}(x_1) - \underline{V}_1^{t_0}(x_1) \leq \frac{1}{T} \sum_{t=1}^T \overline{V}_1^t(x_1) - \underline{V}_1^t(x_1) \leq \sqrt{\frac{\beta^2H^2 \Gamma_{K}(T, \lambda)}{T}},\label{help:12}
\end{align}
where the first inequality holds due to \eqref{help:recur} and its counterpart for $\underline{V}_1^t$, the second one holds due to the selection of $t_0$. From \eqref{help:12} we can see that by selecting $T$ as what our statement suggests, the $\epsilon$-approximate NE can be guaranteed.


\section{Proof of Results for $\algname  +$}\label{sec:mainproof_2}
In this section we give the proof of the results in Appendix \ref{sec:bern}. One of the key results of this paper is the following Bernstein self-normalized concentration
inequality:
\begin{theorem}[Bernstein inequality for vector-valued martingales]\label{lemma:concentration_variance}
Let $\{\cG_{t}\}_{t=1}^\infty$ be a filtration, $\{\xb_t,\eta_{t+1}\}_{t\ge 1}$ be a stochastic process so that $\xb_t \in \RR^d$ is $\cG_t$-measurable and $\eta_{t+1} \in \RR$ is $\cG_{t+1}$-measurable. 
Fix $R,L,\sigma,\lambda>0$, $\bmu^*\in \RR^d$. 
For $t\ge 1$ we observe $\la \bmu^*, \xb_t\ra + \eta_{t+1}$ and suppose that $\eta_{t+1}, \xb_t$ also satisfy 
\begin{align}
    |\eta_{t+1}| \leq R
    ,\quad 
    \EE[\eta_{t+1}|\cG_t] = 0
    ,\quad 
    \EE [\eta_{t+1}^2|\cG_t] \leq \sigma^2
    ,\quad 
    \|\xb_t\|_2 \leq L
    .\notag
\end{align}
Then, for any $\delta\in (0,1)$, with probability at least $1-\delta$ we have 
\begin{align}
    \bigg\|\sum_{i=1}^t \xb_i \eta_{i+1}\bigg\|_{\Zb_t^{-1}} \leq \beta_t
    ,\quad 
    \forall t>0
    ,\label{eq:concentration_variance:xx}
\end{align}
where for each $t\ge 1$, $\Zb_t = \lambda\Ib + \sum_{i=1}^t \xb_i\xb_i^\top$, and
\[
\beta_t = 8\sigma\sqrt{\log\det(\Ib + \Kb_t/\lambda) \log(4t^2/\delta)} + 4R \log(4t^2/\delta)
,\qquad 
[\Kb_t]_{i,j} = \la \xb_i, \xb_j\ra_{\cH}
.\notag
\]
\end{theorem}
\begin{proof}

The proof can be derived by following the proof of Theorem 2 in \citet{zhou2020nearly}. We only need to replace Lemma 12 in \citet{zhou2020nearly} with Lemma \ref{lemm:bound_b}, then the remaining of the proof goes through the same as \citet{zhou2020nearly}. 
\end{proof}

We first give the proof of Lemma \ref{lem:main_weighted}.

\subsection{Proof of Lemma~\ref{lem:main_weighted}}

\begin{proof}

We only provide the proof for the max-player; results for the min-player can be derived similarly. We recall that we define $\overline{\yholder}_{h, 1}^t$ to be the vector of regression targets 
\begin{align}
    \left( \overline{V}_{h+1}^1(x_{h+1}^1)/\uppvar^1, \ldots, \overline{V}_{h+1}^{t - 1}(x_{h+1}^{t - 1})/\uppvar^{t - 1}\right)^\top \in \RR^{t - 1}.\notag
\end{align}
Furthermore by Lemma~\ref{lem:facts} in Section~\ref{sec:facts}, we know that $\overline{\btheta}_{h, 1}^t = \left(\overline{\phiholder}_{h, 1}^t\right)^\top \left[\overline{K}_{h, 1}^t + \lambda_1 \cdot \Ib \right]^{-1} \overline{\yholder}_{h, 1}^t$ and $\phi_{\overline{V}_{h+1}^t}(z) = \left(\overline{\phiholder}_{h, 1}^t\right)^\top (\overline{K}_{h, 1}^t + \lambda_1 \cdot \Ib )^{-1} \overline{k}_{h, 1}^t(z) + \lambda_1 \cdot (\oLambda_{h, 1}^t)^{-1} \ophi(z)$, which enable us to bound the difference $\left\langle \ophi(z), \otheta_{h, 1}^t - \btheta_{h, 1}^*\right\rangle_{\cH}$ as follows:
\begin{align*}
\hprod{\ophi(z)}{ \overline{\btheta}_{h, 1}^t -   \btheta_{h}^*}
&=
\ophi(z)^\top \left(\ophiholder_{h, 1}^t\right)^\top \left[\oK_{h, 1}^t + \lambda_1 \cdot \Ib \right]^{-1} \overline{\yholder}_{h, 1}^t
	\\&~\quad -
\left(\otheta_h^*\right)^\top\left[ \left(\ophiholder_{h, 1}^t\right)^\top \left[\oK_{h, 1}^t + \lambda_1 \cdot \Ib \right]^{-1} \ok_{h, 1}^t(z) + \lambda_1 \cdot (\oLambda_{h, 1}^t)^{-1} \ophi(z) \right] 
\\&=
\underbrace{
(\ok_{h, 1}^t)^\top \left[\oK_{h, 1}^t + \lambda_1 \cdot \Ib \right]^{-1} \left[\overline{\yholder}_{h, 1}^t - \ophiholder_{h, 1}^t \btheta_{h}^* \right]
}_{\mbox{I}_1}
-
\underbrace{
\lambda_1 \cdot \phi_{\overline{V}_{h+1}^t}(z)^\top (\oLambda_{h, 1}^t)^{-1} \btheta_{h}^*
}_{\mbox{I}_2}
.
\end{align*}
For bounding $\mbox{I}_2$, we apply the Cauchy-Schwarz inequality and have 
\begin{align*}
\lambda_1 \cdot \phi_{\overline{V}_{h+1}^t}(z)^\top (\oLambda_h^t)^{-1} \btheta_{h}^*
&\leq 
\normh{\lambda_1 \cdot \phi_{\overline{V}_{h+1}^t}(z)^\top (\oLambda_{h, 1}^t)^{-1}} \cdot \normh{\btheta_h^*}
\\&\overset{(a)}{\leq} 
B \cdot \sqrt{ \lambda_1 \phi_{\overline{V}_{h+1}^t}(z)^\top (\oLambda_{h, 1}^t)^{-1} \lambda_1 (\oLambda_{h, 1}^t)^{-1} \phi_{\overline{V}_{h+1}^t}(z)} 
\overset{(b)}{\leq} \sqrt{\lambda_1} B\cdot \overline{\bonus}_{h, 1}^t(z)
,
\end{align*}%
where $(a)$ is due to the assumption that $\normh{\btheta_h^*} \leq B$, and $(b)$ is by the definition of $\overline{\bonus}_h^t$ and the fact that $\left(\oLambda_{h, 1}^t\right)^{-1}$ is a self-adjoint mapping on the RKHS $\cH$ satisfying $\norm{\left(\oLambda_{h, 1}^t\right)^{-1}}_{op} \leq \frac{1}{\lambda_1}$.

For bounding $\mbox{I}_1$, we observe the following equality:
\begin{align*}
&
(\ok_{h, 1}^t)^\top \left[\oK_{h, 1}^t + \lambda_1 \cdot \Ib \right]^{-1} \left[\overline{\yholder}_{h, 1}^t - \ophiholder_{h, 1}^t \btheta_h^* \right]
\\&=
\phi_{\overline{V}_{h+1}^t}(z)^\top (\oLambda_{h, 1}^t)^{-1} (\ophiholder_{h, 1}^t)^\top \left[\overline{\yholder}_{h, 1}^t - \phiholder_{h, 1}^t \btheta_h^* \right]
\\&=
\phi_{\overline{V}_{h+1}^t}(z)^\top (\oLambda_{h, 1}^t)^{-1} \sum_{\tau = 1}^{t - 1}\phi_{\overline{V}_{h+1}^\tau}(z_h^\tau) \left[\overline{V}_{h + 1}^\tau(x_{h+1}^\tau) - (\PP_h \overline{V}_{h+1}^\tau)(z_h^\tau) \right]/\left(\uppvar^\tau\right)^2
.
\end{align*}
Again by applying the Cauchy-Schwarz inequality, we bound the RHS of the above equality as 
\begin{align*}
|\mbox{I}_1| &\leq 
\norm{\phi_{\overline{V}_{h+1}^t}(z)}_{\left(\oLambda_{h, 1}^t \right)^{-1}}
	\cdot
\norm{\sum_{\tau = 1}^{t - 1}\phi_{\overline{V}_{h+1}^\tau}(z_h^\tau) \left[\overline{V}_{h + 1}^\tau(x_{h+1}^\tau) - (\PP_h \overline{V}_{h+1}^\tau)(z_h^\tau)\right]/\left(\uppvar^\tau\right)^2}_{\left(\oLambda_{h, 1}^t\right)^{-1}}
.
\end{align*}
We note that for the $h$ considered in Lemma~\ref{lem:main_weighted}, if we define $\{\cF_t\}_{t \geq 0}$ as the $\sigma$-algebra generated by all data before iteration $t - 1$ along with all data before time $h$ at iteration $t$, then
\beq\label{eq:eta}
\eta_{\tau + 1} 
	:= 
\left(\oV_{h+1}^\tau(x_{h+1}^\tau) - (\PP_h \oV_{h+1}^\tau)(z_h^\tau)\right)/\uppvar^\tau
	\in
\cF_{\tau + 1}
\eeq
is a mean zero random variable with respect to filtration $\cF_{\tau}$.
By the choice of $\uppvar$ such that $\uppvar \geq \alpha_t$, we can bound the absolute value of $\eta_{\tau + 1}$ by $\left|\left(\oV_{h+1}^\tau(x_{h+1}^\tau) - (\PP_h \oV_{h+1}^\tau)(z_h^\tau)\right)/\uppvar^\tau\right| \leq 2 H / \alpha_{\tau}$. We take $\eta_{\tau + 1}$ as in Eq.~\eqref{eq:eta} in Theorem~\ref{lemma:concentration_variance} and $\{\xholder_t\}_{t \geq 1}$ is $\left\{ \ophi(z_h^t)/\uppvar^t\right\}_{t \geq 1} \in \cF_{\tau}$. Then by directly utilizing Theorem~\ref{lemma:concentration_variance}, we have that the following inequality holds with probability at least $1 - \delta/H$, 
\begin{align*}
&
\norm{\sum_{\tau = 1}^{t - 1}\phi_{\overline{V}_{h+1}^\tau}(z_h^\tau) \left[\overline{V}_{h + 1}^t(x_{h+1}^\tau) - (\PP_h \overline{V}_{h+1}^t)(z_h^\tau)\right]/\uppvar^\tau}^2_{\left(\oLambda_{h, 1}^t \right)^{-1}}
    \\&\leq 
16 H/\alpha \sqrt{\log\det(\Ib + K_{t}^{(1)}/\lambda_1) \log(4t^2 H/\delta)} + 8 H/\alpha \log(4t^2 H/\delta)
    \\&\leq 
16 H/\alpha \sqrt{\log\det(\Ib + K_{t}/\left(\lambda_1 (\alpha_t)^2 \right)) \log(4t^2 H/\delta)} + 8 H/\alpha \log(4t^2 H/\delta)
    \\&\leq 
16 H/\alpha \sqrt{\Gamma_K(T, \lambda_1 (\alpha_t)^2)} \sqrt{\log(4t^2 H/\delta)} + 8 H/\alpha \log(4t^2 H/\delta)
,
\end{align*}%
where $K_{t}^{(1)}$ is the Gram matrix for $\{\xholder_\tau\}_{\tau \in [t - 1]} = \left\{ \ophi(z_h^\tau)/\uppvar^\tau\right\}_{\tau \in [t - 1]}$,  $K_{t}$ is the Gram matrix for $\left\{ \ophi(z_h^\tau)\right\}_{\tau \in [t - 1]}$.

On the other hand, when estimating $\PP_h \left(\overline{V}_{h+1}^t\right)^2$, we have the following result regarding $\overline{\btheta}_{h, 2}^t$ holds with probability at least $1 - \delta/H$:
\begin{align*}
\left| 
\left\langle 
\phi_{\left(\overline{V}_{h+1}^t \right)^2}(z), \overline{\btheta}_{h, 2}^t - \btheta_h^*
\right\rangle_{\cH}
\right|
    \leq 
16 H^2 \sqrt{\Gamma_K(T, \lambda_2/H^2)} \sqrt{\log(4t^2 H/\delta)} + 8 H^2 \log (4t^2 H/\delta) + \sqrt{\lambda_2} \cdot B
.
\end{align*}
By letting
\begin{align*}\beta_t^{(1)} = 16 H/\alpha \sqrt{\Gamma_K(T, \lambda_1 (\alpha_t)^2)} \sqrt{\log(4t^2 H/\delta)} + 8 H/\alpha \log(4t^2 H/\delta) + \sqrt{\lambda_1} \cdot B
,
\end{align*}
and 
\begin{align*}
\beta_t^{(2)} = 16 H^2 \sqrt{\Gamma_K(T, \lambda_2/H^2)} \sqrt{\log(4t^2 H/\delta)}+ 8 H^2 \log (4t^2 H/\delta) + \sqrt{\lambda_2} \cdot B
,
\end{align*}
we have that from the above derivation and by taking union bounds over $h \in [H]$,
\begin{align}
    \left| \phi_{\overline{V}_{h+1}^\tau}(z)^\top (\btheta_h^* - \overline{\btheta}_{h, 1}^t) \right|
\leq
\beta_t^{(1)} \cdot \overline{\bonus}_{h, 1}^t(z),\notag
\end{align}
and
\begin{align}
    \left| \phi_{(\overline{V}_{h+1}^\tau)^2}(z)^\top (\btheta_h^* - \overline{\btheta}_{h, 2}^t) \right|
\leq
\beta_t^{(2)} \cdot \overline{\bonus}_{h, 2}^t(z)
,
\end{align}
with probability at least $1 - 2\delta$. 
This concludes our proof.
\end{proof}
 
From Lemma~\ref{lem:main_weighted}, we can prove that $\uppvar^t$ is an upper bound of the actual variance of $\overline{V}_{h+1}^t$.

\begin{lemma}\label{lem:est_var}
Following the setting of Lemma~\ref{lem:main_weighted} and assume that event $\mathcal{E}$ occurs, the following holds for any $(t, h) \in T \times H$:
\begin{align*}
\left| 
\Vest \overline{V}_{h+1}^t(z_h^t) 
    - 
\VV \overline{V}_{h+1}^t(z_h^t)
\right| 
    \leq 
\min\left\{ 
H^2, \beta_t^{(2)} \overline{\bonus}_{h, 2}^t
\right\} 
    +
\min\left\{ 
H^2, 2H \beta_t^{(1)} \overline{\bonus}_{h, 1}^t
\right\}
.
\end{align*}
\end{lemma}
\begin{proof}
By the triangle inequality we have that
\beq\label{eq:inSet-1}
\begin{aligned}
& |\mathbb{V}^{\text{est}}\overline{V}_{h+1}^t(z_h^t) - \mathbb{V}\overline{V}_{h+1}^t(z_h^t)|
\\&\leq 
\Big|\langle\bphi_{(\overline{V}_{h+1}^t)^2}(z_h^t), \btheta^{*}_{h}\rangle_{\cH}  - \big[\langle\bphi_{(\overline{V}_{h+1}^t)^2}(z_h^t), \overline{\btheta}_{h, 2}^t\rangle_{\cH}\big]_{[0,H^{2}]} \Big|\\
&\quad + \Big| (\langle\bphi_{\overline{V}_{h+1}^t}(z_h^t), \btheta^{*}_{h}\rangle_{\cH})^{2}-  
\big[\langle\bphi_{\overline{V}_{h+1}^t}(z_h^t), \overline{\btheta}_{h, 1}^t \rangle_{\cH}\big]^{2}_{[-H,H]}\Big|
\\&\leq 
\min \left\{ H^2, \Big|\langle\bphi_{(\overline{V}_{h+1}^t)^2}(z_h^t), \btheta^{*}_{h} - \overline{\btheta}_{h, 2}^t\rangle_{\cH}\Big| \right\}
+ 
\min\left\{H^2, 2H \Big| \langle\bphi_{\overline{V}_{h+1}^t}(z_h^t), \btheta^{*}_{h} - \overline{\btheta}_{h, 1}^t \rangle_{\cH}\Big|\right\}
    \\&\leq 
\min\left\{ 
H^2, \beta_t^{(2)} \overline{\bonus}_{h, 2}^t
\right\} 
    +
\min\left\{ 
H^2, 2H \beta_t^{(1)} \overline{\bonus}_{h, 1}^t
\right\}
,
\end{aligned}
\eeq
where the last inequality directly comes from Lemma~\ref{lem:main_weighted}.
\end{proof}

\begin{lemma}[Fine-tuned bound]\label{lem:main_weighted_fine}
Assuming that for any $h \in [H]$, $\norm{\btheta_h^*}_{\cH} \leq B$. Let $\beta_t$ satisfy
\beq\label{eq:beta0_choice}
\beta_t \geq 
16 \sqrt{\Gamma_K(T, \lambda_1 (\alpha_t)^2)} \sqrt{\log(4t^2 H/\delta)} + 8 H/\alpha \log(4t^2 H/\delta) + \sqrt{\lambda_1} \cdot B.
\eeq 
Then on the event defined in Lemma~\ref{lem:main_weighted}, there exists an event $\mathcal{E}_1$ such that the following holds with probability at least $1 - \delta$ for any $(t, h) \in [T] \times [H]$ and any $(x, a, b) \in \cS \times \cA \times \cA$:

\begin{align*}
\Big|\langle\bphi_{\overline{V}_{h_1}^t}(z_h^t), \btheta^{*}_{h} - \overline{\btheta}_{h, 1}^t\rangle_{\cH}\Big|
    \leq 
\beta_t
\cdot 
\overline{\bonus}_{h, 1}^k(z_h^k)
.
\end{align*}
\end{lemma}
\begin{proof}
From the definition of $\uppvar$ and $\overline{E}_h^t$, we know that 
\begin{align*}
\uppvar^t 
	\geq 
\Vest \overline{V}_{h+1}^t (z_h^t) 
    +
\min\left\{ 
H^2, \beta_t^{(2)} \overline{\bonus}_{h, 2}^t
\right\}
    +
\min\left\{
H^2, 2H \beta_t^{(1)} \overline{\bonus}_{h, 1}^t
\right\}
.
\end{align*}
Combining with the result in Lemma~\ref{lem:est_var} where we bound the absolute difference between the estimated variance and the true variance in Eq.~\eqref{eq:inSet-1}, we have that on the event $\mathcal{E}$ defined in Lemma~\ref{lem:main_weighted}:
\begin{align*}
\uppvar^t \geq \mathbb{V}\overline{V}_{h+1}^t(z_h^t)
.
\end{align*}
We derive a fine-tuned bound on the variance of $\eta_{t + 1}$ defined in Eq.~\eqref{eq:eta} that on event $\mathcal{E}$:
\begin{align*}
\EE \left[ \eta_{t + 1}^2 \mid \cF_{t} \right] 
	=
\mathbb{V}\overline{V}_{h+1}^t(z_h^t)/\left(\uppvar^t\right)^2
	\leq 
1
.
\end{align*}
The rest of the proof follows by a direct application of Theorem~\ref{lemma:concentration_variance}.
\end{proof}

\begin{lemma}\label{lemma:boundofvariance}
On the event $\mathcal{E} \cap \mathcal{E}_1$, there exists an event $\mathcal{E}_2$ such that $\mathcal{E}_2$ holds with probability at least $1 - \delta$, we have
\begin{align}
    \sum_{t = 1}^T \sum_{h = 1}^H \left(\uppvar^t \right)^2
    &\leq 
HT\alpha^2 + 3(H^2 T + H^3 \log (1/\delta)) + 4 H \sum_{t = 1}^T \sum_{h = 1}^H \PP_h[\overline{V}_{h+1}^t - V_{h+1}^{\mu^t}]\notag
    \\&\quad +
2 \beta_T^{(2)} \sqrt{T H}\cdot \sqrt{2H\Gamma_K(T, \lambda_2/H^2)}\notag \\
&\quad 
    +
7 \beta_T^{(1)} H^2 \sqrt{T H} \cdot \sqrt{2 H\Gamma_K(T, \lambda_1/\alpha^2)} 
.\notag
\end{align}

\end{lemma}

\begin{proof}
First by considering the definition of $\uppvar$, we know that
\begin{align*}
&\sum_{t = 1}^T \sum_{h = 1}^H \left(\uppvar^t \right)^2\notag \\
	&\leq 
\sum_{t = 1}^T \sum_{h = 1}^H \left(
\alpha_t^2 + 
\Vest \overline{V}_{h+1}^t (z_h^t) 
    +
\overline{E}_h^t
\right)
	\\&=
\sum_{t = 1}^T \sum_{h = 1}^H \left(
\alpha_t^2 +
\Vest \overline{V}_{h+1}^t (z_h^t) 
	+
\min\left\{ 
H^2, \beta_t^{(2)} \overline{\bonus}_{h, 2}^t
\right\}
    +
\min\left\{
H^2, 2H \beta_t^{(1)} \overline{\bonus}_{h, 1}^t
\right\}
\right)
	\\&\leq 
HT\alpha^2
	+
 \sum_{t = 1}^T \sum_{h = 1}^H \left(
\VV \overline{V}_{h+1}^t (z_h^t) 
	+
2\min\left\{ 
H^2, \beta_t^{(2)} \overline{\bonus}_{h, 2}^t
\right\}
    +
2\min\left\{
H^2, 2H \beta_t^{(1)} \overline{\bonus}_{h, 1}^t
\right\}
\right)
	\\&\leq 
HT\alpha^2
	+
\underbrace{ \sum_{t = 1}^T \sum_{h = 1}^H 
\left[\VV \overline{V}_{h+1}^t (z_h^t) 
	-
\VV \overline{V}_{h+1}^{\pi_t} (z_h^t)
\right]
}_{\mbox{I}}
	+
 \underbrace{\sum_{t = 1}^T \sum_{h = 1}^H 
 \VV \overline{V}_{h+1}^{\pi_t} (z_h^t)
 }_{\mbox{II}}
	\\&~\quad +
 \underbrace{\sum_{t = 1}^T \sum_{h = 1}^H 
2\min\left\{ 
H^2, \beta_t^{(2)} \overline{\bonus}_{h, 2}^t
\right\}
    +
 \sum_{t = 1}^T \sum_{h = 1}^H 
2\min\left\{
H^2, 2H \beta_t^{(1)} \overline{\bonus}_{h, 1}^t
\right\}
}_{\mbox{III}}
.
\end{align*}
The rest of the proof for bounding $\mbox{I}, \mbox{II},  \mbox{III}$ goes the same as in the proof of Lemma A.6 in~\citet{chen2021almost}, except that we replace Lemma B.4 in~\citet{chen2021almost} with Lemma~\ref{lemm:bound_b}.
\end{proof}

\begin{proof}[Proof of Theorem~\ref{thm:main_bernstein}]
The first part of the proof follows almost the same as in the proof of Theorem~\ref{thm:main} by replacing $\overline{\bonus}_h^t$ with $\overline{\bonus}_{h, 1}^t$ and $\underline{\bonus}_h^t$ with $\underline{\bonus}_{h, 1}^t$, except that now we have $\beta \uppvar^t \geq 2H$ so that we have~\eqref{eq:bound_combine_weighted} instead of~\eqref{eq:bound_combine}. 

\beq\label{eq:bound_combine_weighted}\begin{aligned}
&  V_1^{*, \nu^t}(x_1^t) - V_1^{\pi^t, *}(x_1^t) 
	\\&\leq 
\sum_{h = 1}^H \left(4 \beta_t \uppvar^t \min\{1, \overline{\bonus}_{h, 1}^t(z_h^t)/\uppvar^t\} + 4 \beta_t \lowvar^t \min\{1, \underline{\bonus}_{h, 1}^t(z_h^t)/\uppvar^t\}
	+
\xi_h^t + \gamma_h^t\right)
,
\end{aligned}\eeq
Similarly, for any $1 \leq h' \leq H$, we have
\beq\label{eq:bound_part_weighted}\begin{aligned}
&\overline{V}_{h'}^t(x_{h'}^t) - \underline{V}_{h'}^t (x_{h'}^t) \notag \\
	&\leq 
\sum_{h = 1}^H \left(2 \beta_t \uppvar^t \min\{1, \overline{\bonus}_{h, 1}^t(z_h^t)/\uppvar^t\} + 2 \beta_t \lowvar^t \min\{1, \underline{\bonus}_{h, 1}^t(z_h^t)/\uppvar^t\}
	+
\xi_h^t + \gamma_h^t  \right).
\end{aligned}\eeq
Applying the Azuma-Hoeffding inequality onto \eqref{eq:bound_part_weighted}, we have with probability at least $1-\delta$,
\begin{align}
    &\sum_{t = 1}^T \sum_{h = 1}^H \PP_h[\overline{V}_{h+1}^t - \underline{V}_{h+1}^t](x_h^t, a_h^t, b_h^t) \notag \\
    &\leq \sum_{t=1}^T\sum_{h = 1}^H \left(2 \beta_t \uppvar^t \min\{1, \overline{\bonus}_{h, 1}^t(z_h^t)/\uppvar^t\} + 2 \beta_t \lowvar^t \min\{1, \underline{\bonus}_{h, 1}^t(z_h^t)/\uppvar^t\}
	+
\xi_h^t + \gamma_h^t  \right) + \sum_{t=1}^T\sum_{h = 1}^H\xi_h^t.\label{help:2}
\end{align}
We now estimate the two summation terms $\sum_{t = 1}^T \sum_{h = 1}^H \uppvar^t \min\{1, \overline{\bonus}_{h, 1}^t(z_h^t)/\uppvar^t\}$ and \\$\sum_{t = 1}^T \sum_{h = 1}^H \uppvar^t \min\{1, \underline{\bonus}_{h, 1}^t(z_h^t)/\uppvar^t \}$, separately.
By the definitions in Lemma~\ref{lem:facts}, Section~\ref{sec:facts}, we know that 
\begin{align*}
&\sum_{t = 1}^T \sum_{h = 1}^H \uppvar^t \min\{1, \overline{\bonus}_{h, 1}^t(z_h^t)/\uppvar^t\}\notag \\
    &=
\sum_{t = 1}^T \sum_{h = 1}^H \uppvar^t \min\left\{1, \left[\phi_{\overline{V}_{h+1}^t}(z)^\top \oLambda_{h, 1}^t \phi_{\overline{V}_{h+1}^t}(z) \right]^{1/2}/\uppvar^t\right\}
    \\&\leq 
\sqrt{\sum_{t = 1}^T \sum_{h = 1}^H \left(\uppvar^t\right)^2} \sqrt{\sum_{t = 1}^T \sum_{h = 1}^H  \min\left\{1, \left[\phi_{\overline{V}_{h+1}^t}(z)^\top \oLambda_{h, 1}^t \phi_{\overline{V}_{h+1}^t}(z) \right]/\uppvar^t\right\}}
    \\&\leq 
\sqrt{\sum_{t = 1}^T \sum_{h = 1}^H \left(\uppvar^t\right)^2} \cdot \sqrt{2H \Gamma_K(T, \lambda_1\alpha^2)}
.
\end{align*}
Similarly,
\begin{align*}
\sum_{t = 1}^T  \sum_{h = 1}^H \lowvar^t\min\{1, \underline{\bonus}_{h, 1}^t(z_h^t)/\lowvar^t\}
    \leq 
\sqrt{\sum_{t = 1}^T \sum_{h = 1}^H \left(\lowvar^t\right)^2} \cdot \sqrt{2H \Gamma_K(T, \lambda_1\alpha^2)}
.
\end{align*}
By Lemma \ref{lemma:boundofvariance}, we have 
\begin{align}
\lefteqn{
    \sum_{t = 1}^T \sum_{h = 1}^H \left(\uppvar^t \right)^2 + \left(\lowvar^t \right)^2
}\notag \\
    &= 
O\bigg( HT\alpha^2 + H^2 T + H^3 \log (1/\delta) +  H \sum_{t = 1}^T \sum_{h = 1}^H \PP_h[\overline{V}_{h+1}^t - \underline{V}_{h+1}^t]\notag
    \\&~\quad +
 \beta_T^{(2)} \sqrt{T H}\sqrt{H\Gamma_K(T, \lambda_2/H^2)}
    +
 \beta_t^{(1)} H^2 \sqrt{T H} \sqrt{H\Gamma_K(T, \lambda_1/\alpha^2)}\bigg)\notag
    \\&\leq 
O\bigg( HT\alpha^2 + H^2 T + H^3 \log (1/\delta))\notag
    \\&~\quad + 
H^2 \beta_t \sqrt{\sum_{t = 1}^T \sum_{h = 1}^H \left(\uppvar^t \right)^2 + \left(\lowvar^t \right)^2} \cdot \sqrt{H \cdot \Gamma_K(T, \lambda_1\alpha^2) } 
    +
H^3 \sqrt{H T \log(H/\delta)}\notag
    \\&~\quad +
\beta_t^{(2)} \sqrt{T H} \sqrt{H\Gamma_K(T, \lambda_2/H^2)}
    +
\beta_t^{(1)} H^2 \sqrt{T H} \sqrt{H\Gamma_K(T, \lambda_1\alpha^2)}\bigg),\label{help:3}
\end{align}
where the inequality holds due to the Cauchy-Schwarz inequality. Next, 
by taking 
\begin{align}
    \alpha = H/\sqrt{\Gamma_K(T, 1/B^2)},\ \lambda_2 = H^2/B^2,\ \lambda_1 = 1/(\alpha^2B^2),\notag
\end{align}
we have
\begin{align}
    &\beta_t^{(1)} = 16 H/\alpha \sqrt{\Gamma_K(T, \lambda_1 (\alpha)^2)} \sqrt{\log(4t^2 H/\delta)} + 8 H/\alpha \log(4t^2 H/\delta) + \sqrt{\lambda_1} \cdot B = \tilde O(\Gamma_K(T, H^2/B^2)),\notag \\
    &\beta_t^{(2)} = 16 H^2 \sqrt{\Gamma_K(T, \lambda_2/H^2)} \sqrt{\log(4t^2 H/\delta)}+ 8 H^2 \log (4t^2 H/\delta) + \sqrt{\lambda_2} \cdot B = \tilde O(H^2),\notag \\
    &\beta_t =
16 \sqrt{\Gamma_K(T, \lambda_1 (\alpha)^2)} \sqrt{\log(4t^2 H/\delta)} + 8 H/\alpha \log(4t^2 H/\delta) + \sqrt{\lambda_1} \cdot B = \tilde O(\sqrt{\Gamma_K(T, H^2/B^2)}).\notag
\end{align}
For simplicity, let $\deff: = \Gamma_K(T, H^2/B^2)$, then by \eqref{help:3} we have
\begin{align}
\lefteqn{
\sum_{t = 1}^T \sum_{h = 1}^H \left(\uppvar^t \right)^2 + \left(\lowvar^t \right)^2
}\notag \\
    &\leq 
\tilde{O}\left(\sqrt{\sum_{t = 1}^T \sum_{h = 1}^H \left(\uppvar^t \right)^2 + \left(\lowvar^t \right)^2} H^{5/2}\deff     +
H^3 \deff^{3/2} T^{1/2}
    +
H^{7/2} T^{1/2} + H^3 T/\deff + H^2 T
\right).\notag
\end{align}
With the fact that $x \leq a\sqrt{x} + b$ leads to $x = O(a^2 + b)$, we have
\begin{align}
\sum_{t = 1}^T \sum_{h = 1}^H \left(\uppvar^t \right)^2 + \left(\lowvar^t \right)^2  
    &=
\tilde{O}\left(
\deff^2 H^5     
+
H^3 \deff^{3/2} T^{1/2}
+
H^{7/2} T^{1/2} + H^3 T/\deff + H^2 T \right).\label{help:4}
\end{align}
Finally, substituting \eqref{help:4} into \eqref{eq:bound_combine_weighted} and bound the summation of $\xi_h^t, \gamma_h^t$ by the Azuma-Hoeffding inequality, we have
\begin{align*}
&
\sum_{t = 1}^T V_1^{*, \nu^t}(x_h^t) - V_1^{\pi^t, *}(x_h^t)
\notag \\&\leq 
\tilde{O}\left( \beta_t \sqrt{\sum_{t = 1}^T \sum_{h = 1}^H \left(\uppvar^t \right)^2 + \left(\lowvar^t \right)^2} \cdot \sqrt{H \cdot \deff } 
    +
H \sqrt{2 H T \log(H/\delta)}
\right)
    \\& =
\tilde{O}\left(
\deff^2 H^3 + \deff^{1.75} H^2 T^{0.25} + \deff H^{2.25} T^{0.25}
    +
\sqrt{\deff} H^2 \sqrt{T}
    +
\deff H^{1.5} \sqrt{T}
\right)
    \\&=
\tilde{O} \left( \deff^2 H^3 
    +
\sqrt{\deff H^4 + \deff^2 H^3} \sqrt{T} 
    +
\left(\deff^7 H^8 + \deff^4 H^9 \right)^{1/4} T^{1/4}
\right)
.
\end{align*}
This completes the proof of the theorem.
\end{proof}

\section{Proof of Results for $\algname$ with Misspecification}\label{sec:mainproof_3}
In this section we prove Theorem \ref{thm:main_RKHSmis}. 
\begin{lemma}\label{lem:main_misspecification}
Assuming that for any $h \in [H]$, $\norm{\btheta_h^*}_{\cH} \leq B$. Let $\lambda = 1 + 1/T$ and $\beta_t$ satisfies
\beq\label{eq:beta_choice_mis}
\left(\frac{\beta_t}{H}\right)^2 \geq 3\Gamma_K(T, \lambda)
+
3
+
6\cdot \log \left( \frac{1}{\delta} \right)
+
3 \lambda \left(\frac{B}{H}\right)^2
+
3 \iota_{\mis}^2 t
.
\eeq
Then for any $\delta > 0$, with probability at least $1 - \delta$ the following holds for any $(t, h) \in [T] \times [H]$ and any $z \in \cZ$:
\begin{align*}
\left| \left\langle \phi_{\overline{V}_{h+1}^t}(z) , \overline{\btheta}_h^t\right\rangle_{\cH} - \PP_h \overline{V}_{h+1}^t(z)\right| 
&\leq 
\beta_t \cdot \overline{\bonus}_h^t(z) + H \cdot \iota_{\mis}
,
\\
\left| \left\langle \phi_{\underline{V}_{h+1}^t}(z) , \underline{\btheta}_h^t\right\rangle_{\cH} - \PP_h \underline{V}_{h+1}^t(z)\right| 
&\leq 
\beta_t \cdot \underline{\bonus}_h^t(z) + H \cdot \iota_{\mis}
.
\end{align*}
\end{lemma}
The proof of Theorem~\ref{thm:main_RKHSmis} shares similar techniques with the proof of Theorem~\ref{thm:main}, except that we need Lemma~\ref{lem:main_misspecification} instead of Lemma~\ref{lem:main}. 
We present the whole proof for completeness.

We recall that the duality gap is defined as $\sum_{t = 1}^T V_1^{*, \nu^t}(x_1^t) - V_1^{\pi^t, *}(x_1^t)$.
As can be seen in our Algorithm~\ref{alg:base}, we maintain an optimistic estimate of $V_h^{*, v^t}(\cdot)$ as $\VUp(\cdot)$ and a pessimistic estimate of $V_h^{\pi^t, *}(\cdot)$ as $\VLp(\cdot)$.
Hence the term $\VUp(x_h^t) - \VLp(x_h^t)$ is approximately the upper bound of the duality gap.
We write the decomposition formally as below:
\beq\label{eq:decomp_misspecification}
\begin{aligned}
V_h^{*, \nu^t}(x_h^t) - V_h^{\pi^t, *}(x_h^t)
=
\underbrace{\overline{V}_h^t(x_h^t) - \underline{V}_h^t(x_h^t)}_{\mbox{I}}
-
\underbrace{\left(  V_h^{\pi^t, *}(x_h^t) - \underline{V}_h^t(x_h^t) \right)}_{\mbox{II}}
-
\underbrace{\left( \overline{V}_h^t(x_h^t) - V_h^{*, \nu^t}(x_h^t) \right)}_{\mbox{III}}
.
\end{aligned}
\eeq
We use $\overline{\delta}_h^t$ to denote the important quantity $\left\langle \phi_{\overline{V}_{h+1}^t}(z_h^t), \btheta_h^* - \overline{\btheta}_h^t \right\rangle_{\cH}$ (and $\left\langle \phi_{\underline{V}_{h+1}^t}(z_h^t), \btheta_h^* - \underline{\btheta}_h^t \right\rangle_{\cH}$) in estimating the duality gap. In the rest of the proof we aim to show that all of the above three terms can be bounded by a quantity related to $\overline{\delta}_h^t$ ($\underline{\delta}_h^t$) and a stochastic random variable that forms a martingale difference sequence for all $h \in [H], t \in [T]$.

For bounding term $\mbox{I}$, we first define two sequences of zero-mean variables:
\beq\label{eq:def_gamma_xi_misspecification}\begin{aligned}
\gamma_h^t
&:=
\QUxh - \QLxh 
- \EE_{(a, b)} \left[\QUxab - \QLxab\right]
,
\\
\xi_h^t
&:=
\left(\PP_h (\VUp - \VLp) \right)(x_h^t, a_h^t, b_h^t) 
-
\left(\VUp(x_{h+1}^t) - \VLp(x_{h+1}^t)\right)
,
\end{aligned}\eeq 
where $\gamma_h^t$ depicts the stochastic error with respect to the policy and $\xi_h^t$ depicts the stochastic error with respect to the transition.
We refer the readers to the proof of Lemma~\ref{lem:bound_UCB_LCB_misspecification} for detailed explanations of these two error terms.
Given the above definition, we have the following lemma.

\begin{lemma}\label{lem:bound_UCB_LCB_misspecification}
Under the settings of Lemma~\ref{lem:main_misspecification}, we have the following recursive bound for $\forall h \in [H]$:
\beq\label{eq:bound_UCB_LCB_misspecification}
\begin{aligned}
\lefteqn{
\overline{V}_h^t(x_h^t) 
-
\underline{V}_h^t(x_h^t)
}
\\&\leq 
\overline{V}_{h+1}^t(x_{h+1}^t)
-
\underline{V}_{h+1}^t(x_{h+1}^t)
+
2\beta_t \min\{1, \overline{\bonus}_h^t(x_h^t)\}
+
2 \beta_t \min\{1, \underline{\bonus}_h^t(x_h^t)\}
+
2 H \cdot \iota_{\mis}
 + \xi_h^t + \gamma_h^t 
.
\end{aligned}\eeq 
\end{lemma}

\begin{proof}[Proof of Lemma~\ref{lem:bound_UCB_LCB_misspecification}]
Follows the same derivative as in Lemma~\ref{lem:bound_UCB_LCB} as it still holds that $\beta_t \geq H$ in Lemma~\ref{lem:main_misspecification}
\end{proof}

For bounding $\mbox{II}$ and $\mbox{III}$, we observe that the gap between two consecutive points of the values of term $\mbox{II}$ can be decomposed as 
\beq\label{eq:sub_decomp_misspecification}\begin{aligned}
&
\left( 
\overline{V}_h^t(x_h^t) - V_h^{*, \nu_h^t}(x_h^t) 
\right)
-
\left(
\overline{V}_{h + 1}^t(x_{h + 1}^t) - V_{h + 1}^{*, \nu_h^t}(x_{h + 1}^t) 
\right)
    \\&=
\underbrace{\left( 
\overline{V}_h^t(x_h^t) - V_h^{*, \nu_h^t}(x_h^t) 
\right)
-
\left( 
\overline{Q}_h^t(z_h^t) - Q_h^{*, \nu_h^t}(z_h^t) 
\right)
}_{\mbox{I}_1}
\\&~\quad + 
\underbrace{\left( 
\overline{Q}_h^t(z_h^t) - Q_h^{*, \nu_h^t}(z_h^t) 
\right) 
-
\left(\PP_h(\overline{V}_{h + 1}^t - V_{h + 1}^{*, \nu_h^t})\right)(z_h^t)
}_{\mbox{I}_2}
 \\&~\quad +
\underbrace{\left(\PP_h(\overline{V}_{h + 1}^t - V_{h + 1}^{*, \nu_h^t})\right)(z_h^t)
-
\left(
\overline{V}_{h + 1}^t(x_{h + 1}^t) - V_{h + 1}^{*, \nu_h^t}(x_{h + 1}^t) 
\right)
}_{\mbox{I}_3}
.
\end{aligned}\eeq %
In the two-player game setting $\mbox{II}$ and $\mbox{III}$ are symmetric, so the terms $\mbox{I}_1, \mbox{I}_2, \mbox{I}_3$ have  correspondending terms denoted as $\mbox{I}_1', \mbox{I}_2', \mbox{I}_3'$ separately.
Utilizing the bound in Lemma~\ref{lem:main}, we can bound $|\mbox{I}_2|$ and $|\mbox{I}_2'|$ by $2 \beta_t \min\{1, \overline{\bonus}_h^t(z_h^t)\} + H \cdot \iota_{\mis}$ and $2 \beta_t \min\{1, \underline{\bonus}_h^t(z_h^t)\} + H \cdot \iota_{\mis}$, respectively, via similar techniques as in the proof of Lemma~\ref{lem:bound_UCB_LCB} and the fact that $\left|\overline{Q}_h^t(z_h^t) - Q_h^{*, \nu_h^t}(z_h^t)\right|  \leq 2H$, $\left|\underline{Q}_h^t(z_h^t) - Q_h^{\pi_h^t, *}(z_h^t)\right|  \leq 2H$ as well as the following inequalities:

\beq\label{eq:I_2_bound_misspecification}\begin{aligned}
\left| \left( 
\overline{Q}_h^t(z_h^t) - Q_h^{*, \nu^t}(z_h^t) 
\right) 
-
\left(\PP_h(\overline{V}_{h + 1}^t - V_{h + 1}^{*, \nu^t})\right)(z_h^t)\right|
	&\leq 2 \beta_t \overline{\bonus}_h^t(z_h^t) + H \cdot \iota_{\mis}
,	\\
\left| \left( 
\underline{Q}_h^t(z_h^t) - Q_h^{\pi^t, *}(z_h^t) 
\right) 
-
\left(\PP_h(\underline{V}_{h + 1}^t - V_{h + 1}^{\pi^t, *})\right)(z_h^t)\right|
	&\leq 2\beta_t \underline{\bonus}_h^t(z_h^t) + H \cdot \iota_{\mis}
.
\end{aligned}\eeq
For $\mbox{I}_3$ and $\mbox{I}_3'$, we note that both terms are stochastic noises with mean zero, where the stochasticity lies in the transition probability. We denote them as $\alpha_{h, t}^1$ and $\alpha_{h, t}^2$ respectively.

Finally for bounding $\mbox{I}_1$ and $\mbox{I}_1'$, we utilize the properties of the CCE (for simplicity we only prove the bound for $\mbox{I}_1$, it is trivial to generalize to the bound for $\mbox{I}_1$):
\begin{align*}
\overline{V}_h^t(x_h^t) - V_h^{*, \nu_h^t}(x_h^t) 
    &= 
\EE_{(a, b) \sim \sigma_h^t(x_h^t)} \overline{Q}_h^t(x_h^t, a, b) 
- 
\EE_{a \sim \text{br}(\nu_h^t), b\sim \nu_h^t}
Q_h^{*, \nu_h^t}(x_h^t, a, b)
,
    \\&\geq 
\EE_{a \sim \text{br}(\nu_h^t), b\sim \nu_h^t} \overline{Q}_h^t(x_h^t, a, b) 
- 
\EE_{a \sim \text{br}(\nu_h^t), b\sim \nu_h^t}
Q_h^{*, \nu_h^t}(x_h^t, a, b)
    \\&=
\EE_{a \sim \text{br}(\nu_h^t), b\sim \nu_h^t} \left[\overline{Q}_h^t(x_h^t, a, b) 
- 
Q_h^{*, \nu_h^t}(x_h^t, a, b)
\right]
,
\end{align*}
where $\nu_h^t := \cP_2 \sigma_h^t$ and $\pi_h^t := \cP_1 \sigma_h^t$ is the projection of $\sigma_h^t$ on the first and second coordinate respectively and $\text{br}$ is the best-response policy of a given distribution.
Defining 
\begin{align*}\zeta_{h, t}^1 &:=
\EE_{a \sim \text{br}(\nu_h^t), b\sim \nu_h^t} \left[\overline{Q}_h^t(x_h^t, a, b) 
- 
Q_h^{*, \nu_h^t}(x_h^t, a, b)
\right]
- 
\left[\overline{Q}_h^t(x_h^t, a_h^t, b_h^t) 
- 
Q_h^{*, \nu_h^t}(x_h^t, a_h^t, b_h^t)
\right]
,
    \\
\zeta_{h, t}^2 &:= \EE_{a \sim \pi_h^t, b\sim \text{br}(\pi_h^t)} \left[
Q_h^{\pi_h^t, *}(x_h^t, a, b)
-
\underline{Q}_h^t(x_h^t, a, b) 
\right]
- 
\left[
Q_h^{\pi_h^t, *}(x_h^t, a_h^t, b_h^t)
-
\underline{Q}_h^t(x_h^t, a_h^t, b_h^t) 
\right]
,
\end{align*}
we arrive at the conclusion that 
\begin{align*}
\overline{V}_h^t(x_h^t) - V_h^{*, \nu_h^t}(x_h^t) 
    &\geq 
\overline{Q}_h^t(x_h^t, a_h^t, b_h^t) 
- 
Q_h^{*, \nu_h^t}(x_h^t, a_h^t, b_h^t)
    +
\zeta_{h, t}^1
,    \\
V_h^{\pi_h^t, *}(x_h^t) - \underline{V}_h^t(x_h^t)
    &\geq 
Q_h^{\pi_h^t, *}(x_h^t, a_h^t, b_h^t)
-
\underline{Q}_h^t(x_h^t, a_h^t, b_h^t) 
    +
\zeta_{h, t}^2
.
\end{align*}
Bringing this lower-bound result together with the previous absolute bound~\eqref{eq:I_2_bound_misspecification} into Eq.~\eqref{eq:sub_decomp_misspecification} and its counterpart for the min-player, we have
\beq\label{eq:bound_23}\begin{aligned}
\lefteqn{
\left( 
V_h^{*, \nu_h^t}(x_h^t) -  \overline{V}_h^t(x_h^t)
\right)
-
\left(
V_{h + 1}^{*, \nu_h^t}(x_{h + 1}^t) - \overline{V}_{h + 1}^t(x_{h + 1}^t)
\right)
}    \\&\hspace{1in}
+
\left( 
\underline{V}_h^t(x_h^t) - V_h^{\pi_h^t, *}(x_h^t) 
\right)
-
\left(
\underline{V}_{h + 1}^t(x_{h + 1}^t) - V_{h + 1}^{\pi_h^t, *}(x_{h + 1}^t)
\right)
    \\&\geq 
\alpha_{h, t}^1 + \alpha_{h, t}^2 
    +
\zeta_{h, t}^1 + \zeta_{h, t}^2 
    -
2 \beta_t \min\{1, \overline{\bonus}_h^t(z_h^t) \}
-
2\beta_t \min\{1,  \underline{\bonus}_h^t(z_h^t)\}  
    -
2H \cdot \iota_{\mis}
.
\end{aligned}\eeq 
%
Combining with Eq.~\eqref{eq:decomp_misspecification}, we arrive at a bound in terms of $\overline{\bonus}_h^t(z_h^t), \underline{\bonus}_h^t(z_h^t)$ and the martingale difference sequences:
\begin{align*}
& \sum_{h = 1}^H V_h^{*, \nu^t}(x_h^t) - V_h^{\pi^t, *}(x_h^t) 
	\\&\leq 
\sum_{h = 1}^H \left(4 \beta_t \min\{1, \overline{\bonus}_h^t(x_h^t)\} + 4 \beta_t \min\{1, \underline{\bonus}_h^t(x_h^t)\}
	+
4 H \cdot \iota_{\mis}
	+
\xi_h^t + \gamma_h^t + \alpha_{h, t}^1 + \alpha_{h, t}^2 + \zeta_{h, t}^1 + \zeta_{h, t}^2\right)
,
\end{align*}
where $\nu^t$ is the policy that operates according to $\nu_h^t$ at time $h$ and $\pi^t$ is the sequence of $\pi_h^t$ accordingly.
The rest of the proof follows by bounding $\sum_{t = 1}^T \sum_{h = 1}^H \min\{1, \overline{\bonus}_h^t(z_h^t)\}, \sum_{t = 1}^T \sum_{h = 1}^H \min\{1, \underline{\bonus}_h^t(z_h^t)\}$ and the martingale difference sequences.
Following the same techniques as in the proof of Theorem~\ref{thm:main}, we again apply Lemma~\ref{lemm:bound_b} and the Cauchy-Schwarz inequality, together with Azuma-Hoeffding for bounded martingale differences, we arrive at our final result.

\section{Proof of Results for $\algname$ with Neural Function Approximation}\label{sec:mainproof_4}
In this section, we proof our result for the neural network approximation.

\begin{proof}[Proof of Theorem~\ref{theo:misspecification}]
Let $C$ be defined in Lemma \ref{lemm:initiallinear}. 
We define a rescaled version of $\phi$ as $\tilde{\phi}/C$, then we know that for any bounded value functions $V_h(\cdot): \cS \mapsto [-1, 1]$, 
\begin{align}
      |\PP_h V_h(z) - \la \tilde \bphi_{V_h}(z), C(\btheta_h^* - \btheta^{(0)})\ra|  \leq C_1 |\cS|B^{4/3} m^{-1/6}L^3\sqrt{\log m}
      ,\qquad 
      \|\tilde\phi_{V_h}(z)\|_2 \leq 1
      .\notag
\end{align}
Defining $\tilde{\btheta}_h^* := C \btheta_h^*$ and $\tilde{\btheta}^{(0)} := C \btheta^{(0)}$ and taking $\iota_{\mis} := C \cdot B^{4/3} \cdot m^{-1/6} \cdot \sqrt{\log m}$, by Theorem~\ref{thm:main_RKHSmis} we know that 
for $\tilde{\lambda} = 1 + \frac{1}{T}$, any $\delta > 0$ and any $\beta$ satisfying
\begin{align*}
\left(\frac{\beta}{H}\right)^2 \geq 2\Gamma_{\tilde{K}}(T, \tilde{\lambda})
+
3
+
6\cdot \log \left( \frac{1}{\delta} \right)
+
3 \lambda \left(\frac{\tilde{B}}{H}\right)^2
+
3 C^2 \cdot B^{8/3} \cdot m^{-1/12} \cdot \log m \cdot t
,
\end{align*}
there exists a global constant $c > 0$ such that with probability at least $1 - \delta$, we have 
\begin{align*}
\operatorname{Regret}(T)
    \leq 
c \left(\beta H \sqrt{T \cdot \Gamma_{\tilde K}(T, \lambda)} + 1 + H^2 T \iota_{\mis} \right)
.
\end{align*}
where $\tilde{B} = C \cdot B$, and 
\begin{align}
        \Gamma_{\tilde K}(T, \lambda): = \sup_{(V_i)_i, (z_i)_i }\frac{1}{2}\log \det (\Ib + \tilde K(\{V_i\}_i, \{z_i\}_i)/\lambda),
\end{align}
where $\tilde{K} \in \RR^{T \times T}$ is the matrix based on the kernel function $k$ induced by the feature mapping $\tilde{\phi}$, where
\begin{align}
    k_{V_1, V_2}(z_1, z_2) = \la \tilde\phi_{V_1}(z_1), \tilde\phi_{V_2}(z_2)\ra.\notag
\end{align}
Rescaling gives
\begin{align*}
\Gamma_{\tilde{K}}(T, \tilde{\lambda}) = \Gamma_{K}(T, C^2 \tilde{\lambda})
.
\end{align*}
Choosing $\lambda := C^2 (1 + \frac{1}{T})$ completes our proof of Theorem~\ref{theo:misspecification}.

\end{proof}

\section{Proof of Auxiliary Lemmas}\label{sec:auxproof}
In this section, we prove the essential lemmas in the proof of our main theorems. First of all we present a concentration bound for self-normalized processes in an RKHS $\cH$, which is critical in determining the main term of the regret bound.
\begin{theorem}[Self-Normalized Concentration Bounds for RKHS~\citep{chowdhury2017kernelized, yang2020function}]\label{theo:concentration}
Let $\{ \xholder_t \}_{t \geq 1}$ be a discrete time stochastic process taking values in $\cZ$, $\cH$ is an RKHS with kernel $K(\cdot, \cdot): \cZ \times \cZ \mapsto \RR$, $\{\cF_t \}_{t \geq 0}$ is a given filtration. We assume that $\xholder_t$ is $\cF_{t - 1}$ measurable in the sense that for $\forall t \geq 1$, $\xholder_t \in \cF_{t - 1}$. Furthermore, $\{ \epsilon_t\}_{t \geq 1}$ is a real-valued stochastic process with each $\epsilon_t$ $\cF_{t}$ measurable and $\sigma$-sub-Gaussian. Define $K_t \in \RR^{(t - 1) \times (t - 1)}$ as the Gram matrix for data $\{\xholder_\tau\}_{\tau \in [t - 1]}$ of the RKHS $\cH$. Then for any $\lambda > 1$ and $\delta \in (0, 1)$, with probability at least $1 - \delta$, the following holds simultaneosly for all $t \geq 0$:
\begin{align*}
\left\|\varepsilon_{1: t - 1}\right\|_{\left(\left(K_{t}+ (\lambda - 1) \cdot \Ib\right)^{-1}+I\right)^{-1}}^{2} 
\leq
 2 \sigma^2 \log \frac{\sqrt{\operatorname{det}\left(\lambda \cdot \Ib+K_{t}\right)}}{\delta}
.
\end{align*}
\end{theorem}
\begin{proof}
See Lemma E.1 in~\citet{yang2020function} and Theorem 1 in~\citet{chowdhury2017kernelized} for the detailed proof.
\end{proof}

\subsection{Proof of Lemma~\ref{lem:main}}
\begin{proof}[Proof of Lemma~\ref{lem:main}]
We only provide the proof of the max-player, and results of the min-player can be derived similarly. We recall the definition of $\overline{\yholder}_h^t$ is the vector of regression targets $\left( \overline{V}_{h+1}^1(x_{h+1}^1), \ldots, \overline{V}_{h+1}^{t - 1}(x_{h+1}^{t - 1})\right)^\top \in \RR^{t - 1} $. Furthermore by Lemma~\ref{lem:facts} in Section~\ref{sec:facts}, we know that $\overline{\btheta}_h^t = \left(\overline{\phiholder}_h^t\right)^\top \left[\overline{K}_h^t + \lambda \cdot \Ib \right]^{-1} \overline{\yholder}_h^t$ and $\phi_{\overline{V}_{h+1}^t}(z) = \left(\overline{\phiholder}_h^t\right)^\top (\overline{K}_h^t + \lambda \cdot \Ib )^{-1} \overline{k}_h^t(z) + \lambda \cdot (\oLambda_h^t)^{-1} \ophi(z)$, which enable us to bound the difference $\left\langle \ophi(z), \otheta_h^t - \btheta_h^*\right\rangle_{\cH}$ as follows:
\begin{align*}
\hprod{\ophi(z)}{ \overline{\btheta}_h^t -   \btheta_h^*}
&=
\ophi(z)^\top \left(\ophiholder_h^t\right)^\top \left[\oK_h^t + \lambda \cdot \Ib \right]^{-1} \overline{\yholder}_h^t
	\\&~\quad -
\left(\otheta_h^*\right)^\top\left[ \left(\ophiholder_h^t\right)^\top \left[\oK_h^t + \lambda \cdot \Ib \right]^{-1} \ok_h^t(z) + \lambda \cdot (\oLambda_h^t)^{-1} \ophi(z) \right] 
\\&=
\underbrace{(\ok_h^t)^\top \left[\oK_h^t + \lambda \cdot \Ib \right]^{-1} \left[\overline{\yholder}_h^t - \ophiholder_h^t \btheta_h^* \right]}_{\mbox{I}_1}
-
\underbrace{\lambda \cdot \phi_{\overline{V}_{h+1}^t}(z)^\top (\oLambda_h^t)^{-1} \btheta_h^*}_{\mbox{I}_2}
.
\end{align*}
For bounding $\mbox{I}_2$, we apply the Cauchy-Schwarz inequality and have 
\begin{align*}
\lambda \cdot \phi_{\overline{V}_{h+1}^t}(z)^\top (\oLambda_h^t)^{-1} \btheta_h^*
&\leq 
\normh{\lambda \cdot \phi_{\overline{V}_{h+1}^t}(z)^\top (\oLambda_h^t)^{-1}} \cdot \normh{\btheta_h^*}
\\&\overset{(a)}{\leq} 
B \cdot \sqrt{ \lambda \phi_{\overline{V}_{h+1}^t}(z)^\top (\oLambda_h^t)^{-1} \lambda (\oLambda_h^t)^{-1} \phi_{\overline{V}_{h+1}^t}(z)} 
\overset{(b)}{\leq} \sqrt{\lambda} B\cdot \overline{\bonus}_h^t(z)
,
\end{align*}%
where $(a)$ is due to the assumption that $\normh{\btheta_h^*} \leq B$ and $(b)$ is by the definition of $\overline{\bonus}_h^t$ and the fact that $\left(\oLambda_h^t\right)^{-1}$ is a self-adjoint mapping on the RKHS $\cH$ satisfying $\norm{\left(\oLambda_h^t\right)^{-1}}_{op} \leq \frac{1}{\lambda}$.

For bounding $\mbox{I}_1$, we observe the following equality:

\begin{align*}
 (\ok_h^t)^\top \left[\oK_h^t + \lambda \cdot \Ib \right]^{-1} \left[\overline{\yholder}_h^t - \ophiholder_h^t \btheta_h^* \right]
&=
\phi_{\overline{V}_{h+1}^t}(z)^\top (\oLambda_h^t)^{-1} (\ophiholder_h^t)^\top \left[\overline{\yholder}_h^t - \phiholder_h^t \btheta_h^* \right]
\\&=
\phi_{\overline{V}_{h+1}^t}(z)^\top (\oLambda_h^t)^{-1} \sum_{\tau = 1}^{t - 1}\phi_{\overline{V}_{h+1}^\tau}(z_h^\tau) \left[\overline{V}_{h + 1}^\tau(x_{h+1}^\tau) - (\PP_h \overline{V}_{h+1}^\tau)(z_h^\tau) \right]
.
\end{align*}
Again by applying the Cauchy-Schwarz inequality, we bound the RHS of the above equality as 
\begin{align*}
|\mbox{I}_1| &\leq 
\norml{\phi_{\overline{V}_{h+1}^t}(z)} 
	\cdot
\norml{\sum_{\tau = 1}^{t - 1}\phi_{\overline{V}_{h+1}^\tau}(z_h^\tau) \left[\overline{V}_{h + 1}^\tau(x_{h+1}^\tau) - (\PP_h \overline{V}_{h+1}^\tau)(z_h^\tau)\right]}
.
\end{align*}
We note that for a given $h$ that we consider in Lemma~\ref{lem:main}. If we define $\{\cF_t\}_{t \geq 0}$ as the $\sigma$-algebra generated by all data before iteration $t$ and all data before step $h$ at iteration $t + 1$,
$\oV_{h+1}^\tau(x_{h+1}^\tau) - (\PP_h \oV_{h+1}^\tau)(z_h^\tau) \in \cF_{\tau}$ is a mean zero random variable with respect to filtration $\cF_{\tau - 1}$ with $\left|\oV_{h+1}^\tau(x_{h+1}^\tau) - (\PP_h \oV_{h+1}^\tau)(z_h^\tau)\right| \leq H$. We take $\epsilon_\tau = \oV_{h+1}^\tau(x_{h+1}^\tau) - (\PP_h \oV_{h+1}^\tau)(z_h^\tau)$ in Theorem~\ref{theo:concentration} and $\{\xholder_t\}_{t \geq 1}$ is $\left\{ \ophi(z_h^t)\right\}_{t \geq 1}$. Then by directly utilizing Theorem~\ref{theo:concentration}, we have that 
\begin{align*}
&\norml{\sum_{\tau = 1}^{t - 1}\phi_{\overline{V}_{h+1}^\tau}(z_h^\tau) \left[\overline{V}_{h + 1}^t(x_{h+1}^\tau) - (\PP_h \overline{V}_{h+1}^t)(z_h^\tau)\right]}^2 
= 
\norml{(\phiholder_h^t)^\top \epsilon_h^t}^2 
= 
(\epsilon_h^t)^\top \phiholder_h^t \Lambda_t^{-1} (\phiholder_h^t)^\top \epsilon_h^t
.
\end{align*}%
By Lemma~\ref{lem:facts} and Theorem~\ref{theo:concentration} again, we arrive at the following: 
\begin{align*}
(\epsilon_h^t)^\top \phiholder_h^t \Lambda_t^{-1} (\phiholder_h^t)^\top \epsilon_h^t
	&\leq 
(\epsilon_h^t)^\top K_t (K_t + \lambda \cdot \Ib)^{-1} \epsilon_h^t
	\leq 
(\epsilon_h^t)^\top \left(K_t   + (\lambda - 1) \cdot \Ib \right)(K_t + \lambda \cdot \Ib)^{-1} \epsilon_h^t
	\\&=
(\epsilon_h^t)^\top \left(\left(K_t   + (\lambda - 1) \cdot \Ib \right)^{-1} + I \right)^{-1} \epsilon_h^t
	\\&\leq 
 H^2 \log \frac{\sqrt{\operatorname{det}\left(\lambda \cdot \Ib+K_{t}\right)}}{\delta}
	\leq 
2 H^2 \cdot \logdet (\lambda \cdot \Ib + K_t ) + 2 H^2 \cdot \log \frac{1}{\delta}
.
\end{align*}
By taking $\lambda = 1 + \frac{1}{T}$,
\begin{align*}
&\left| \phi(z)^\top (\btheta_h^* - \overline{\btheta}_h^t) \right|\notag \\
&\leq 
\left\{\left[ H^2 \cdot \logdet \left[ \lambda \cdot \Ib + K_t \right] + 2 H^2 \cdot \log \left( \frac{1}{\delta} \right) \right]^{1/2} + \sqrt{\lambda} B\right\} \cdot b_h^t(z)
\\&\leq 
\left\{\left[H^2 \cdot \logdet \left[ \Ib + K_t/\lambda \right] + (\lambda - 1) t H^2 + 2 H^2 \cdot \log \left( \frac{1}{\delta} \right)\right]^{1/2} + \sqrt{\lambda} B\right\} \cdot b_h^t(z)
\\&\leq 
\left\{\left[H^2 \cdot \Gamma_K(T, \lambda) +  H^2+ 2H^2 \cdot \log \left( \frac{1}{\delta} \right)\right]^{1/2} + \sqrt{\lambda} B \right\} \cdot b_h^t(z)
.
\end{align*}
Let $
(\beta/H)^2
=
2\Gamma_K(T, \lambda)
+
2
+
4\cdot \log \left( \frac{1}{\delta} \right)
+
2 \lambda \left(\frac{B}{H}\right)^2
$, we know from the above theoretical derivation that $
\left| \phi(z)^\top (\btheta_h^* - \overline{\btheta}_h^t) \right|
\leq
\beta \cdot b_h^t(z)
$ with probability at least $1 - \delta$, which concludes our proof.
\end{proof}

\subsection{Proof of Lemma~\ref{lem:main_misspecification}}
\begin{proof}[Proof of Lemma~\ref{lem:main_misspecification}]
The first half of the proof of Lemma~\ref{lem:main_misspecification} follows exactly as in the proof of Lemma~\ref{lem:main}. We restate the proof for completeness and analyze the error terms brought by misspecification in the later half of the proof. We only provide the proof for the max-player; results of the min-player can be derived similarly. We recall the definition of $\overline{\yholder}_h^t$ is the vector of regression targets $\left( \overline{V}_{h+1}^1(x_{h+1}^1), \ldots, \overline{V}_{h+1}^{t - 1}(x_{h+1}^{t - 1})\right)^\top \in \RR^{t - 1} $. Furthermore by Lemma~\ref{lem:facts} in Section~\ref{sec:facts}, we know that $\overline{\btheta}_h^t = \left(\overline{\phiholder}_h^t\right)^\top \left[\overline{K}_h^t + \lambda \cdot \Ib \right]^{-1} \overline{\yholder}_h^t$ and $\phi_{\overline{V}_{h+1}^t}(z) = \left(\overline{\phiholder}_h^t\right)^\top (\overline{K}_h^t + \lambda \cdot \Ib )^{-1} \overline{k}_h^t(z) + \lambda \cdot (\oLambda_h^t)^{-1} \ophi(z)$, which enable us to bound the difference $\left\langle \ophi(z), \otheta_h^t - \btheta_h^*\right\rangle_{\cH}$ as follows:
\begin{align*}
\hprod{\ophi(z)}{ \overline{\btheta}_h^t -   \btheta_h^*}
&=
\ophi(z)^\top \left(\ophiholder_h^t\right)^\top \left[\oK_h^t + \lambda \cdot \Ib \right]^{-1} \overline{\yholder}_h^t
	\\&~\quad -
\left(\otheta_h^*\right)^\top\left[ \left(\ophiholder_h^t\right)^\top \left[\oK_h^t + \lambda \cdot \Ib \right]^{-1} \ok_h^t(z) + \lambda \cdot (\oLambda_h^t)^{-1} \ophi(z) \right] 
\\&=
\underbrace{(\ok_h^t)^\top \left[\oK_h^t + \lambda \cdot \Ib \right]^{-1} \left[\overline{\yholder}_h^t - \ophiholder_h^t \btheta_h^* \right]}_{\mbox{I}_1}
-
\underbrace{\lambda \cdot \phi_{\overline{V}_{h+1}^t}(z)^\top (\oLambda_h^t)^{-1} \btheta_h^*}_{\mbox{I}_2}
.
\end{align*}
For bounding $\mbox{I}_2$, we apply the Cauchy-Schwarz inequality and have 
\begin{align*}
\lambda \cdot \phi_{\overline{V}_{h+1}^t}(z)^\top (\oLambda_h^t)^{-1} \btheta_h^*
&\leq 
\normh{\lambda \cdot \phi_{\overline{V}_{h+1}^t}(z)^\top (\oLambda_h^t)^{-1}} \cdot \normh{\btheta_h^*}
\\&\overset{(a)}{\leq} 
B \cdot \sqrt{ \lambda \phi_{\overline{V}_{h+1}^t}(z)^\top (\oLambda_h^t)^{-1} \lambda (\oLambda_h^t)^{-1} \phi_{\overline{V}_{h+1}^t}(z)} 
\overset{(b)}{\leq} \sqrt{\lambda} B\cdot \overline{\bonus}_h^t(z)
,
\end{align*}%
where $(a)$ is due to the assumption that $\normh{\btheta_h^*} \leq B$ and $(b)$ is by the definition of $\overline{\bonus}_h^t$ and the fact that $\left(\oLambda_h^t\right)^{-1}$ is a self-adjoint mapping on the RKHS $\cH$ satisfying $\norm{\left(\oLambda_h^t\right)^{-1}}_{op} \leq \frac{1}{\lambda}$.

For bounding $\mbox{I}_1$, we observe the following equality:
\begin{align*}
 &(\ok_h^t)^\top \left[\oK_h^t + \lambda \cdot \Ib \right]^{-1} \left[\overline{\yholder}_h^t - \ophiholder_h^t \btheta_h^* \right]\notag \\
&=
\phi_{\overline{V}_{h+1}^t}(z)^\top (\oLambda_h^t)^{-1} (\ophiholder_h^t)^\top \left[\overline{\yholder}_h^t - \phiholder_h^t \btheta_h^* \right]
\\&=
\phi_{\overline{V}_{h+1}^t}(z)^\top (\oLambda_h^t)^{-1} \sum_{\tau = 1}^{t - 1}\phi_{\overline{V}_{h+1}^\tau}(z_h^\tau) \left[\overline{V}_{h + 1}^\tau(x_{h+1}^\tau) - (\PP_h \overline{V}_{h+1}^\tau)(z_h^\tau) \right]
	\\&~\quad +
\phi_{\overline{V}_{h+1}^t}(z)^\top (\oLambda_h^t)^{-1} \sum_{\tau = 1}^{t - 1}\phi_{\overline{V}_{h+1}^\tau}(z_h^\tau) \left[ (\PP_h \overline{V}_{h+1}^\tau)(z_h^\tau) - \hprod{\phi_{\overline{V}_{h+1}^\tau}}{ \btheta_h^*}\right]
.
\end{align*}
Again by applying the Cauchy-Schwarz inequality, we bound the RHS of the above equality as 
\beq\label{eq:mid_miss}\begin{aligned}
|\mbox{I}_1| &\leq 
\norml{\phi_{\overline{V}_{h+1}^t}(z)} 
	\cdot
 \underbrace{\norml{\sum_{\tau = 1}^{t - 1}\phi_{\overline{V}_{h+1}^\tau}(z_h^\tau) \left[\overline{V}_{h + 1}^\tau(x_{h+1}^\tau) - (\PP_h \overline{V}_{h+1}^\tau)(z_h^\tau)\right]}}_{\mbox{$A_1$}}
	\\&~\quad 
+ 
\norml{\phi_{\overline{V}_{h+1}^t}(z)}
\cdot
\underbrace{\norml{\sum_{\tau = 1}^{t - 1}\phi_{\overline{V}_{h+1}^\tau}(z_h^\tau) \left[(\PP_h \overline{V}_{h+1}^\tau)(z_h^\tau) - \phi_{\overline{V}_{h+1}^\tau}^\top  \btheta_h^*\right]}}_{\mbox{$A_2$}}
.
\end{aligned}\eeq
For bounding $A_1$, we note that for a given $h$ that we consider in Lemma~\ref{lem:main_misspecification}. If we define $\{\cF_t\}_{t \geq 0}$ as the $\sigma$-algebra generated by all data before iteration $t$ and all data before step $h$ at iteration $t + 1$,
$\oV_{h+1}^\tau(x_{h+1}^\tau) - (\PP_h \oV_{h+1}^\tau)(z_h^\tau) \in \cF_{\tau}$ is a mean-zero random variable with respect to filtration $\cF_{\tau - 1}$ with $\left|\oV_{h+1}^\tau(x_{h+1}^\tau) - (\PP_h \oV_{h+1}^\tau)(z_h^\tau)\right| \leq H$. We take $\epsilon_\tau = \oV_{h+1}^\tau(x_{h+1}^\tau) - (\PP_h \oV_{h+1}^\tau)(z_h^\tau)$ in Theorem~\ref{theo:concentration} and $\{\xholder_t\}_{t \geq 1}$ is $\left\{ \ophi(z_h^t)\right\}_{t \geq 1}$. Then by directly utilizing Theorem~\ref{theo:concentration}, we have that 
\begin{align*}
&\norml{\sum_{\tau = 1}^{t - 1}\phi_{\overline{V}_{h+1}^\tau}(z_h^\tau) \left[\overline{V}_{h + 1}^t(x_{h+1}^\tau) - (\PP_h \overline{V}_{h+1}^t)(z_h^\tau)\right]}^2
\notag \\&
= 
\norml{(\phiholder_h^t)^\top \epsilon_h^t}^2
= 
(\epsilon_h^t)^\top \phiholder_h^t \Lambda_t^{-1} (\phiholder_h^t)^\top \epsilon_h^t
.
\end{align*}%
By Lemma~\ref{lem:facts} and Theorem~\ref{theo:concentration} again, we arrive at the following:
\begin{align*}
(\epsilon_h^t)^\top \phiholder_h^t \Lambda_t^{-1} (\phiholder_h^t)^\top \epsilon_h^t
	&\leq 
(\epsilon_h^t)^\top K_t (K_t + \lambda \cdot \Ib)^{-1} \epsilon_h^t
	\leq 
(\epsilon_h^t)^\top \left(K_t   + (\lambda - 1) \cdot \Ib \right)(K_t + \lambda \cdot \Ib)^{-1} \epsilon_h^t
	\\&=
(\epsilon_h^t)^\top \left(\left(K_t   + (\lambda - 1) \cdot \Ib \right)^{-1} + I \right)^{-1} \epsilon_h^t
	\\&\leq 
2H^2 \log \frac{\sqrt{\operatorname{det}\left(\lambda \cdot  \Ib+K_{t}\right)}}{\delta}
	\leq 
 H^2 \cdot \logdet (\lambda \cdot \Ib + K_t ) + 2 H^2 \cdot \log \frac{1}{\delta}
.
\end{align*}

For bounding the term $A_2$ in Eq.~\eqref{eq:mid_miss}, we apply the following lemma, which is the RKHS version of the Lemma 8 in~\citet{zanette2020learning}:
\begin{lemma}[Lemma 8 in~\citet{zanette2020learning}]\label{lem:ab} Let $\left\{a_{i}\right\}_{i=1, \ldots, t}$ be any sequence of vectors in the RKHS $\cH$ and $\left\{b_{i}\right\}_{i=1, \ldots, t}$ be any sequence of scalars such that $
\left|b_{i}\right|  \le \epsilon \in \mathbb{R}^{+}
$.
For any $\lambda \ge 0$ and $t \in \mathbb{N}$ we have:
\begin{align*}
\left\|\sum_{i=1}^{t} a_{i} b_{i}\right\|_{\left[\sum_{i=1}^{t} a_{i} a_{i}^{\top}+\lambda I\right]^{-1}}
\le
t \epsilon^{2}
.
\end{align*}
\end{lemma}

\begin{proof}[Proof of Lemma~\ref{lem:ab}]
By defining the feature matrix $\Ab := \left(a_1, \ldots, a_t \right)$ and the vector $\bbb := \left(b_1, \ldots, b_t \right)^\top $, we have the following formulation
\begin{align*}
\sum_{t = 1}^t a_i b_i = \Ab \bbb, \qquad  
\left\|\sum_{i=1}^{t} a_{i} b_{i}\right\|_{\left[\sum_{i=1}^{t} a_{i} a_{i}^{\top}+\lambda I\right]^{-1}}
    = 
\norm{\Ab \bbb}_{[\Ab \Ab^\top + \lambda \cdot \Ib_{\cH}]^{-1}}
    =
\bbb^\top \Ab^\top [\Ab \Ab^\top + \lambda \cdot \Ib_{\cH}]^{-1} \Ab \bbb
.
\end{align*}
Via the same reasoning as in the proof of item $(a)$ in Lemma~\ref{lem:facts}, 
\begin{align*}
\bbb^\top \Ab^\top [\Ab \Ab^\top + \lambda \cdot \Ib_{\cH}]^{-1} \Ab \bbb
    =
\bbb^\top [ \Ab^\top \Ab + \lambda \cdot \Ib_{\cH}]^{-1} \Ab^\top \Ab \bbb
    =
\bbb^\top \bbb - \lambda \cdot \bbb^\top [ \Ab^\top \Ab + \lambda \cdot \Ib_{\cH}]^{-1} \bbb 
    \leq 
\norm{\bbb}^2 \leq t \epsilon^2
,
\end{align*}
which concludes our proof of Lemma~\ref{lem:ab}.
\end{proof}
By letting $b_\tau = \left[(\PP_h \oV_{h+1}^\tau)(z_h^\tau) - \phi_{\overline{V}_{h+1}^t}^\top \btheta_h^* \right]$ and $a_\tau = \phi_{\overline{V}_{h+1}^\tau}(z_h^\tau)$ and knowing that $|b_\tau| \leq \iota_{\mis}$, we have the bound for item $A_2$ that
$
A_2 \leq \iota_{\mis} \cdot \sqrt{t}
$.
Then by taking $\lambda = 1 + \frac{1}{T}$,
\begin{align}
&\left| \phi(z)^\top (\btheta_h^* - \overline{\btheta}_h^t) \right|\notag \\
&\leq 
\left\{\left[ H^2 \cdot \logdet \left[ \lambda \cdot \Ib + K_t \right] + 2 H^2 \cdot \log \left( \frac{1}{\delta} \right) \right]^{1/2} + \sqrt{\lambda} B
+ H \cdot \iota_{\mis} \sqrt{t}
\right\} \cdot b_h^t(z)
\\&\leq 
\left\{\left[H^2 \cdot \logdet \left[ \Ib + K_t/\lambda \right] + (\lambda - 1) t H^2 + 2 H^2 \cdot \log \left( \frac{1}{\delta} \right)\right]^{1/2} + \sqrt{\lambda} B +H \cdot \iota_{\mis} \sqrt{t}\right\} \cdot b_h^t(z)
\\&\leq 
\left\{\left[H^2 \cdot \Gamma_K(T, \lambda) +  H^2+ 2H^2 \cdot \log \left( \frac{1}{\delta} \right)\right]^{1/2} + \sqrt{\lambda} B + H \cdot \iota_{\mis} \sqrt{t} \right\} \cdot b_h^t(z)    
.
\end{align}
Let $
(\beta/H)^2
=
3\Gamma_K(T, \lambda)
+
3
+
6\cdot \log \left( \frac{1}{\delta} \right)
+
3 \lambda \left(\frac{B}{H} \right)^2  + \iota_{\mis} t^2
$, $
\left| \phi(z)^\top (\btheta_h^* - \overline{\btheta}_h^t) \right|
\leq
\beta \cdot b_h^t(z)
$ with probability at least $1 - \delta$. We present the last step of our proof in the following part.
Equipped with Assumption~\ref{assu:misspecification}, we are able to bound the distance between the optimal estimated expected value at time $h+1$ with the true expected value as in the following lemma, which concludes our proof.
\begin{lemma}\label{lemm:misspecification}
For any bounded value function $V(\cdot): \cS \mapsto [-1, 1]$ and any $z \in \cZ$, there exists a $\btheta_h^* \in \cH$ such that:
\begin{align*}
\left|\PP_h V (z) - \left\langle \phi_V(z), \btheta_h^* \right\rangle_{\cH}\right| 
    \leq 
\iota_{\mis}
,
\end{align*}
where $\phi_V$ is defined in Section~\ref{sec_prelim_rkhs}.
The proof of Lemma~\ref{lemm:misspecification} is a simply application of the definition of total variation distance.
\end{lemma}

\end{proof}

\subsection{Proof of Lemma \ref{lemm:initiallinear}}
In this section we prove Lemma \ref{lemm:initiallinear}. We need the following lemmas.
\begin{lemma}[Lemma 4.1 in \citealt{cao2019generalization, zhou2020neural}]\label{lemma:yuan1}
There exist constants $C_i >0$ such that for any $\delta \in (0,1)$, if $B$ satisfies that
\begin{align}
    C_1m^{-1}L^{-3/2}\max\{\log^{-3/2}m, \log^{3/2}(|\cZ|L^2/\delta)\} \leq B \leq C_2 L^{-6}(\log m)^{-3/2},\notag
\end{align}
then with probability at least $1-\delta$, for all $\btheta_1$ and $\btheta_2$ satisfying $\btheta_1, \btheta_2 \in B(\btheta^{(0)}, B)$ and all $(s', z) \in \cS \times \cZ$, we have
\begin{align}
    |f(s', z; \btheta_1) - f(s', z; \btheta_2) - \la \phi(s',z; \btheta_2), \btheta_1 - \btheta_2\ra| \leq C_3 B^{4/3}m^{-1/6}L^3\sqrt{\log m}.\notag
\end{align}
\end{lemma}

\begin{lemma}[Lemma B.3 in \citealt{cao2019generalization, zhou2020neural}]\label{lemma:yuan2}
There exist constants $C_i >0$ such that for any $\delta \in (0,1)$, if $B$ satisfies that
\begin{align}
    C_1m^{-1}L^{-3/2}\max\{\log^{-3/2}m, \log^{3/2}(|\cZ|L^2/\delta)\} \leq B \leq C_2 L^{-6}(\log m)^{-3/2},\notag
\end{align}
then, with probability at least $1-\delta$, for all $\btheta$ satisfying $\btheta \in B(\btheta^{(0)}, B)$ and all $(s', z) \in \cS \times \cZ$, we have $\|\phi(s'|z)\|_2 \leq C_3 \sqrt{L}$. 
\end{lemma}

\begin{proof}[Proof of Lemma \ref{lemm:initiallinear}]
By Lemma \ref{lemma:yuan2} we have 
\begin{align}
    \|\phi_{V}(z)\|_2 = \bigg\|\sum_{s'}V(s')\phi(s'|z)\bigg\|_2 \leq \sum_{s'}|V(s')|\|\phi(s'|z)\|_2 \leq C_1|\cS|\sqrt{L}.\notag
\end{align}
By the assumptions on $\PP_h(s'|z)$, we have $\btheta_h^* \in B(\btheta^{(0)}, B)$. Thus, by Lemma \ref{lemma:yuan1}, we have with probability at least $1-\delta$, for all $s'\in \cS, z \in \cZ, h \in [H]$, 
\begin{align}
    |\PP_h(s'|z) -  \la \phi(s',z; \btheta^{(0)}), \btheta_h^* - \btheta^{(0)}\ra|
    & = 
    |f(s', z; \btheta_h^*)-f(s', z; \btheta^{(0)}) -  \la \phi(s',z; \btheta^{(0)}), \btheta_h^* - \btheta^{(0)}\ra|\notag \\
    & \leq 
    C_2 B^{4/3}m^{-1/6}L^3\sqrt{\log m},\notag
\end{align}
where the equality holds by the assumptions on $\PP_h$ and $f(s', z;\btheta^{(0)}) = 0$ guaranteed by the initialization scheme, the inequality holds due to Lemma \ref{lemma:yuan1}. Therefore, for any value function $V: \cS \rightarrow [-1, 1]$, we have
\begin{align}
    |\PP_h V(z) - \la \bphi_V(z), \btheta_h^* - \btheta^{(0)}\ra|  &= \bigg|\sum_{s'} V(s') \PP_h(s'|z) - \sum_{s'} V(s')\la \bphi(s'|z), \btheta_h^* - \btheta^{(0)}\ra\bigg|\notag \\
    & \leq \sum_{s'}|V(s')||\PP_h(s'|z) - \la \bphi(s'|z), \btheta_h^* - \btheta^{(0)}\ra|\notag \\
    & \leq C_2 |\cS|HB^{4/3} m^{-1/6}L^3\sqrt{\log m},\notag
\end{align}
where for the second inequality we use the fact that the range of $V$ is a subset of $[-1, 1]$. 
\end{proof}

\section{Implementation Details of \texttt{FIND\_CCE}}\label{app:cce}
Suppose that we have $Q_1, Q_2 \in \cS \times \cA \times \cA \rightarrow \RR$. Given a state $x \in \cS$, let $P_1, P_2 \in \RR^{|\cA| \times |\cA|}$ denote the matrices of Q values such that $[P_i]_{m,n} = Q_i(x, a_m, a_n)$ for $i = 1,2$, where $a_m, a_n$ denote the $m$-th and $n$-th actions of $\cA$. Suppose the CCE of $Q_1, Q_2$ given $x$ is denoted by a matrix $\sigma \in \RR^{|\cA| \times |\cA|}$, where $\sigma_{m,n}$ denotes the probability of selecting $m$-th and $n$-th actions. Recall from Section~\ref{sec_prelim_twoplayer} \texttt{FIND\_CCE} finds $\sigma$ that satisfies the two groups of constraints, repeated here as:
\begin{align}
\mathbb{E}_{(a, b) \sim \sigma}\left[Q_{1}(x, a, b)\right] &\geq \mathbb{E}_{b \sim \mathcal{P}_{2} \sigma}\left[Q_{1}\left(x, a^{\prime}, b\right)\right]
,\quad
\forall a^{\prime} \in \mathcal{A}\tag{\ref{cce:1}}
,\\
\mathbb{E}_{(a, b) \sim \sigma}\left[Q_{2}(x, a, b)\right] & \leq \mathbb{E}_{a \sim \mathcal{P}_{1} \sigma}\left[Q_{2}\left(x, a, b^{\prime}\right)\right]
,\quad
\forall b^{\prime} \in \mathcal{A}
.\tag{\ref{cce:2}}
\end{align}
Since $\sigma$ is a probability matrix, we need
    \begin{align}
        &\forall 1 \leq m, n \leq |\cA|
        ,\quad
        0 \leq \sigma_{m,n} \leq 1
        ,\\
        &\sum_{i=1}^{|\cA|}\sum_{j=1}^{|\cA|}\sigma_{i,j} = 1.
    \end{align}
To satisfy \eqref{cce:1}, we need
    \begin{align}
        &\forall 1 \leq m \leq |\cA|
        ,\quad
        \sum_{i=1}^{|\cA|}\sum_{j=1}^{|\cA|}\sigma_{i,j}[P_1]_{i,j} \geq \sum_{j=1}^{|\cA|} [P_1]_{m,j} \sum_{i=1}^{|\cA|}\sigma_{i,j}
        \notag \\
        &\Leftrightarrow \forall 1 \leq m \leq |\cA|
        ,\quad
        \sum_{i=1}^{|\cA|}\sum_{j=1}^{|\cA|}\sigma_{i,j}([P_1]_{m,j} - [P_1]_{i,j}) \leq 0
        .
    \end{align}
To satisfy \eqref{cce:2}, we need
    \begin{align}
        &\forall 1 \leq n \leq |\cA|
        ,\quad
        \sum_{i=1}^{|\cA|}\sum_{j=1}^{|\cA|}\sigma_{i,j}[P_2]_{i,j} \leq \sum_{i=1}^{|\cA|} [P_2]_{i,n} \sum_{j=1}^{|\cA|}\sigma_{i,j}
        \notag \\
        &\Leftrightarrow \forall 1 \leq n \leq |\cA|
        ,\quad
        \sum_{i=1}^{|\cA|}\sum_{j=1}^{|\cA|}\sigma_{i,j}([P_2]_{i,j} - [P_2]_{i,n}) \leq 0
        .
    \end{align}
There are total $|\cA|^2$ number of unknown variables ($\sigma_{m,n}$) with 1 equality constraint and $|\cA|^2 + 2|\cA|$ number of inequality constrains. The above linear system can be converted into a standard linear programming problem with $2|\cA|^2$ number of variables $\sigma_{m,n}, \hat\sigma_{m,n}, 1 \leq m,n \leq |\cA|$, such that
\allowdisplaybreaks
\begin{align*}
\max_{\sigma_{m,n}, \hat\sigma_{m,n}} 0
\\
    \sigma_{m,n} \geq 0
    \\
    \hat\sigma_{m,n} \geq 0
    \\
    \sigma_{m,n} + \hat\sigma_{m,n} \leq 1
    \\
    -\sigma_{m,n} - \hat\sigma_{m,n} \leq -1
    \\
    \sum_{i=1}^{|\cA|}\sum_{j=1}^{|\cA|}\sigma_{i,j} \leq 1
    \\
    -\sum_{i=1}^{|\cA|}\sum_{j=1}^{|\cA|}\sigma_{i,j} \leq -1
    \\
    \sum_{i=1}^{|\cA|}\sum_{j=1}^{|\cA|}\sigma_{i,j}([P_1]_{m,j} - [P_1]_{i,j}) \leq 0
    \\
    \sum_{i=1}^{|\cA|}\sum_{j=1}^{|\cA|}\sigma_{i,j}([P_2]_{i,j} - [P_2]_{i,n}) \leq 0
\end{align*}
It is well known that the above linear system can be solved by Karmarkar's algorithm \citep{karmarkar1984new} with $\tilde O(|\cA|^7)$ time complexity, or with the Stochastic Central Path Method \citep{cohen2021solving} with $\tilde O(|\cA|^{2w})$ time complexity, where $w = 2.373\ldots$ is the matrix multiplication constant. 
\end{document}